\documentclass{article}

\PassOptionsToPackage{numbers}{natbib}

\usepackage{arxiv}
\usepackage{researchpack}

\makeatletter
\renewcommand{\@maketitle}{%
  \vbox{%
    \hsize\textwidth
    \linewidth\hsize
    \vskip 0.1in
    \@toptitlebar
    \centering
    {\LARGE\sc \@title\par}
    \@bottomtitlebar
    \textsc{\undertitle}\\
    \vskip 0.2in
    \@author
    \vskip 0.4in \@minus 0.1in \center{\@date} \vskip 0.2in
  }
}
\makeatother

\usepackage[american]{babel}
\usepackage[utf8]{inputenc} %
\usepackage[T1]{fontenc}    %
\usepackage{hyperref}       %
\usepackage{url}            %
\usepackage{booktabs}       %
\usepackage{amsfonts}       %
\usepackage{nicefrac}       %
\usepackage{microtype}      %
\usepackage{subcaption}     %
\usepackage[table]{xcolor}         %
\usepackage{natbib}

\usepackage{amsmath}
\usepackage{amssymb}
\usepackage{mathtools}
\usepackage{amsthm}

\newtheorem{theorem}{Theorem}[section]
\newtheorem{lemma}[theorem]{Lemma}
\newtheorem{proposition}[theorem]{Proposition}
\newtheorem*{proposition*}{Proposition}
\newtheorem{corollary}[theorem]{Corollary}

\theoremstyle{definition}
\newtheorem{definition}[theorem]{Definition}

\usepackage{microtype}
\usepackage{graphicx}
\usepackage[dvipsnames]{xcolor}
\usepackage{booktabs}
\usepackage{hyperref}
\usepackage{url}
\usepackage{multirow,rotating,booktabs,arydshln}

\usepackage{thm-restate}
\usepackage{thmtools}
\usepackage[capitalise,nameinlink]{cleveref}
\usepackage{xspace}
\usepackage{enumitem}

\definecolor{acoolcolor}{HTML}{815b9b}

\hypersetup{
    colorlinks=true,
    citecolor=acoolcolor,
    linkcolor=acoolcolor,
    urlcolor=acoolcolor,
}

\newcommand{\Yset}{\Upsilon_\alpha(\vx)\xspace}
\newcommand{\Yseti}{\Upsilon_\alpha(\vx_i)\xspace}
\newcommand{\Yfwset}{\Upsilon^{\sf de}(\vx)\xspace}
\newcommand{\Yrev}{\Upsilon^{{\sf rev}}(\vx)\xspace}

\newcommand{\Cset}{\Gamma_\beta(\vx)\xspace}
\newcommand{\Cind}{\Gamma^{j}_{\beta_j}(\vx)\xspace}
\newcommand{\Cprod}{\vGamma^\times(\vx)\xspace}
\newcommand{\Cab}{\Gamma^{{\sf ab}}(\vx)\xspace}
\newcommand{\Crev}{\Gamma^{{\sf rev}}(\vx)\xspace}

\newcommand{\acronym}{COnformality, COnsistency, COnciseness\xspace}
\newcommand{\methodName}{\texttt{COCOCO}}
\newcommand{\method}{\methodName\xspace}
\newcommand{\EMethod}{\methodName$^{*}$\xspace}

\newcommand{\bastani}{{\tt RPB}\xspace}
\newcommand{\TaskOnly}{{\tt TO}\xspace}
\newcommand{\TaskPlusAbduction}{{\tt TAb}\xspace}
\newcommand{\ConceptsOnly}{{\tt CO}\xspace}
\newcommand{\ConceptsPlusDeduction}{{\tt CDe}\xspace}

\usepackage{pifont}
\newcommand{\cmark}{\ding{51}}
\newcommand{\xmark}{\ding{55}}
\newcommand{\CMARK}{\textcolor{teal}{\cmark}\xspace}
\newcommand{\XMARK}{\textcolor{WildStrawberry}{\xmark}\xspace}

\theoremstyle{plain}

\title{Concise and Logically Consistent Conformal Sets for Neuro-Symbolic Concept-Based Models}

\author{%
  {\bf Samuele Bortolotti$^{1}$ \quad
  Emanuele Marconato$^{1}$ \quad
  Andrea Pugnana$^{1}$} \\[0.4em]
  {\bf Andrea Passerini$^{1}$ \quad
  Stefano Teso$^{1,2}$} \\[0.6em]
  {\normalsize $^{1}$Department of Information Engineering and Computer Science, University of Trento, Italy} \\
  {\normalsize $^{2}$CIMeC, University of Trento, Rovereto, Italy} \\[0.3em]
  {\normalsize \texttt{\{name.surname\}@unitn.it}}
}

\hypersetup{
    pdftitle={Concise and Logically Consistent Conformal Sets for Neuro-Symbolic Concept-Based Models},
    pdfauthor={Samuele Bortolotti, Emanuele Marconato, Andrea Pugnana, Andrea Passerini, Stefano Teso},
    pdfkeywords={Neuro-Symbolic AI, Conformal Prediction, Concept-Based Models, Reasoning Shortcuts, Uncertainty Quantification, Logical Consistency, Trustworthy Machine Learning},
}

\everypar=\expandafter{\the\everypar\looseness=-1 }  %
\newcommand{\changed}[1]{\textcolor{black}{#1}}

\begin{document}
\maketitle

\begin{abstract}
    Neuro-Symbolic Concept-based Models (NeSy-CBMs) are a family of architectures that integrate neural networks with symbolic reasoning for enhanced reliability in high-stakes applications.  They work by first extracting high-level concepts from the input and then inferring a task label from these compatibly with given logical constraints.
    Yet, their label and concept predictions can be overconfident, making it difficult for stakeholders to gauge when the model's decisions can be trusted.
    We address this issue by integrating ideas from Conformal Prediction (CP), a framework providing rigorous, distribution-free coverage guarantees.
    We formalize three desiderata -- consistency, coverage, and conciseness -- that any conformal method for NeSy-CBMs should satisfy, and show that existing approaches fall short of at least one. We then introduce \method, a post-hoc framework that conformalizes concepts and labels jointly and reconciles them via \emph{a single deduction–abduction revision step}. \method satisfies all three desiderata, retains distribution-free coverage, is robust to imperfect knowledge and supports user-specified size budgets.
    Our experiments on $8$ data sets highlight how \method compares favorably against competitors and natural baselines in terms of performance and set size.
\end{abstract}

\section{Introduction}

Neuro-Symbolic AI (NeSy) architectures integrate neural networks with symbolic reasoning, and as such are prime candidates for designing reliable-by-construction AI systems \citep{de2021statistical, garcez2022neural}.
We focus on Neuro-Symbolic Concept-based Models (\textit{\textbf{NeSy-CBMs}}), a general family of NeSy classifiers that distinguish between the neural and reasoning steps at the architectural level~\citep{giunchiglia2022deep, dash2022review, knabs2026bottle, marconato2026symbol}.
Given a low-level input, a NeSy-CBM first distills it into a set of high-level \textit{\textbf{concepts}} using a neural backbone, and then applies differentiable logic reasoning to infer a \textit{\textbf{task prediction}} that complies to constraints we wish to hold.  E.g., if tasked with computing the sum of two MNIST digits \citep{lecun1998mnist}, it would first predict the corresponding digits (the concepts) and then apply the rule of addition (the constraint) \citep{manhaeve2018deepproblog}.  

Despite being geared for reliability, NeSy-CBMs can still be overconfident \citep{marconato2024bears, vankrieken2025neurosymbolic}, making their predictive uncertainty unreliable, see \cref{fig:overview} (Left).
This hinders their application to high-stakes tasks, where one often requires machine learning models to be aware of their own mistakes~\citep{DBLP:conf/aaai/RuggieriP25}. Moreover, even if a few NeSy approaches correct overconfidence in a principled manner (e.g., \citep{marconato2024bears,vankrieken2025neurosymbolic}), their predicted probabilities can be difficult to grasp for practitioners and end-users~\citep{rastogi2022deciding}. 

We address these limitations with \textit{\textbf{conformal prediction}} (CP), a general framework for constructing predictors that, rather than producing a point-wise probability estimate, output a set of plausible alternatives with finite-sample and distribution-free \textit{\textbf{coverage guarantees}} \citep{shafer2008conformal}.
Specifically, we distill three desiderata %
-- \textit{\textbf{consistency}}, \textit{\textbf{coverage}}, and \textit{\textbf{conciseness}} -- that conformal NeSy methods ought to satisfy, and introduce \method (\acronym), a novel and general conformal approach that can turn any NeSy-CBM into a full-fledged conformal NeSy predictor satisfying these desiderata.
In a nutshell, \method implements conformal prediction at the concept and task level and combines them through a post-processing step that ensures the resulting prediction sets are logically consistent with each other.  
Moreover, we show how to leverage \textit{e-values}~\citep{shafer2019game, vovk2021values, grunwald2020safe, ramdas2025hypothesis} to ensure the prediction sets~\citep{vovk2025conformal, gauthier2025values, gauthier2025Conformal, gauthier2026Conformal} attain maximal coverage for any given size budget \changed{while remaining consistent with prior knowledge}.
Our extensive experiments -- on eight image and text datasets and two NeSy CBMs -- show that \method achieves high coverage while keeping the size of the prediction sets under control, even in the presence of noisy, incorrect, or underspecified knowledge.

\textbf{Our Contributions}. %
We \textit{i}) formulate three natural desiderata -- consistency, coverage and conciseness -- that any conformal method for NeSy-CBMs ought to satisfy; %
\textit{ii}) design \method, a post-hoc framework that applies CP jointly at the concept and task levels and reconciles the two via a \emph{single deduction–abduction revision step}; %
\textit{iii}) prove \method satisfies all three desiderata: consistency by construction, distribution-free coverage whose lower bound degrades linearly under imperfect knowledge, and optimal conciseness; 
\textit{iv}) show that \method enables user-specified size budgets with guarantees;
\textit{v}) evaluate \method on $8$ NeSy data sets with different characteristics, highlighting the promise of \method in terms of consistency, coverage and predictive set size \changed{compared to natural competitors}.

\begin{figure*}[!t]
    \centering
    \includegraphics[width=0.8\linewidth]{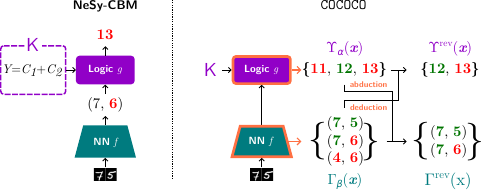}
    \caption{\textbf{Left}: the task and concept predictions of NeSy-CBMs can make confident mistakes (\eg $6$ and $13$, in \textbf{\textcolor{red}{red}}), making it difficult to gauge the (un)reliability of their predictions.
    \textbf{Right}: \method enables any NeSy-CBM to output task- and concept-level prediction sets with coverage guarantees for any desired size, while ensuring their mutual consistency via an abduction-deduction refinement step.  E.g., in the picture the task label $11$ (and the concept label $(4, 6)$) is removed because it is not logically valid, leaving the correct outputs (in \textbf{\textcolor{ForestGreen}{green}}) untouched with high probability.}
    \label{fig:overview}
\end{figure*}

\section{Preliminaries}
\label{sec:preliminaries}

\textbf{Notation}. We write scalar constants $x$ in lower-case, random variables $X$ in upper case, vector constants and variables in bold typeface ($\vx$ and $\vX$), %
and sets $\calX$, $\Upsilon$ in upper-case italics/greek. Given a function %
$f: \calX \to \calY$, we indicate the image of a set $\calX' \subseteq \calX$ as $f(\calX')$ and the pre-image of a set $\calY' \subseteq \calY$ as $f^{-1}(\calY')$. We write $\Delta_\calS$ for the simplex of distributions over $\calS$ and $[n]$ for $\{1, \ldots n\}$.

\textbf{Neuro-Symbolic AI}. NeSy approaches differ wildly in terms of aim and architecture \citep{de2021statistical, garcez2022neural}.
We focus on Neuro-Symbolic Concept-based Models (\textit{\textbf{NeSy-CBMs}} for short), a family of classifiers designed for reliability \citep{giunchiglia2022deep, dash2022review, marconato2026symbol}.  All NeSy-CBMs extract intermediate, high-level \textit{concepts}\footnote{Also known as \textit{neural predicates} \citep{manhaeve2021neural}.} from the input and infer predictions by \textit{reasoning} over the concepts compatibly with user-specified prior knowledge.  The latter is encoded as a logic formula $\BK$.
A NeSy-CBM~\citep{bortolotti2025shortcuts, marconato2023not, marconato2026symbol} maps inputs $\vx \in \calX$ to $k$ concepts $\vc = (c_1, \ldots, c_k) \in \calC$ and then to labels $\vy \in \calY$ using two components.
The \textit{\textbf{concept extractor}} $f: \calX \rightarrow \Delta_\calC$ is a neural network mapping the input $\vx \in \calX$ to a distribution over concepts $p_\theta(\vC \mid \vx)$.
The \textit{\textbf{inference layer}} $g: \calC \to \Delta_\calY$ maps concept vectors to a distribution over labels $p(\vY \mid \vc, \BK)$ according to given and fixed prior knowledge \BK.
With slight abuse of notation, we write $(g \circ f)$ to indicate the conditional distribution $p(\vY \mid \vx; \BK)$ obtained by passing the concept probabilities output by $f$ to the inference layer $g$.
In practice, how these are combined depends on the specific
NeSy-CBM~\citep{marconato2026symbol}; we describe the two we use experimentally below, after a running example.

\begin{example}
    \label{ex:mnist-add}
    In \MNISTAdd \citep{manhaeve2018deepproblog}, the goal is to predict the sum of two \MNIST \citep{lecun1998mnist} digits.
    Here, the concepts $\vc \in \calC := \{0, \dots, 9\}^2$ represent the two digits and the label $y \in \calY := \{0, \dots, 18\}$ is their sum.  The knowledge $\BK = (y = c_1 + c_2)$ specifies\footnote{This constraint can be naturally written in propositional and first-order logic \citep{manhaeve2018deepproblog}.} that $y$ should be the arithmetic sum of the two concepts. E.g., given an input $\vx = \MSeven\MFive$, a NeSy-CBM applies $f$ to infer a distribution over digits and then $g$ to obtain one over sums, typically assigning sums that violate $\BK$ (\eg $\vc = (7, 5)$, $y = 11$) lower probability.
\end{example}

We focus on two representative NeSy-CBMs. \emph{\textbf{DeepProbLog}} (\DPL)~\citep{manhaeve2018deepproblog} models the inference layer $g$ using probabilistic logic semantics~\citep{de2015probabilistic}. This guarantees that the predictive distribution $p_\theta(\vY \mid \vc)$ always allocates zero mass to labels that violate the prior knowledge $\BK$, \eg incorrect sums. \DPL infers hard predictions by solving the MAP problem $\argmax_{\vy} p_\theta(\vy \mid \vc)$, using knowledge compilation \citep{darwiche2002knowledge, vergari2021compositional} for efficiency. \emph{\textbf{LogicTensorNetworks}} (\LTN) \citep{badreddine2022logic} implement the inference layer using fuzzy logic~\citep{donadello2017logic, van2020analyzing}, scoring labels by their fuzzy degree of satisfaction. 
Both architectures learn the parameters $\theta$ via end-to-end training: 
\changed{given a training set $\calD_\textrm{tr} = \{ (\vx_i, \vy_i^*) \}$ sampled from a ground-truth data distribution $p^*(\vX, \vC, \vY)$, \DPL maximizes the likelihood of the training data, while \LTN maximizes the fuzzy degree of satisfaction of the prior knowledge \BK evaluated over the ground-truth training data $\vy_i^*$.}  \changed{When concept annotations are available, they can be leveraged during training by combining the original loss with an average cross-entropy term over the concepts.}

\textbf{\changed{Concept independence and reasoning shortcuts}}.  We highlight two key points.
First, it is common practice to implement the concept extractor $f$ by stacking per-concept neural classifiers~\citep{manhaeve2018deepproblog, badreddine2022logic, marconato2026symbol}. This means the distribution $p_\theta(\vC \mid \vx)$ factorizes as $\prod_i p_\theta(C_i \mid \vx)$, that is, the predicted concepts are \textit{\textbf{conditionally independent}}: $C_i \indep C_j \mid \vX$ for all $i \ne j$ \citep{vankrieken2025neurosymbolic, van2024independence}.
Second, typically ground-truth concept annotations $\vc^*_i$ are not used during training:  as a result, NeSy-CBMs often solve the task by learning concepts that do not match human semantics, a phenomenon known as \textit{\textbf{reasoning shortcuts}} (RSs) \citep{marconato2023not, marconato2026symbol}. Models affected by such RSs output confident but wrong concept predictions even when their label accuracy is high, \changed{and thus suffer from impaired} interpretability and reliability~\citep{marconato2024bears}.

\textbf{Conformal Prediction} (CP) is a general framework for constructing set-valued predictions with \emph{finite-sample} and \emph{distribution-free coverage guarantees}~\citep{shafer2008conformal}. Such guarantees are especially relevant in decision support systems~\citep{detoni2024towards, straitouri2024designing}, as well as in high-stakes applications where human judgment is needed~\citep{arnaiz2025towards, straitouri2024controlling} or where statistical reliability assurances are required~\citep{boger2025functional}.

Given a predictor $p_\theta(\vY \mid \vX)$ and a calibration dataset $\calD_{\mathrm{cal}} = \{(\vx_i, \vy_i^*)\}_{i=1}^n$ of exchangeable examples, CP maps a test input $\vx_i \in \calX$ to a prediction set $\Yseti \subseteq \calY$ such that:
\[
    \Pr( \vy^*_i \in \Yseti) \geq 1 - \alpha .
    \label{eq:basic-conformal-guarantee}
\]
CP relies on a nonconformity score $s(\vx, \vy)$ that quantifies how atypical a candidate output $\vy$ is for an input $\vx$, \eg a typical choice being $1 - p(\vy^* \mid \vx)$~\citep{DBLP:conf/ecml/PapadopoulosPVG02}.  
Let $s_i = s(\vx_i, \vy_i^*)$ for $i \in [n]$ be the nonconformity score of each calibration example, and let $q_{1-\alpha}$ be the $\lceil (n+1)(1-\alpha) \rceil / n$ empirical quantile of $\{s_i\}_{i=1}^n$. The factor $(n+1)/n$ is the standard finite-sample correction that treats the test point as exchangeable with the calibration set, and it ensures the exact finite-sample guarantee~\cref{eq:basic-conformal-guarantee}~\citep{DBLP:conf/ecml/PapadopoulosPVG02, DBLP:journals/ftml/AngelopoulosB23}.
Then, the conformal prediction set for a test input $\vx_i$ %
can be built as:
\[
    \Yseti = \{ \vy \in \calY : s(\vx_i, \vy) \leq q_{1-\alpha} \} .
    \label{eq:basic-conformal-set}
\]

\section{Building Blocks of Conformal NeSy-CBMs}
\label{sec:conformal-nesy}

NeSy-CBMs are geared for reliability in high-stakes applications~\citep{marconato2026symbol, giunchiglia2022learning, giunchiglia2023road}, where it is essential that predictions satisfy domain-specific constraints $\BK$~\citep{manhaeve2018deepproblog, ahmed2022semantic, kurscheidt2025pal, kurscheidt2026theory}. 
Yet, their predictive uncertainty is both unreliable and difficult to understand.  In fact, they are often miscalibrated~\citep{marconato2024bears, vankrieken2025neurosymbolic, DBLP:conf/aaai/RuggieriP25}, meaning stakeholders cannot use their uncertainty to assess whether to trust their predictions \citep{marconato2024bears, vankrieken2025neurosymbolic}.  Moreover, predictive uncertainty is difficult for users to understand in the first place \citep{rastogi2022deciding}.  Both issues are problematic, especially for NeSy-CBMs affected by reasoning shortcuts, in which case users ought to be made aware that the learned concepts are not human aligned, yet these are often predicted wrong with confidence \citep{marconato2026symbol, marconato2023not}.

This calls for a different strategy that can simultaneously handle uncertainty while providing a user-friendly interface for reliability assessment.  This is where CP comes into play, as it offers a principled, distribution-free remedy that does not involve communicating probabilities to users.  We outline the benefits we wish to obtain from such a combination by formalizing three desiderata that \emph{any} NeSy conformal approach ought to satisfy:
\begin{itemize}[leftmargin=2em,itemsep=2pt,topsep=2pt]

    \item[\textbf{D1}] \emph{Consistency}: predicted concepts and labels must be logically consistent with the knowledge $\BK$. %

    \item[\textbf{D2}] \emph{Coverage}: prediction sets must enjoy distribution-free, finite-sample marginal coverage guarantees on \emph{both} labels and concepts -- so as to ensure stakeholders can trust the prediction sets.

    \item[\textbf{D3}] \emph{Conciseness}: prediction sets should be as small as possible, retaining only relevant alternatives -- large sets are uninterpretable and operationally useless~\citep{straitouri2024designing}. \end{itemize}

In~\cref{sec:method}, we will introduce \method, a novel NeSy CP strategies that \textit{does} complies with \textbf{D1}--\textbf{D3}.
Yet, integrating CP into NeSy-CBMs without compromising any desideratum is non-trivial.  In fact, as we will show in \cref{sec:one-sided-prior-work}, any ``\textit{\textbf{one-sided}}'' CP strategy that independently targets either the concept extractor or the inference layer is insufficient to satisfy all desiderata.

\textbf{Inference layer operators}. Before proceeding, we introduce two operators
that all methods we will discuss build on, together with their coverage properties.  We denote with $g^\dagger$ the $\argmax$ of the inference layer $g$ and its preimage as $g^{-\dagger}$, that is:
\[
    \textstyle
    g^\dagger(\vc) := \argmax_{\vy \in \calY} g(\vc) \subseteq \calY, \qquad
    g^{-\dagger}(\vy) := \{ \vc \in \calC: \vy \in g^\dagger(\vc) \}.
    \label{eq:inverse-g-inference-layer-def}
\]
Intuitively, $g^{\dagger}$ \textit{\textbf{deduces}} a (subset of) label(s) from given concepts, while $g^{-\dagger}$ \textit{\textbf{abduces}} a (subset of) concepts from given label(s).
Two properties of $g^\dagger$ play a central role in our coverage guarantees:

\begin{definition}[Deductive \& Abductive Soundness]
\label{def:soundness}
Let $p^*(\vx, \vc, \vy)$ be the joint data distribution. A NeSy-CBM $(f, g)$ satisfies \emph{deductive soundness} if, for every $(\vc^*, \vy^*) \in \mathrm{supp}(p^*(\vx, \vc, \vy))$, $g^\dagger(\vc^*) = \vy^*$. It satisfies \emph{abductive soundness} if, for every $(\vc^*, \vy^*) \in \mathrm{supp}(p^*(\vx, \vc, \vy))$, $\vc^* \in g^{-\dagger}(\vy^*)$.
\end{definition}

Deductive and abductive soundness depend on both the prior knowledge $\BK$ and on how the inference layer $g$ implements it.
In particular, they hold when: (\textit{i}) the knowledge is \textbf{\emph{correct and complete}}, meaning that it allows the ground-truth labels and concepts to be retrieved from one another without errors (as in \MNISTAdd, see \cref{ex:mnist-add}), and (\textit{ii}) the inference layer $g$ \textbf{\emph{preserves this property}}. %
The latter holds for both %
\DPL
and 
\LTN%
\footnote{We specifically refer to versions of \LTN for which this property holds, namely, those that use t-(co)norms~\citep{van2020analyzing}.} -- and in general for most NeSy-CBMs~\citep{ahmed2022semantic, giunchiglia2024ccn, kurscheidt2025pal}.
However, %
these properties may not hold when the knowledge is incomplete or erroneous, as for some \changed{of the datasets used in our experiments (\cref{sec:experiments})}, or when the inference layer is learned from data~\citep{koh2020concept, daniele2023deep, shindo2023thinking, hindo2024learning, wust2024pix2code, tang2023perception, debot2024interpretable}.

\textbf{Label and concept sets via standard CP}.  Given a calibration set $\calD_{\vY} = \{(\vx_i, \vy_i^*)\}$ of exchangeable examples and a miscoverage level $0 < \alpha < 1$, \changed{we can obtain a}
\textbf{\textit{label set}} $\Yset \subseteq \calY$ satisfying:
\[
    \Pr( \vy^* \in \Yset ) \ge 1 - \alpha \, .
    \label{eq:task-only-guarantee}
\]
\changed{simply by applying CP (\cref{eq:basic-conformal-set}) to the NeSy-CBM's label distribution $p_\theta(\vY \mid \vx, \BK)$.}
\changed{Doing the same for the concept extractor $f$ is not as straightforward, in that doing so requires} concept annotations $\vc^*$ on the entire concept vector, %
which are seldom available~\citep{najjar2025scaling, ramalingam2024uncertainty}. \changed{We work around this issue by}
\changed{exploiting} the conditional independence of concepts.
\changed{Specifically,}
for each concept $C_j$ we collect a \textit{per-concept} calibration set $\calD_{C_j} = \{(\vx_i, c_{ij}^*)\}$, build a conformal set $\Cind$ with miscoverage level $\beta_j$%
, and use the latter to construct a \textit{\textbf{product concept set}} as 
\[
    \Cprod = \bigtimes_{j=1}^{k} \Cind = \Gamma^1_{\beta_1}(\vx) \times \cdots \times \Gamma^k_{\beta_k}(\vx).
    \label{eq:product-concept-set}
\]
Notably, the product set inherits coverage guarantees, as shown \changed{by} the following Proposition:
\begin{proposition}[Coverage of $\Cprod$]
    \label{propr:marginal-coverage-cprod}
    For any NeSy-CBM $(f,g)$ and $0 < \beta_j < 1$, let $\beta = \sum_{j=1}^{k} \beta_j$. If each $\Cind$ satisfies $\Pr(\vc_j^* \in \Cind) \ge 1-\beta_j$, then for every $(\vx, \vc^*, \vy^*) \in \mathrm{supp}(p^*(\vx, \vc, \vy))$:
    \[
        \Pr(\vc^* \in \Cprod) \ge 1 - \beta.
        \label{eq:concept-only-guarantee}
    \]
\end{proposition}

All proofs are reported in \cref{sec:proofs} due to space constraints.  \changed{\cref{eq:concept-only-guarantee} implies} that, in order to ensure $\Cprod$ meets joint coverage $1 - \beta$, it is sufficient to choose $\beta_j = \beta/k$ for all $j$.  One downside is that, as $k$ grows, \changed{this allocates vanishing miscoverage to each per-concept set $\Cind$, inflating them and $\Cprod$ as a result}. We will see in \cref{sec:method} that e-values avoid this issue entirely.

In what follows, we use $\Cset$ to denote any concept set satisfying $\Pr(\vc^* \in \Cset) \ge 1 - \beta$. In our default construction, this is $\Cprod$, but our method applies to any concept set with this guarantee, including those built using a calibration set over the entire concept vector $\vc^*$.

\textbf{Abductive and deductive operators}. With these notions in place, we can map between $\Yset$ and $\Cset$ using the \changed{inference layer operators introduced in \cref{eq:inverse-g-inference-layer-def}, that is}:
\[
    \Cab := g^{-\dagger}(\Yset),
    \qquad
    \Yfwset := g^{\dagger}(\Cset).
    \label{eq:abduction-deduction-ops}
\]
$\Cab$ collects all concept vectors that logically entail at least one label in $\Yset$ via \textit{abduction}~\citep{zhou2019abductive, jang2021training, van2022anesi}; $\Yfwset$ collects all labels derivable from concepts in $\Cset$ via \textit{deduction}~\citep{ramalingam2024uncertainty}. These operators transfer coverage guarantees across levels, as shown \changed{by} the following propositions:

\begin{proposition}[Coverage of $\Cab$]
\label{prop:nesy-abduction-coverage}

For any NeSy CBM $(f,g)$, let $0<\alpha <1$. If $\Yset$ is a label conformal set \cref{eq:task-only-guarantee} and $\Cab$ is the set constructed as~\cref{eq:abduction-deduction-ops}, then for every $(\vx, \vc^*, \vy^*) \in \mathrm{supp}(p^*(\vx, \vc, \vy))$: %
\[
    (1 - \alpha) \delta_\mathrm{ab} \le \Pr(\vc^* \in \Cab) \le \alpha + \delta_\mathrm{ab}
\]
where $\delta_\mathrm{ab} := \Pr(\vc^* \in \Cab \mid \vy^* \in \Yset)$ is the probability that abduction %
recovers the ground-truth concepts $\vc^*$ provided the label set $\Yset$ contains the correct label $\vy^*$.  Moreover, if the NeSy-CBM satisfies abductive soundness, the lower bound tightens to $\Pr(\vc^* \in \Cab) \ge 1 - \alpha$.
\end{proposition}

\begin{proposition}[Coverage of $\Yfwset$]
\label{prop:nesy-forward-coverage}
For any NeSy CBM $(f,g)$, let $0 < \beta < 1$. If $\Cset$ is a concept conformal set~\cref{eq:concept-only-guarantee} and $\Yfwset$ is the set constructed as~\cref{eq:abduction-deduction-ops}. Then, for every $(\vx, \vc^*, \vy^*) \in \mathrm{supp}(p^*(\vx, \vc, \vy))$:
\[
    (1 - \beta) \delta_\mathrm{de} \le \Pr(\vy^* \in \Yfwset) \le \beta + \delta_\mathrm{de}
\]
where $\delta_\mathrm{de} := \Pr(\vy^* \in \Yfwset \mid \vc^* \in \Cset)$ measures how likely it is that deduction recovers the true label $\vy^*$ given that the concept set $\Cset$ contains the correct concepts $\vc^*$.  Moreover, under deductive soundness, the bound tightens to $ \Pr(\vy^* \in \Yfwset) \ge 1 - \beta$.
\end{proposition}

\changed{Next, we show that $\Cab$ or $\Yfwset$, while enjoying coverage guarantees, are insufficient to construct a NeSy CP strategy satisfying \textbf{D1}--\textbf{D3} when used in isolation, while combining them through a revision step, as done by \method, does.}

\section{Consistency, Conformality and Conciseness}
\label{sec:method}

We \changed{begin by introducing}
\method (\acronym), a post-hoc framework for turning any NeSy-CBM into a conformal predictor satisfying all three desiderata \textbf{D1}--\textbf{D3}. The idea is as follows:  starting from a label set $\Yset$ and a concept set $\Cset$, we make them mutually consistent by intersecting \changed{the label set (resp. concept set) with $g^\dagger$ applied to the concept set (resp. $g^{-\dagger}$ to the label set).
We show that this requires only one revision step to converge.
} 

\textbf{Joint Revision Step}. Given $\Yset$ satisfying \cref{eq:task-only-guarantee} and $\Cset$ satisfying~\cref{eq:concept-only-guarantee}, \method computes:
\[
    \Yrev := \Yset \cap g^{\dagger}(\Cset),
    \qquad
    \Crev := \Cset \cap g^{-\dagger}(\Yset).
    \label{eq:y-c-rev}
\]

In other words, $\Crev$ retains only the concepts in $\Cset$ that \changed{logically entail} some label in $\Yset$, and $\Yrev$ retains only the labels in $\Yset$ that are entailed by some concept in $\Cset$. We highlight two key structural properties: a \emph{single} application of \cref{eq:y-c-rev} is enough to reconcile the two sets, and the resulting pair is \emph{maximally informative} among all consistent revisions of $(\Cset, \Yset)$. Both properties are made formal below.

\begin{proposition}[One-step consistency and optimality of \method]
\label{prop:one-step}
Let $R(\Gamma, \Upsilon) := \bigl(\Gamma \cap g^{-\dagger}(\Upsilon),\, \Upsilon \cap g^{\dagger}(\Gamma)\bigr)$ denote the joint revision operator. Then, for any $(\Cset, \Yset)$:
\begin{enumerate}[leftmargin=2em,itemsep=1pt,topsep=1pt]
    \item \emph{(One-step fixed point)} $R\bigl(R(\Cset, \Yset)\bigr) = R(\Cset, \Yset)$, so iterating the revision is unnecessary.
    \item \emph{(Optimality)} $(\Crev, \Yrev) = R(\Cset, \Yset)$ is the \emph{largest} pair $(\Gamma', \Upsilon') \subseteq \Cset \times \Yset$ such that $\Gamma' \subseteq g^{-\dagger}(\Upsilon')$ and $\Upsilon' \subseteq g^{\dagger}(\Gamma')$ -- \ie each label $\vy \in \Upsilon'$ is represented by at least a concept $\vc \in \Gamma'$, and vice-versa, according to the inference layer $g$, making it the largest mutually consistent subset of $(\Cset, \Yset)$. 
\end{enumerate}
\end{proposition}

\Cref{prop:one-step} is the operational core of \method: a single deduction--abduction pass simultaneously achieves consistency~(\textbf{D1}) and optimal conciseness~(\textbf{D3}) given $(\Cset, \Yset)$ by construction.
Next, we show that this joint revision preserves coverage guarantees (\textbf{D2}). %

\begin{proposition}[Marginal coverage of \method]
\label{prop:joint-coverage}
For any NeSy CBM $(f,g)$, let $0 < \alpha < 1$ and $0 < \beta < 1$. If $\Cset$ is a concept conformal set~\cref{eq:concept-only-guarantee} and $\Yset$ is a label conformal set~\cref{eq:task-only-guarantee}, then, for every $(\vx, \vc^*, \vy^*) \in \mathrm{supp}(p^*(\vx, \vc, \vy))$, we have
\begin{align}
    &\Pr(\vc^* \in \Crev) \ge (1-\alpha)\delta_\mathrm{ab} - \beta + \Pr(\vc^* \notin \Cset \land \vc^* \notin \Cab),
    \\
    &\Pr(\vy^* \in \Yrev) \ge (1-\beta)\delta_\mathrm{de} - \alpha + \Pr(\vy^* \notin \Yset \land \vy^* \notin \Yfwset) \, ,
\end{align}
where $\Pr(\vc^* \notin \Cset \land \vc^* \notin \Cab)$ and $\Pr(\vy^* \notin \Yset \land \vy^* \notin \Yfwset)$ are the probabilities that the ground-truths are not covered by \emph{neither} the concept/label conformal set \emph{nor} \changed{the} abductive/deductive conformal set. Moreover, if the NeSy predictor satisfies both deductive soundness and abductive soundness (\cref{def:soundness}), the bounds tighten, %
with $\delta_\mathrm{ab} = 1 = \delta_\mathrm{de}$.%
\end{proposition}

Under \changed{soundness}
($\delta_\mathrm{ab} = \delta_\mathrm{de} = 1$), \method retains coverage at least $1 - \alpha - \beta$ on both sides, plus a non-negative joint-failure correction, \ie $\Pr(\vc^* \notin \Cset \land \vc^* \notin \Cab)$ or $\Pr(\vy^* \notin \Yset \land \vy^* \notin \Yfwset)$. This correction term is empirically estimable from a held-out set and is small in practice ($\le 0.08$ across our datasets, see \cref{sec:additional-experiments}). When soundness degrades -- \eg on incomplete-knowledge datasets like \RIVAL\ -- the bounds \changed{worsen} according to the quality of the inference layer ($\delta$).

\textbf{User-specified size budgets}.
Standard CP fixes $\alpha, \beta$ a priori and yields the corresponding set sizes. In practice, specifically in decision-support systems~\citep{detoni2024towards, straitouri2024controlling, straitouri2024designing}, it is often more useful to fix a \emph{size budget} (``at most $5$ candidate concepts'') and report the corresponding coverage, thereby avoiding overwhelming practitioners with a range of choices. As shown by~\changed{\citet{gauthier2025Conformal}}, this can be achieved by replacing standard CP with \emph{conformal e-predictions}~\citep{vovk2025conformal, gauthier2025values, gauthier2025Conformal, gauthier2026Conformal}, which use \emph{e-values}~\citep{shafer2019game, vovk2021values, grunwald2020safe, ramdas2025hypothesis} in place of quantiles. 
E-values bring two benefits to \method. First, we can provably build a conformal set over concepts \emph{\changed{while} avoiding Bonferroni's correction} (\cref{cor:bonferroni-coverage}), \changed{meaning that} the resulting concept-set guarantee does \emph{not} degrade with the number of concepts $k$ (\cref{prop:e-world-coverage}).
Second, e-values allow data-dependent miscoverage levels $\tilde{\alpha}, \tilde{\beta}$ to be chosen \emph{after} observing calibration scores, in particular to satisfy user-specified size budgets while retaining a coverage guarantee \emph{in expectation}~\citep{ramdas2025hypothesis, vovk2025conformal}. We denote this variant \EMethod, and prove its joint-coverage guarantee in \cref{sec:proof-eval}. \changed{Details about e-values and \EMethod are left to \cref{sec:evalues} due to space constraints.}

\textbf{Runtime}. 
A key advantage of \method is that the joint revision operator reaches a fixed point in a single step (\cref{prop:one-step}), so the additional overhead amounts to computing the abductive and deductive images once, \ie $\Cab = g^{-\dagger}(\Yset)$ and $\Yfwset = g^{\dagger}(\Cset)$. The main computational bottleneck lies in constructing the concept set $\Cprod$, which scales exponentially in the number of concepts $k$ due to the cross-product structure. This issue is mitigated when full concept annotations are available or if \EMethod is employed due to \emph{e-values}, in which cases $\Cset$ \changed{can be obtained carrying out any set product}. In practice, we find that the runtime of \method does not exceed $1.5\times$ that of standard \DPL and \LTN, which \changed{is} acceptable in high-stakes settings where reliability is crucial. A detailed runtime analysis is deferred to \cref{sec:runtime} due to space constraints.

\subsection{One-Sided Strategies do not Suffice}
\label{sec:one-sided-prior-work}

\begin{table*}[!t]
    \caption{\textbf{Summary of guarantees} provided by different NeSy conformal strategies.} %
    \scriptsize
    \centering
    \scalebox{1}{
    \begin{tabular}{lccccccc}
        \toprule
            & \multicolumn{2}{c}{{\sc Consistency} (\textbf{D1})}
            & \multicolumn{2}{c}{{\sc Coverage} (\textbf{D2})}
            & \multicolumn{2}{c}{{\sc Conciseness} (\textbf{D3})}
        \\
        \cmidrule(lr){2-3}
        \cmidrule(lr){4-5}
        \cmidrule(lr){6-7}
        {\sc Method}
            & {\sc Task}
            & {\sc Concepts}
            & {\sc Task $\vY$}
            & {\sc Concepts $\vC$}
            & {\sc Task $\vY$}
            & {\sc Concepts $\vC$}
        \\
        \midrule
        Task Only (\TaskOnly)~\cite{ledaguenel2024neurosymbolic}
            & \XMARK
            & --
            & \CMARK~{(\ref{eq:task-only-guarantee})}
            & \XMARK
            & \XMARK
            & --
        \\
        Task plus Abduction (\TaskPlusAbduction)
            & \CMARK
            & \CMARK
            & \CMARK~{(\ref{eq:task-only-guarantee})}
            & \CMARK~{(\ref{prop:nesy-abduction-coverage})}
            & \XMARK
            & \XMARK
        \\
        Concepts Only (\ConceptsOnly)
            & --
            & \XMARK
            & \XMARK
            & \CMARK~{(\ref{propr:marginal-coverage-cprod})}
            & --
            & \XMARK
        \\
        Concepts plus Deduction (\ConceptsPlusDeduction)
            & \CMARK
            & \CMARK
            & \CMARK~{(\ref{prop:nesy-forward-coverage})}
            & \CMARK~{(\ref{propr:marginal-coverage-cprod})}
            & \XMARK
            & \XMARK
        \\
        Joint (\bastani)~\cite{ramalingam2024uncertainty}
            & \CMARK
            & \XMARK
            & \CMARK
            & \CMARK
            & \CMARK
            & \XMARK
        \\
        \method
            & \CMARK~{(\ref{prop:joint-fixed-point})}
            & \CMARK~{(\ref{prop:joint-fixed-point})}
            & \CMARK~{(\ref{prop:joint-coverage})}
            & \CMARK~{(\ref{prop:joint-coverage})}
            & \CMARK~{(\ref{prop:joint-optimality}, \ref{rsq:four})}
            & \CMARK~{(\ref{prop:joint-optimality}, \ref{rsq:four})}
        \\
        \bottomrule
    \end{tabular}
    }
    \label{tab:methods-summary}
\end{table*}

The joint revision operator
$R(\Gamma, \Upsilon) = (\Gamma \cap g^{-\dagger}(\Upsilon),\, \Upsilon \cap g^{\dagger}(\Gamma))$ admits four \emph{one-sided} restrictions, obtained by conformalizing only one side and either bypassing the symbolic layer or applying $g^\dagger / g^{-\dagger}$ unilaterally.  \changed{We briefly discuss them and their properties (summarized in \cref{tab:methods-summary})}:
\begin{itemize}[leftmargin=1.5em,itemsep=1pt,topsep=2pt]

    \item \changed{Task only (\underline{\TaskOnly}) applies CP to the labels but not to the concepts, meaning it violates coverage \textbf{D2}; it also violates consistency \textbf{D1} as a single concept vector may not entail the entire label set.}

    \item \changed{Concepts only (\underline{\ConceptsOnly}) is the converse, \ie it does not apply CP to the task labels, violating coverage \textbf{D2}; moreover, if employed, Bonferroni inflates $\Cprod$, breaking conciseness \textbf{D3}.} %

    \item \changed{Task plus Abduction} (\underline{\TaskPlusAbduction}) computes $(\Yset, g^{-\dagger}(\Yset))$, in which case the abduced concept set contains \textit{all} concept vectors $\vc$ entailing some label in $\Yset$, \changed{affecting conciseness} \textbf{D3}.

    \item \changed{Concepts plus Deduction} (\underline{\ConceptsPlusDeduction}) computes $(\Cset, g^{\dagger}(\Cset))$ \changed{instead, meaning that it inflates the label side, violating~\textbf{D3}}.

\end{itemize}
\citet{ramalingam2024uncertainty} propose a variation of the revision operator, here denoted as \underline{\bastani}, that
applies $R$ asymmetrically, revising only the label side: it matches \method's $\Yrev$ but leaves $\Cset$ untouched, breaking concept-side \changed{consistency} \textbf{D1}.  \changed{We defer to \cref{sec:other-ablations} for an in-depth discussion.}

\section{Related Work}
\label{sec:related-work}

\textbf{Conformal Prediction}. CP produces prediction sets containing the ground-truth value with a pre-specified coverage~\citep{shafer2008conformal}.
Core theoretical guarantees are established by \citet{DBLP:conf/ecml/PapadopoulosPVG02}. Extensions cover regression~\citep{lei2014distribution,lei2018distribution}; predictive distributions~\citep{vovk2017nonparametric, vovk2018cross, vovk2019computationally}; conditional guarantees~\citep{foygel2021limits,gibbs2025conformal}; noisy ground truth~\citep{DBLP:journals/tmlr/StutzRMSCD23,DBLP:journals/tmlr/CaprioSLD25}, covariate shifts~\citep{DBLP:conf/nips/TibshiraniBCR19} and size control~\citep{gauthier2025Conformal,gauthier2026Conformal}.
We refer the reader to \citet{DBLP:journals/ftml/AngelopoulosB23} for a broader survey on conformal prediction.

\textbf{Conformal prediction for NeSy}. The two closest works are \citet{ledaguenel2024neurosymbolic}, who apply CP only at the label level (matching our \TaskOnly baseline and inheriting its lack of concept-side coverage), and \citet{ramalingam2024uncertainty}, who propose three strategies for symbolic programs over bounding-boxes that align with our \ConceptsPlusDeduction, \TaskOnly, and \bastani baselines. Both are subsumed by our framework. \method goes further by \emph{also} revising the concept set via $g^{-\dagger}(\Yset)$, which is what guarantees concept-side consistency and drives the empirical concept-set reductions reported in \cref{sec:experiments}. Other lines of work~\citep{mell2025fast, barnaby2025active} use CP for LLM-generated programs and are orthogonal to ours.

\changed{
\textbf{NeSy-CBMs and Reasoning Shortcuts}. 
NeSy-CBMs constitute a broad family of architectures that combine neural concept extractors with symbolic inference layers \citep{de2021statistical, bortolotti2025shortcuts}, encompassing models based on probabilistic logic~\citep{manhaeve2018deepproblog,huang2021scallop, dai2020abductive, ahmed2022semantic, van2022anesi,maene2023softunification}, fuzzy logic~\citep{donadello2017logic, badreddine2022logic,giunchiglia2020coherent}, and algebraic constraints~\cite{kurscheidt2025pal,giunchiglia2024ccn}.
A known failure mode of these models is that of \emph{reasoning shortcuts}~\citep{marconato2023not},
whereby the concept extractor learns incorrect concepts that allow to predict correct labels in training~\citep{marconato2026symbol}. This phenomenon is mainly tackled by resorting to additional supervision  \citep{marconato2023not, yang2024analysis} and uncertainty estimation \citep{marconato2024bears,vankrieken2025neurosymbolic}.
}

\section{Experiments}
\label{sec:experiments}

We address the following research questions:
\begin{itemize}[leftmargin=2.5em]
    \item[\textbf{RQ1}] How does \method fare in terms of \textbf{D1}--\textbf{D3} under correct and noisy knowledge?
    \item[\textbf{RQ2}] How does the \method behave when deductive soundness fails?
    \item[\textbf{RQ3}] Can we impose \emph{hard size constraints} on the prediction sets while retaining coverage guarantees?
\end{itemize}

\subsection{Experimental Settings}

\changed{We report results in the main text using \emph{concept supervision}, as the unsupervised setting yields substantially larger prediction sets that are harder to read and less operationally meaningful; the unsupervised results follow the same qualitative trends and are reported in full in~\cref{sec:additional-experiments}.}

\textbf{Data sets}.  We focus on four datasets, grouped by the properties of the prior knowledge \BK. \emph{Complete-and-correct \BK}: {\sc Chest-Xrays} (\underline{\CHX}~\citep{cohen2022torch}) is a high-stakes medical diagnosis task in which the label is the number of active binary symptoms; {\sc CeBaB} (\underline{\CEBAB}~\citep{abraham2022cebab}) is a restaurant-review dataset with four aspect-level concepts and an overall rating derived by majority vote. \emph{Incomplete \BK}: {\sc RIVAL-10} (\underline{\RIVAL}~\citep{moayeri2022rival}) contains $10$ classes with hand-crafted visual attributes such that at least two class pairs share the same concept signature, so deductive soundness fails (\cref{def:soundness}). \emph{Reasoning-shortcut benchmark}: {\sc MNIST-EvenOdd} (\underline{\MNISTEO}~\citep{ bortolotti2024benchmark}) is a biased variant of \MNISTAdd~\citep{manhaeve2018deepproblog} designed to induce reasoning shortcuts. \changed{Details for these and the 
four additional datasets are in~\cref{sec:dataset-details}}. %

\textbf{Competitors and metrics}. We employ two state-of-the-art NeSy-CBMs, \DPL~\citep{manhaeve2018deepproblog} and \LTN~\citep{badreddine2022logic}, \changed{trained with and without concept supervision}, and compare \method against \bastani~\citep{ramalingam2024uncertainty}, the strongest competitor; full ablations against \TaskOnly, \TaskPlusAbduction, \ConceptsOnly, and \ConceptsPlusDeduction are in~\cref{sec:ablations}.
Following \textbf{D1}--\textbf{D3}, we report \textsc{Coverage}, \textsc{Size}, and \textsc{Consistency} at nominal coverages $1-\alpha = 1-\beta = 0.9$, averaged over $10$ seeds.  \changed{Details about hyperparameters and metrics can be found in \cref{sec:implementation-details} and \cref{sec:model-selection-super}.}

\textbf{Setup}. We train \DPL and \LTN end-to-end on the training split following~\cite{bortolotti2024benchmark}. We partition the test set into a calibration set $\calD_{\mathrm{cal}}$ ($20\%$) and a held-out evaluation set $\calD_{\mathrm{test}}$ ($80\%$), computing nonconformity scores $s(\vx, \vy) = -\log p_\theta(\vy \mid \vx)$ on $\calD_{\mathrm{cal}}$ for both concepts and labels, and use them to construct the conformal sets $\Cset$ \changed{(in practice, $\Cprod$)} and $\Yset$ at the desired miscoverage levels. \method's revision step (\cref{eq:y-c-rev}) is then applied at test time on $\calD_{\mathrm{test}}$, \changed{for which} we report all metrics. For \EMethod, we replace the standard nonconformity scores with soft-rank e-values~(\cref{eq:soft-rank-evariable}) and adaptively select $\tilde{\alpha}$ and $\tilde{\beta}$ via bootstrap (\cref{sec:evalues}) to satisfy user-specified size budgets. Full training, conformalization, and evaluation details are in~\cref{sec:implementation-details}.

\subsection{Experimental Results and Discussion}

\textbf{RQ1: \method preserves coverage and consistency with smaller prediction sets across all regimes}. %
By symmetrically revising both sides, \method matches \bastani's label sets while \emph{strictly} improving the concept side. 
\method attains concept consistency $1.00$ across all four datasets, whereas \bastani ranges from $0.12$ (\MNISTEO) to $0.94$ (\RIVAL).
With \emph{complete-and-correct} \BK: on \CHX, \method shrinks the concept set from $5.30$ (\bastani) to $2.90$ ($1.8\times$); on \CEBAB, a consistent reduction, $75.06 \to 53.49$ ($1.4\times$), with concept consistency rising $0.78 \to 1.00$.
With \emph{incomplete} \BK (\RIVAL), \method reduces the concept set $6.33 \to 5.85$ and increases consistency $0.94 \to 1.00$.
On \emph{RS-benchmark} (\MNISTEO, unsupervised), \bastani returns concept sets of size $100$ -- 
\ie all possible digits combinations
-- at consistency $0.12$, reflecting the strong miscalibration induced by RSs~\citep{marconato2024bears, vankrieken2025neurosymbolic}; \method shrinks this to $12.13$ ($8.2\times$) at consistency $1.00$, with label coverage preserved at $0.90$. \method and \bastani both maintain label-set sizes $\leq 2$, except on \CEBAB ($2.31$); label coverage is $\geq 0.87$ for correct and concept coverage is $\geq 0.83$. Results for \LTN are deferred to Appendix~\cref{sec:additional-experiments} due to space constraints.

\begin{table*}[t]
    \caption{Results for \DPL on four representative datasets. We report concept and label consistency ({\sc Const.}, fraction of consistent pairs -- see \cref{sec:metrics-details}), size, and coverage at nominal coverage $1 - \alpha = 1 - \beta = 0.9$, averaged over $10$ seeds (mean $\pm$ std). Full results, including \LTN, are in~\cref{sec:additional-experiments}.}
    \centering
    \begin{minipage}{0.95\textwidth}
        \centering
        \scriptsize
\scalebox{0.8}{
\begin{tabular}{p{1em}lcccccc}
\toprule
& & \multicolumn{3}{c}{{\sc Concepts}} & \multicolumn{3}{c}{{\sc Labels}} \\
\cmidrule(lr){3-5} \cmidrule(lr){6-8} {\sc} & {\sc Method} & {\sc Const.} & {\sc Size} & {\sc Cov.} & {\sc Const.} & {\sc Size} & {\sc Cov.} \\
\midrule
\multirow[c]{3}[0]{1em}{\rotatebox{90}{\CHX}} & \rule{0pt}{0ex} \DPL & $\mathbf{1.00 \pm 0.00}$ & $1.00 \pm 0.00$ & $0.47 \pm 0.11$ & $\mathbf{1.00 \pm 0.00}$ & $1.00 \pm 0.00$ & $0.72 \pm 0.01$ \\
 & \rule{0pt}{0ex} \DPL + \bastani & $0.55 \pm 0.05$ & $5.30 \pm 1.01$ & $\mathbf{0.91 \pm 0.03}$ & $\mathbf{1.00 \pm 0.00}$ & $\mathbf{1.60 \pm 0.18}$ & $\mathbf{0.88 \pm 0.03}$ \\
 & \rule{0pt}{0ex} \DPL + \method & $\mathbf{1.00 \pm 0.00}$ & $\mathbf{2.90 \pm 0.56}$ & $0.83 \pm 0.04$ & $\mathbf{1.00 \pm 0.00}$ & $\mathbf{1.60 \pm 0.18}$ & $\mathbf{0.88 \pm 0.03}$ \\
\cmidrule(lr){2-8}
\multirow[c]{3}[0]{1em}{\rotatebox{90}{\RIVAL}} & \rule{0pt}{0ex} \DPL & $\mathbf{1.00 \pm 0.00}$ & $1.00 \pm 0.00$ & $0.43 \pm 0.02$ & $\mathbf{1.00 \pm 0.00}$ & $1.00 \pm 0.00$ & $0.80 \pm 0.00$ \\
 & \rule{0pt}{0ex} \DPL + \bastani & $0.94 \pm 0.01$ & $6.33 \pm 0.65$ & $\mathbf{0.92 \pm 0.01}$ & $\mathbf{1.00 \pm 0.00}$ & $\mathbf{1.20 \pm 0.02}$ & $\mathbf{0.89 \pm 0.01}$ \\
 & \rule{0pt}{0ex} \DPL + \method & $\mathbf{1.00 \pm 0.00}$ & $\mathbf{5.85 \pm 0.62}$ & $0.84 \pm 0.01$ & $\mathbf{1.00 \pm 0.00}$ & $\mathbf{1.20 \pm 0.02}$ & $\mathbf{0.89 \pm 0.01}$ \\
\cmidrule(lr){2-8}
\multirow[c]{3}[0]{1em}{\rotatebox{90}{\CEBAB}} & \rule{0pt}{0ex} \DPL & $\mathbf{1.00 \pm 0.00}$ & $1.00 \pm 0.00$ & $0.18 \pm 0.02$ & $\mathbf{1.00 \pm 0.00}$ & $1.00 \pm 0.00$ & $0.58 \pm 0.03$ \\
 & \rule{0pt}{0ex} \DPL + \bastani & $0.78 \pm 0.06$ & $75.06 \pm 38.35$ & $\mathbf{0.90 \pm 0.04}$ & $\mathbf{1.00 \pm 0.00}$ & $\mathbf{2.31 \pm 0.16}$ & $\mathbf{0.87 \pm 0.01}$ \\
 & \rule{0pt}{0ex} \DPL + \method & $\mathbf{1.00 \pm 0.00}$ & $\mathbf{53.49 \pm 25.93}$ & $0.83 \pm 0.03$ & $\mathbf{1.00 \pm 0.00}$ & $\mathbf{2.31 \pm 0.16}$ & $\mathbf{0.87 \pm 0.01}$ \\
\cmidrule(lr){2-8}
\multirow[c]{3}[0]{1em}{\rotatebox{90}{\texttt{MN-EO}}} & \rule{0pt}{0ex} \DPL$^{\square}$ & $\mathbf{1.00 \pm 0.00}$ & $1.00 \pm 0.00$ & $0.13 \pm 0.08$ & $\mathbf{1.00 \pm 0.00}$ & $1.00 \pm 0.00$ & $0.72 \pm 0.01$ \\
 & \rule{0pt}{0ex} \DPL + \bastani$^{\square}$ & $0.12 \pm 0.01$ & $100.00 \pm 0.00$ & $\mathbf{1.00 \pm 0.00}$ & $\mathbf{1.00 \pm 0.00}$ & $\mathbf{1.55 \pm 0.18}$ & $\mathbf{0.90 \pm 0.02}$ \\
 & \rule{0pt}{0ex} \DPL + \method$^{\square}$ & $\mathbf{1.00 \pm 0.00}$ & $\mathbf{12.13 \pm 1.37}$ & $0.90 \pm 0.02$ & $\mathbf{1.00 \pm 0.00}$ & $\mathbf{1.55 \pm 0.18}$ & $\mathbf{0.90 \pm 0.02}$ \\
\bottomrule
\end{tabular}
}

    \end{minipage}

    \label{tab:results-global-table}
\end{table*}

\textit{Trade-off note.} %
Concept coverage of \method is systematically $7$--$10$ points below \bastani's (\eg $0.83$ vs $0.91$ on \CHX, $0.84$ vs $0.92$ on \RIVAL). This is unavoidable \changed{since} $\Crev \subseteq \Cset$, so any concept removed by intersection with $\Cab$ previously contributed to coverage. The drop reflects cases in which \changed{$\vy^* \not\in \Yset$}, \changed{potentially} causing $\Cab$ to exclude $\vc^*$. \method trades this small coverage drop for guaranteed concept-side consistency (always $1.00$) %
and substantially smaller sets. %

\textbf{RQ2: coverage guarantees degrades controllably under incomplete knowledge}. The bound in~\cref{prop:joint-coverage} degrades with $\delta_{\mathrm{ab}}$ and $\delta_{\mathrm{de}}$. 
\RIVAL is \emph{designed} so that $2$ of $10$ classes share concepts (\texttt{dog}/\texttt{horse}, \texttt{automobile}/\texttt{truck}), yielding $\delta_{\mathrm{de}} = 0.8$ by construction (verified empirically;
see~\cref{tab:delta-estimates}%
). %
Abductive soundness holds, $\hat{\delta}_{\mathrm{ab}} = 1.00$. Plugging $\delta_{\mathrm{de}} = 0.80$, $\beta = 0.1$ into~\cref{prop:joint-coverage} yields $\Pr(\vy^* \in \Yrev) \geq 0.62$.\footnote{This excludes the joint-failure correction term, which further tightens the bound slightly; see~\cref{sec:additional-experiments} for more details.}
This bound is respected: \DPL attains label coverage $0.80$ on \RIVAL. %
\cref{fig:joint-failure} reports the joint-failure correction $\Pr(\vy^* \notin \Yset \land \vy^* \notin \Yfwset)$ and $\Pr(\vc^* \notin \Cset \land \vc^* \notin \Cab)$, respectively, for both \CIFAR and \RIVAL, confirming it is small ($\leq 0.06$) and closes the gap between the theoretical lower bound and observed coverage. Results on \CIFAR follow a similar trend, ref.  Section~\ref{sec:additional-experiments}.

\begin{figure}[t]
    \centering
    \begin{subfigure}[b]{0.48\linewidth}
        \centering
        \includegraphics[width=\linewidth]{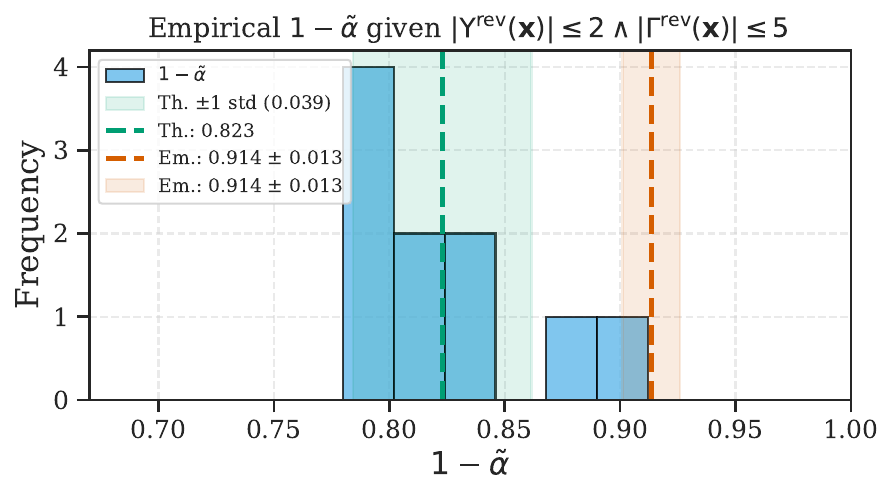}
        \caption{Label coverage}
        \label{fig:joint-alpha}
    \end{subfigure}
    \hfill
    \begin{subfigure}[b]{0.48\linewidth}
        \centering
        \includegraphics[width=\linewidth]{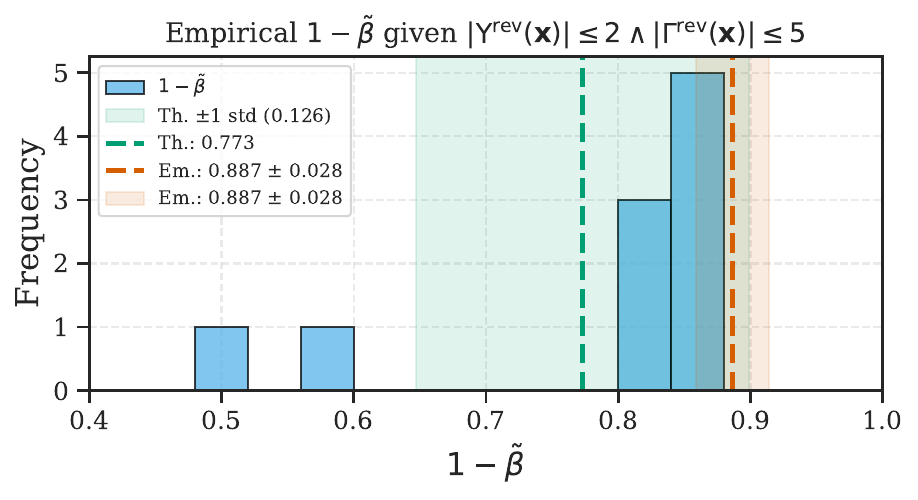}
        \caption{Concept coverage}
        \label{fig:joint-beta}
    \end{subfigure}
    \caption{\EMethod on \CHX with \DPL under fixed regime $|\Yrev| \leq 2 \land |\Crev| \leq 5$, bootstrapped over $100$ calibration resamples and averaged across $10$ seeds. Empirical coverage (\textcolor{orange}{\textbf{orange}}) exceeds the theoretical target (\textcolor{ForestGreen}{\textbf{green}}), while sizes remain \emph{strictly} below the imposed budgets.}
    \label{fig:evalues}
\end{figure}

\textbf{RQ3: \method supports user-specified size budgets while retaining coverage}. \label{rsq:four}
A recurring practical concern with CP is that size and coverage cannot be jointly specified. Conformal e-predictions~\citep{vovk2025conformal, gauthier2025Conformal} invert the procedure: given size budgets, the method adaptively chooses $\tilde{\alpha}$ and $\tilde{\beta}$ while retaining the conformal guarantee~(\cref{prop:evalues-joint-coverage}). E-values additionally avoid Bonferroni's correction~(\cref{prop:e-world-coverage}), so guarantees do not degrade with the number of concepts.
We evaluate on \CHX, fixing label and concept budgets to $2$ and $5$ -- operationally meaningful for a clinical decision-support setting.~\cref{fig:evalues} reports the both-fixed regime, bootstrapped over $100$ calibration resamples. 
Average set sizes are $1.80$ (labels) and $3.57$ (concepts), strictly below the budgets, with empirical coverages $0.91$ and $0.89$, both exceeding by far the theoretical targets $1 - \mathbb{E}[\tilde{\alpha}] = 0.82$ and $1 - \mathbb{E}[\tilde{\beta}] = 0.77$.
Single-sided regimes and full results are in~\cref{tab:econformal} and additional ablation on different e-value formulation in~\cref{sec:product-vs-average}.

\section{Conclusion}
\label{sec:conclusion}

We introduced \method, a post-hoc conformal framework for NeSy-CBMs that jointly conformalizes concept and label predictions and reconciles them via a \emph{single} deduction--abduction revision step. This step provably yields the largest logically consistent revision of the two sets and reaches a fixed point in one pass~(\cref{prop:one-step}). The framework retains distribution-free coverage that degrades controllably under imperfect knowledge, and -- combined with e-values -- supports user-specified size budgets. To our knowledge, \method is the first method that exploits symbolic knowledge to \emph{simultaneously} shrink conformal prediction sets and enforce label--concept consistency while retaining distribution-free coverage. Empirically, \method consistently attains target coverage while reducing concept-set size on eight benchmarks.

\textbf{Limitations \changed{and Future Work}}. Abduction and deduction require enumerating $g^{\dagger}(\Cset)$ and $g^{-\dagger}(\Yset)$, while $\Cprod$ requires to compute the product over a possibly large number of concepts. This becomes a bottleneck for very large concept spaces (\eg ~\BOIA~\citep{xu2020boia, bortolotti2024benchmark}). Similarly, regarding inference time, the inference layer $g$ is called multiple times; while knowledge-compiled circuits~\citep{darwiche2002knowledge, choi2020probabilistic, vergari2021compositional, maene2025klay, derkinderen2025deeplog} (as in \DPL) make this tractable in practice. 
The product-set $\Cprod$ -- if the cross-product construction is used -- assume conditional independence (Proposition~\ref{prop:e-world-coverage}), inherited from standard NeSyCBMs~\citep{van2024independence,vankrieken2025neurosymbolic}. Importantly, the \method revision step itself does \emph{not} require independence: it operates on any concept set satisfying~\cref{eq:concept-only-guarantee}. Plugging in expressive concept distributions~\citep{van2024independence,vankrieken2025neurosymbolic} is a direct extension.
\changed{As future work, we plan to extend \method to a broader range of NeSy-CBMs architectures, leveraging, once mature, recent initiatives that aim to make NeSy-AI more accessible and easier to adopt for practitioners, such as \texttt{DeepLog}~\citep{derkinderen2025deeplog} and \texttt{ULLER}~\citep{van2024uller}, and to architectures that \emph{learn} the knowledge from data~\citep{koh2020concept, daniele2023deep, shindo2023thinking, shindo2024learning, wust2024pix2code, tang2023perception, debot2024interpretable}, where soundness can no longer be assumed a priori and must be estimated, affecting the bounds of \cref{prop:joint-coverage}.}

\subsubsection*{Acknowledgments}

The authors are grateful to Tommaso Carraro for his precious help in the experimental evaluation, and to Cesare Barbera and Carlotta Giacchetta for their useful discussions.
Funded by the European Union. The views and opinions expressed are however those of the author(s) only and do not necessarily reflect those of the European Union, the European Health and Digital Executive Agency (HaDEA) or the European Research Executive Agency. Neither the European Union nor the granting authority can be held responsible for them. Grant Agreement no. 101120763 - TANGO.

\newpage
\bibliography{references, explanatory-supervision}

\newpage
\appendix

\section*{Broader Impact}
By coupling distribution-free coverage with logical consistency, \method makes NeSy-CBMs suitable for decision-support systems~\citep{detoni2024towards, straitouri2024designing, straitouri2024controlling}, particularly in high-stakes settings such as \CHX where practitioners consume both intermediate concepts and final labels. We foresee no negative societal impact specific to \method beyond those shared by any conformal method.

\section{Additional Related Work}
\label{sec:extended-related-work}

\textbf{Uncertainty quantification in NeSy-CBMs}. NeSy-CBMs, despite their architectural bias toward reliability, are not immune to overconfident predictions~\citep{marconato2024bears}. Existing approaches fall into two categories. The first modifies the concept distribution to be more expressive~\citep{van2024independence}, e.g., via deep ensembles~\citep{marconato2024bears} or diffusion-based parametrizations~\citep{vankrieken2025neurosymbolic}. The second enforces uncertainty-related properties during training via constraint-based objectives~\citep{alam2026constraint}. Our approach is complementary and post-hoc: it operates on top of an already-trained NeSy-CBM and provides \emph{distribution-free} coverage on top of any of these methods.

\section{Extended One-Sided Strategies and Prior Work}
\label{sec:other-ablations}
\method combines two design choices: (i) constructing both a label set $\Yset$ and a concept set $\Cset$, and (ii) revising each via the logical image of the other. A natural question is whether either choice can be relaxed without sacrificing the desiderata. In this section we systematically ablate \method along these two axes, showing that each one-sided strategy fails at least one desideratum. The ablations also subsume prior work: \citet{ledaguenel2024neurosymbolic}'s strategy corresponds to applying CP on the label side, while \citet{ramalingam2024uncertainty}'s ``compositional,'' ``direct,'' and ``full'' can be matched in three of our ablations (\ConceptsPlusDeduction, \TaskOnly and \bastani, respectively). \cref{tab:app-methods-summary} summarizes the comparison.

\begin{table*}[!h]
    \centering
    \scalebox{0.8}{
    \begin{tabular}{lccccccc}
        \toprule
            & \multicolumn{2}{c}{{\sc Consistency} (\textbf{D1})}
            & \multicolumn{2}{c}{{\sc Coverage} (\textbf{D2})}
            & \multicolumn{2}{c}{{\sc Conciseness} (\textbf{D3})}
        \\
        \cmidrule(lr){2-3}
        \cmidrule(lr){4-5}
        \cmidrule(lr){6-7}
        {\sc Method}
            & {\sc Task}
            & {\sc Concepts}
            & {\sc Task $\vY$}
            & {\sc Concepts $\vC$}
            & {\sc Task $\vY$}
            & {\sc Concepts $\vC$}
        \\
        \midrule
        Task Only (\TaskOnly)~\cite{ledaguenel2024neurosymbolic}
            & \XMARK
            & --
            & \CMARK~{(\ref{eq:task-only-guarantee})}
            & \XMARK
            & \XMARK
            & --
        \\
        Task plus Abduction (\TaskPlusAbduction)
            & \CMARK
            & \CMARK
            & \CMARK~{(\ref{eq:task-only-guarantee})}
            & \CMARK~{(\ref{prop:nesy-abduction-coverage})}
            & \XMARK
            & \XMARK
        \\
        Concepts Only (\ConceptsOnly)
            & --
            & \XMARK
            & \XMARK
            & \CMARK~{(\ref{propr:marginal-coverage-cprod})}
            & --
            & \XMARK
        \\
        Concepts plus Deduction (\ConceptsPlusDeduction)
            & \CMARK
            & \CMARK
            & \CMARK~{(\ref{prop:nesy-forward-coverage})}
            & \CMARK~{(\ref{propr:marginal-coverage-cprod})}
            & \XMARK
            & \XMARK
        \\
        Joint (\bastani)~\cite{ramalingam2024uncertainty}
            & \CMARK
            & \XMARK
            & \CMARK
            & \CMARK
            & \CMARK
            & \XMARK
        \\
        \method
            & \CMARK~{(\ref{prop:joint-fixed-point})}
            & \CMARK~{(\ref{prop:joint-fixed-point})}
            & \CMARK~{(\ref{prop:joint-coverage})}
            & \CMARK~{(\ref{prop:joint-coverage})}
            & \CMARK~{(\ref{prop:joint-optimality}, \ref{rsq:four})}
            & \CMARK~{(\ref{prop:joint-optimality}, \ref{rsq:four})}
        \\
        \bottomrule
    \end{tabular}
    }
    \caption{Summary of guarantees provided by different NeSy conformal strategies.} %
    \label{tab:app-methods-summary}
\end{table*}

\textbf{One-side CP}.
The simplest strategies apply CP to only one level:
\underline{\TaskOnly}:~\citet{ledaguenel2024neurosymbolic} applies CP only to the label distribution to obtain $\Yset$, and predicts the single $\argmax$ concept $\hat{\vc}$. Label coverage is satisfied by construction, but the concept side has no conformalization (\textbf{D2}), and the predicted concept may not entail \emph{all} labels in $\Yset$, breaking consistency (\textbf{D1}).
\underline{\ConceptsOnly} symmetrically applies CP only to the concepts (via $\Cprod$) and predicts the single $\argmax$ label. The label side is uncoformalized and inconsistencies arise in the opposite direction. Moreover, Bonferroni's correction (\cref{cor:bonferroni-coverage}) inflates $\Cprod$ as the number of concepts $k$ grows, breaking conciseness (\textbf{D3}).

\textbf{Cross-Level CP without revision}.
A natural improvement is to use one of $g^{-\dagger}, g^{\dagger}$ to derive the missing side from the conformalized one -- but without revising both jointly.
\underline{\TaskPlusAbduction} reports the label set $\Yset$ and the abducted concept set $\Cab = g^{-\dagger}(\Yset)$. Both sets are consistent by construction (\textbf{D1}) and inherit coverage from $\Yset$ via~\cref{prop:nesy-abduction-coverage} (\textbf{D2}). However, $\Cab$ contains \emph{every} concept vector entailing some label in $\Yset$ -- including low-probability ones under $f$ -- so it is potentially much larger than the input-aware $\Cset$, breaking conciseness (\textbf{D3}).
\underline{\ConceptsPlusDeduction} (the ``direct'' approach of~\citet{ramalingam2024uncertainty}) is the symmetric strategy: report $\Cset$ and the deduced label set $\Yfwset = g^{\dagger}(\Cset)$. Coverage (\cref{prop:nesy-forward-coverage}) and consistency hold, but $\Yfwset$ inflates when $\Cset$ contains low-probability concepts whose images inflate $\calY$, breaking \textbf{D3} on the label side.

\textbf{One-Sided Joint Revision}:
The closest competitor, \underline{\bastani}~\citep{ramalingam2024uncertainty} (their ``full'' approach), can be represented in our framework as a revision applied only to the label side:
\[
    \Yrev^{\bastani} = \Yset \cap g^{\dagger}(\Cset),
    \qquad
    \Cset^{\bastani} = \Cset \quad \text{(unchanged)}.
\]
This delivers label-side consistency and matches \method's label set by construction. However, the concept set is left untouched: it still contains concept vectors that entail labels \emph{outside} $\Yrev^{\bastani}$, breaking concept-side consistency (\textbf{D1}). Empirically (\cref{sec:experiments}), this manifests as concept-side consistency dropping to $\approx 0.08$--$0.94$ across our datasets, while \method attains $1.00$ throughout.

\textbf{Discussion}: 
On the one hand, $\Yset$ and $\Cset$, used in isolation, leave the other side unconformalized -- so stakeholders cannot tell which concepts (resp.\ labels) can be trusted, violating coverage (\textbf{D2}) -- and the predicted side may not entail or abduce \emph{all} elements of the conformalized one, violating consistency (\textbf{D1}). On the other hand, $\Cab$ and $\Yfwset$ ensure concept-side (resp. label-side) consistency by design, but potentially inflate the resulting set -- $\Cab$ contains every concept vector $\vc$ entailing some label in $\Yset$, and $\Yfwset$ every label deducible from some concept in $\Cset$ -- violating conciseness (\textbf{D3}). \Cref{sec:ablations} analyzes these failure modes in detail, exploring all combinations of these four objects -- including those of \citet{ledaguenel2024neurosymbolic} and \citet{ramalingam2024uncertainty} as special cases -- and shows that none satisfies the three desiderata simultaneously. %

\section{Proofs}
\label{sec:proofs}

\subsection*{Notation and Standing Assumptions}

We recall the definition of $g^{\dagger}$ instead is
\[
    g^{\dagger} := \argmax_{\vy \in \calY} g(\vc) \subseteq \calY \, , 
\]
and its inverse
$g^{-\dagger}$ when it is applied to a subset $\calY_s \subseteq \calY$
\[
    g^{-\dagger}(\vy) = \{
    \vc \in \calC: \vy \in \argmax g(\vc)
    \}, 
    \quad
    g^{-\dagger}(\calY_s) = \bigcup_{\vy \in \calY_s} \{
    \vc \in \calC: \vy \in \argmax g(\vc)
    \} .
    \label{eq:inverse-g-appendix}
\]
We also recall that for a given $0 < \alpha < 1$ and every pair $(\vx, \vy^*) \in \calD_{\vY}$, we define the \textbf{\textit{label set}} $\Yset \subseteq \calY$ that satisfies
\[
    \Pr(\vy^* \in \Yset) \geq 1 - \alpha.
    \label{eq:label-set-constraint-alpha-appendix}
\]
Similarly, for $0 < \beta < 1$ and $(\vx, \vc^*) \in \calD_{\vC}$, we define the \textbf{\textit{concept set}} $\Cset \subset \calC$ that satisfies
\[
    \Pr(\vc^* \in \Cset) \geq 1 - \beta.
    \label{eq:concept-set-constraint-beta-appendix}
\]
To make self-contained this appendix without loss of generality, we prove the main statements in the text, providing equivalent statements with the material here recollected.

\subsection{Proof of abductive and deductive assumptions~(\cref{def:soundness})}

\begin{proposition}[Deductive soundness implies abductive soundness]
\label{prop:deductive-implies-abductive}
For any NeSy-CBM $(f, g)$, deductive soundness implies abductive soundness~(\cref{def:soundness}).  Consequently, a model that satisfies deductive soundness but not abductive soundness cannot exist.
\end{proposition}

\begin{proof}
Suppose the NeSy-CBM satisfies deductive soundness, \ie $g^{\dagger}(\vc^*) = \vy^*$ for every $(\vc^*, \vy^*)$ in $\operatorname{supp}(p^*)$. Then $g^{\dagger}(\vc^*) = \vy^*$ means $\vc^*$ belongs to $g^{-\dagger}(\vy^*) = \{\vc \in \calC: \argmax g(\vc) = \vy^*\}$, which is exactly abductive soundness (\cref{def:soundness}).
\end{proof}

\begin{proposition}[Abductive soundness implies $\delta_{\mathrm{ab}} = 1$]
\label{prop:abductive-soundness-implies-delta-ab}
Let $\Cab := g^{-\dagger}(\Yset)$ be the concept set obtained by abduction according to \cref{eq:inverse-g-appendix}.  If the NeSy-CBM satisfies abductive soundness (\cref{def:soundness}), then
\[
    \delta_{\mathrm{ab}}
    := \Pr \bigl(\vc^* \in \Cab \mid \vy^* \in \Yset\bigr)
    = 1.
\]
\end{proposition}

\begin{proof}
From \cref{def:soundness},
abductive soundness ensures that $\vc^* \in g^{-\dagger}(\vy^*)$ for every $(\vc^*, \vy^*)$ in the support of $p^*(\vx, \vc, \vy)$. Suppose $\vy^* \in \Yset$. 
Because $\Cab = g^{-\dagger}(\Yset) = \bigcup_{\vy \in \Yset} g^{-\dagger}(\vy)$ and $\vc^* \in g^{-\dagger}(\vy^*)$ with $\vy^* \in \Yset$, we have $\vc^* \in g^{-\dagger}(\vy^*) \subseteq \Cab$. 
Hence, the event $\{\vc^* \in \Cab\}$ holds surely whenever $\{\vy^* \in \Yset\}$ holds,
giving
\[
    \Pr(\vc^* \in \Cab \mid \vy^* \in \Yset) = 1.
\]
\end{proof}

\begin{proposition}[Deductive soundness implies $\delta_{\mathrm{de}} = 1$]
\label{prop:deductive-soundness-implies-delta-de}
Let $\Yfwset := g^{\dagger}(\Cset)$ be the label set obtained by the image of $g$ on $\Cset$.  
If the NeSy-CBM satisfies deductive soundness (\cref{def:soundness}), then
\[
    \delta_{\mathrm{de}}
    := \Pr\bigl(\vy^* \in \Yfwset \mid \vc^* \in \Cset\bigr)
    = 1.
\]
\end{proposition}

\begin{proof}
Deductive soundness states that $g^{\dagger}(\vc^*) = \vy^*$ for every
$(\vc^*, \vy^*)$ in the support of $p^*(\vx, \vc, \vy)$. Suppose $\vc^* \in \Cset$.  Then $g^{\dagger}(\vc^*) \in g^{\dagger}(\Cset) = \Yfwset$. Since $g^{\dagger}(\vc^*) = \vy^*$, we have $\vy^* \in \Yfwset$. Hence the event $\{\vy^* \in \Yfwset\}$ holds surely whenever
$\{\vc^* \in \Cset\}$ holds, giving
\[
    \Pr(\vy^* \in \Yfwset \mid \vc^* \in \Cset) = 1.
\]
\end{proof}

\begin{remark}
\Cref{prop:abductive-soundness-implies-delta-ab,prop:deductive-soundness-implies-delta-de} are central to give tighter bounds stated in \cref{prop:nesy-abduction-coverage,prop:nesy-forward-coverage,prop:joint-coverage}: when $\delta_{\mathrm{ab}} = 1$ (resp.\ $\delta_{\mathrm{de}} = 1$) in the general bounds, this  yields the $1 - \alpha$ (resp.\ $1 - \beta$) guarantees.
\end{remark}

\subsection{Coverage proof of \TaskPlusAbduction~(\cref{prop:nesy-abduction-coverage})}
\label{proof:nesy-abduction-coverage}

\begin{proposition*}[\ref{prop:nesy-abduction-coverage}, Marginal coverage of $\Cab$]
For $0< \alpha<1$ and for any $(\vx, \vc^*, \vy^*)$ in  the support of $p^*(\vx, \vc, \vy)$, suppose $\Yset$ satisfies \cref{eq:label-set-constraint-alpha-appendix}. 
Let $\Cab := g^{-\dagger} (\Yset)$.
Then
\[
    (1 - \alpha) \delta_\mathrm{ab} \le \Pr(\vc^* \in \Cab) \le \alpha + \delta_\mathrm{ab}
    \label{eq:task-plus-abduction-guarantee}
\]
where $\delta_\mathrm{ab} := \Pr(\vc^* \in \Cab \mid \vy^* \in \Yset)$ is the probability that abduction recovers the ground-truth concepts $\vc^*$ given that the label set $\Yset$ contains the correct label $\vy^*$.  Moreover, if the NeSy-CBM satisfies abductive soundness (\cref{def:soundness}), the bound tightens to
\[
    \Pr(\vc^* \in \Cab) \ge 1 - \alpha.
\]
\end{proposition*}

\begin{proof}

Let the two binary events $A = (\vy^* \in \Yset)$ and $B = (\vc^* \in \Cab)$.
We first derive the lower bound. By the law of total probability,
\begin{align}
    \Pr(B)
        &= \Pr(B \mid A)\Pr(A) + \Pr(B \mid \neg A)\Pr(\neg A) \label{eq:law-of-total-probability}
    \\
        & \geq \Pr(B \mid A)\Pr(A)
        \label{eq:prob-b-coverage-abduction}
\end{align}
Since $\Yset$ is a conformal label set with miscoverage $\alpha$, %
we have that, by~\cref{eq:task-only-guarantee}:
\[
    \text{Pr}(A) \geq 1 - \alpha
\]
We can plug this result into the previous one, \ie $\Pr(B) \geq \Pr(B \mid A)\Pr(A)$~(\cref{eq:prob-b-coverage-abduction}), to obtain
\begin{align}
    \Pr(B) &\ge 
    \Pr( B | A) (1-\alpha) \label{eq:proof-ab-lower}\\
\end{align}
Which proofs the lower bound:
\begin{align}
    \Pr(\vc^* \in \Cab)
    &\ge
    \Pr(\vc^* \in \Cab \mid \vy^* \in \Yset)(1 - \alpha).
\end{align}
Next, we prove the upper bound. Consider again the right-hand side of~\cref{eq:law-of-total-probability}.  Since $\Pr(B \mid \neg A) \le 1$ and $\Pr(\neg A) \le \alpha$, we have
\[ \label{eq:proof-ab-upper-right}
    \Pr(B \mid \neg A) \Pr(\neg A) \le \alpha.
\]
Moreover, $\Pr(A) \le 1$.  Hence,
\[ \label{eq:proof-ab-upper-left}
    \Pr(B \mid A) \Pr(A) \le \Pr(B \mid A).
\]
Combining these observations yields: %
\begin{align}  \label{eq:proof-ab-upper}
    \Pr(B) &= \Pr(B \mid A)\Pr(A) + \Pr(B \mid \neg A)\Pr(\neg A)
    \tag{Law of total probability - ~\cref{eq:law-of-total-probability}}
    \\ 
    &\le \Pr(B \mid A) + \Pr(B \mid \neg A)\Pr(\neg A) \tag{Substitute the left term - \cref{eq:proof-ab-upper-left}} \\
    & \le \Pr(B \mid A) + \alpha \tag{Substitute the right term - \cref{eq:proof-ab-upper-right}}
\end{align}
Substituting the expression for $A$ and $B$, we have the upper bound 
\[
 \Pr(\vc^* \in \Cab) \le \alpha + \delta_\mathrm{ab}.
\]

Under \emph{abductive soundness} (\cref{def:soundness}), 
by \cref{prop:abductive-soundness-implies-delta-ab} we have $\delta_\mathrm{ab} = 1$. Hence, we get
\[
    \Pr(B) \geq \delta_\mathrm{ab} %
    (1 - \alpha) \geq 1 - \alpha \, .
\]
While for the upper bound~\cref{eq:proof-ab-upper-right}
\[
    \Pr(B) \le \delta_\mathrm{ab} %
    + \alpha \le \min(1, \delta_{\textrm{ab}} + \alpha) \, .
\]
This concludes the proof.
\end{proof}

\subsection{Coverage proof of Product Set~(\cref{propr:marginal-coverage-cprod})}
\label{proof:marginal-coverage-cprod}

For $0<\beta_j<1$, with $j \in \{1, \ldots k\}$, and  for any $(\vx, \vc^*)$ in the support of $p^*(\vx, \vc, \vy)$, we recall  the definition of individual \textbf{\textit{concept sets }} $\Cind$ that satisfy:
\[
   \Pr (\vc^*_j \in \Cind ) \geq 1 - \beta_j \, .
\]
The cartesian product set is given by
\[
    \Cset = \bigtimes_{j=1}^{k} \Cind.
    \label{eq:cset-product-appendix}
\]

\begin{proposition*}[Marginal coverage of Product Set~\ref{propr:marginal-coverage-cprod}], %
\label{prop:conformality-individual}
For any NeSy CBM $(f,g)$, let $0 < \beta_j < 1$ and $0 < \beta < 1$ such that $\beta = \sum_{j=1}^{k} \beta_j$. If $\Cind$, for $j \in [1, \dots, k]$, are concept conformal sets with miscoverage level $\beta_j$, respectively~(\cref{eq:concept-only-guarantee}), then let $\Cprod$ be the set constructed as in~\cref{eq:product-concept-set} based on the $k$ sets $\Cind$. Then, for every $(\vx, \vc^*, \vy^*) \in \mathrm{supp}(p^*(\vx, \vc, \vy))$:
\[
    \textstyle
    \Pr(\vc^* \in \Cprod) \ge 1 - \beta.
\]

\begin{proof}
By construction, we have $\vc^* \in \Cprod$ if and only if $\vc^*_j \in \Cind$ for all $j \in [k]$. 
Therefore, the probability that the joint vector $\vc^*$ does not belong to $\Cprod$ can be written as %
\[
    \Pr(\vc^* \notin \Cprod) 
    = \Pr\left(\bigcup_{j=1}^k 
    \{ \vc^*_j \notin \Cind \}
    \right). 
\]

Applying the union bound (Boole's inequality) yields
\[
    \Pr\left(\bigcup_{j=1}^k \{ \vc^*_j \notin \Cind\} \right)
    \le \sum_{j=1}^k \Pr( \vc^*_j \notin \Cind)
    \le \sum_{j=1}^k \beta_j,
\]
where the last inequality uses the individual coverage guarantees. Subtracting from 1, we obtain
\[
    \Pr(\vc^* \in \Cprod) = 1 - \Pr(\vc^* \notin \Cprod) \ge 1 - \sum_{i=j}^k \beta_j.
\]

\end{proof}
\end{proposition*}

\begin{corollary}[Bonferroni's correction]
\label{cor:bonferroni-coverage}
If we desire joint coverage at level $1 - \beta$, using Bonferroni's correction, \ie setting $\beta_j = \beta / k$ for all $j$ guarantees
\[
    \Pr(\vc^* \in \Cprod) \ge 1 - \sum_{i=1}^k \frac{\beta}{k} = 1 - \beta.
\]
This follows directly from the proposition above. \qed
\end{corollary}

\subsection{Coverage proof of \ConceptsPlusDeduction~(\cref{prop:nesy-forward-coverage})}
\label{proof:nesy-forward-coverage}

\begin{proposition*}[Marginal coverage of \ConceptsPlusDeduction~\cref{prop:nesy-forward-coverage}]
For $0<\beta<1$ and for any $(\vx, \vc^*, \vy^*)$ in  the support of $p^*(\vx, \vc, \vy)$, suppose $\Cset$ satisfies \cref{eq:concept-set-constraint-beta-appendix}. Let $\Yfwset := g^{\dagger}(\Cset)$. Then
\[
    (1 - \beta) \delta_\mathrm{de} \le \Pr(\vy^* \in \Yfwset) \le \beta + \delta_\mathrm{de}
\]
where $\delta_\mathrm{de} := \Pr(\vy^* \in \Yfwset \mid \vc^* \in \Cset)$ is the probability that deduction recovers the ground-truth labels $\vy^*$ given that the concept set $\Cset$ contains the correct concepts $\vc^*$.  Moreover, if the NeSy-CBM satisfies deductive soundness (\cref{def:soundness}), the bound tightens to:
\[
    \Pr(\vy^* \in \Yfwset) \ge 1 - \beta.
\]
\end{proposition*}

\begin{proof}

Let the two binary events $A = (\vc^* \in \Cset)$ and $B = (\vy^* \in \Yfwset)$.
We first derive the lower bound. By the law of total probability~\cref{eq:law-of-total-probability}
\begin{align}
    \Pr(B)
        &= \Pr(B \mid A)\Pr(A) + \Pr(B \mid \neg A)\Pr(\neg A) \label{eq:law-of-total-probability-fw}
    \\
        & \geq \Pr(B \mid A)\Pr(A)
        \label{eq:prob-b-coverage-abduction-2}
\end{align}
Since $\Cset$ is a conformal concept set with miscoverage $\beta$ by \cref{eq:concept-set-constraint-beta-appendix}, we have that
\[
    \Pr(A) = \Pr(\vc^* \in \Cset) \geq 1 - \beta.
\]
We can plug this result into the previous one, \ie $\Pr(B) \geq \Pr(B \mid A)\Pr(A)$~(\cref{eq:prob-b-coverage-abduction-2}), to obtain
\begin{align}
    \Pr(B) 
    &\ge 
    \Pr(B \mid A)(1-\beta) \label{eq:proof-de-lower}
\end{align}
Which proofs the lower-bound:
\begin{align}
    \Pr(\vy^* \in \Yfwset)
    &\ge
    \Pr(\vy^* \in \Yfwset \mid \vc^* \in \Cset)(1 - \beta).
\end{align}
Next, we prove the upper bound. Consider again the right-hand side of \cref{eq:law-of-total-probability-fw}. Since $\Pr(B \mid \neg A) \le 1$ and $\Pr(\neg A) \le \beta$, we have:
\begin{equation} \label{eq:proof-de-upper-right}
    \Pr(B \mid \neg A) \Pr(\neg A) \le \beta.
\end{equation}
Moreover, $\Pr(A) \le 1$. Hence,
\begin{equation} \label{eq:proof-de-upper-left}
    \Pr(B \mid A) \Pr(A) \le \Pr(B \mid A).
\end{equation}
Combining these observations yields:
\begin{align}  \label{eq:proof-de-upper}
    \Pr(B) 
    &= \Pr(B \mid A)\Pr(A) + \Pr(B \mid \neg A)\Pr(\neg A) \tag{Law of total probability - ~\cref{eq:law-of-total-probability}} \\
    &\le \Pr(B \mid A) + \Pr(B \mid \neg A)\Pr(\neg A) \quad \tag{Substitute the left term - \cref{eq:proof-de-upper-left}} \\
    &\le \Pr(B \mid A) + \beta \quad \tag{Substitute the right term - \cref{eq:proof-de-upper-right}}
\end{align}
Substituting the expression for $A$ and $B$, we have the upper bound
\[
    \Pr(\vy^* \in \Yfwset) \leq \beta + \delta_\mathrm{de} 
\]
Under \emph{deductive soundness} (\cref{def:soundness}), by~\cref{prop:nesy-forward-coverage}, we have $\delta_\mathrm{de} = 1$. Hence, we get
\[
    \Pr(B) \geq \delta_\mathrm{de} (1 - \beta) \geq 1 - \beta.
\]
While for the upper bound~\cref{eq:proof-de-upper-right}:
\[
    \Pr(B) \le \delta_\mathrm{de} + \beta \leq \min(1, \delta_\mathrm{de} + \beta).
\]
This concludes the proof.

\end{proof}

\subsection{Coverage proofs of \method concept set~(\cref{prop:joint-coverage})}

\begin{proposition*}[Marginal coverage of $\Crev$~(\cref{prop:joint-coverage})]
\label{prop:revised-concept-using-knowledge}
For $0<\alpha<1$ and $0<\beta<1$, and for any $(\vx, \vc^*, \vy^*)$ in  the support of $p^*(\vx, \vc, \vy)$, suppose $\Yset$ satisfies \cref{eq:label-set-constraint-alpha-appendix} and $\Cset$ satisfies \cref{eq:concept-set-constraint-beta-appendix}. Let $\Yfwset := g^{\dagger}(\Cset)$ and $\Cab := g^{-\dagger}(\Yset)$. Define the \emph{\textbf{refined concept set}} as $\Crev := \Cset \cap \Cab$. Then
\begin{align}
    \Pr(\vc^* \in \Crev)
    &\ge (1-\alpha)\,\Pr(\vc^* \in \Cab \mid \vy^* \in \Yset) - \beta \notag\\
    &\quad + \Pr(\vc^* \notin \Cset \land \vc^* \notin \Cab)
\end{align}
\end{proposition*}

\begin{proof}

Let $A := (\vc^* \in \Cset)$ and $B := (\vc^* \in \Cab)$, so that
\[
    \Pr(\vc^* \in \Crev) = \Pr(A \cap B).
\]
By the inclusion-exclusion identity, we have
\[
    \Pr(A \cap B) = \Pr(A) + \Pr(B) - \Pr(A \cup B).\label{eq:inclusion-exclusion}
\]
Since $\Pr(A \cup B) = 1 - \Pr(\neg A \cap \neg B)$, this is equivalent to
\begin{equation}
    \Pr(A \cap B) = \Pr(A) + \Pr(B) - 1 + \Pr(\neg A \cap \neg B).
    \label{eq:ie-crev}
\end{equation}
We know that $\Pr(A) \ge 1 - \beta$ follows from~\cref{eq:concept-set-constraint-beta-appendix}, and that $\Pr(B) \ge (1-\alpha)\delta_{\mathrm{ab}}$ follows from~\cref{prop:nesy-abduction-coverage}. Therefore, we can proceed to obtain the lower bound as follows:
\begin{align}
    \Pr(\vc^* \in \Crev)
    &= \Pr(A \cap B)
    \tag{def.\ of $\Crev$} \\
    &= \Pr(A) + \Pr(B) - 1 + \Pr(\neg A \cap \neg B)
    \tag{inclusion--exclusion~\cref{eq:ie-crev}} \\
    &\ge (1-\beta) + \Pr(B) - 1 + \Pr(\neg A \cap \neg B)
    \tag{\cref{eq:concept-set-constraint-beta-appendix}} \\
    &\ge (1-\beta) + (1-\alpha)\delta_{\mathrm{ab}} - 1 + \Pr(\neg A \cap \neg B)
    \tag{\cref{prop:nesy-abduction-coverage}} \\
    &= (1-\alpha)\delta_{\mathrm{ab}} - \beta
      + \Pr(\vc^* \notin \Cset \land \vc^* \notin \Cab).
    \tag{subst. $\Pr(\neg A \cap \neg B)$}
\end{align}

Under \emph{abductive soundness}~(\cref{def:soundness}), we have
$\delta_{\mathrm{ab}} = 1$.
Hence, we get
\[
    \Pr(\vc^* \in \Crev) \ge 1 - \alpha - \beta
    + \Pr(\vc^* \notin \Cset \land \vc^* \notin \Cab).
\]

This concludes the proof.

\end{proof}

\begin{remark}
The term $\Pr(\vc^* \notin \Cset \land \vc^* \notin \Cab)$ 
is the probability that the ground-truth concept $\vc^*$ is not captured by \emph{neither} the concept conformal set ($\Cset$) \emph{nor} the abduced concept set ($\Cab$).

\end{remark}

\subsection{Coverage proofs of \method label set~(\cref{prop:joint-coverage})}

\begin{proposition*}[Marginal coverage of $\Yrev$~(\cref{prop:joint-coverage})]
\label{prop:revised-label-forward}
For $0<\alpha<1$ and $0<\beta<1$, and for any $(\vx, \vc^*, \vy^*)$ in  the support of $p^*(\vx, \vc, \vy)$, suppose $\Yset$ satisfies \cref{eq:label-set-constraint-alpha-appendix} and $\Cset$ satisfies \cref{eq:concept-set-constraint-beta-appendix}. Let $\Yfwset := g^{\dagger}(\Cset)$ and $\Cab := g^{-\dagger}(\Yset)$. Define the \emph{\textbf{refined label set}} as $\Yrev = \Yset \cap \Yfwset$. Then:
\begin{align}
    \Pr(\vy^* \in \Yrev)
    &\ge (1-\beta)\,\Pr(\vy^* \in \Yfwset \mid \vc^* \in \Cset) - \alpha \notag\\
    &\quad + \Pr(\vy^* \notin \Yset \land \vy^* \notin \Yfwset).
\end{align}
\end{proposition*}

\begin{proof}
Let $A := (\vy^* \in \Yset)$ and $B := (\vy^* \in \Yfwset)$, so that
\[
    \Pr(\vy^* \in \Yrev) = \Pr(A \cap B).
\]
By the inclusion--exclusion identity,
\[
    \Pr(A \cap B) = \Pr(A) + \Pr(B) - \Pr(A \cup B).
\]
Since $\Pr(A \cup B) = 1 - \Pr(\neg A \cap \neg B)$, this is equivalent to
\begin{equation}
    \Pr(A \cap B) = \Pr(A) + \Pr(B) - 1 + \Pr(\neg A \cap \neg B).
    \label{eq:ie-yrev}
\end{equation}
We know that $\Pr(A) \ge 1 - \alpha$ follows from~\cref{eq:label-set-constraint-alpha-appendix}, and that $\Pr(B) \ge (1-\beta)\delta_{\mathrm{de}}$ follows from~\cref{prop:nesy-forward-coverage}. Therefore:
\begin{align}
    \Pr(\vy^* \in \Yrev)
    &= \Pr(A \cap B)
    \tag{def.\ of $\Yrev$} \\
    &= \Pr(A) + \Pr(B) - 1 + \Pr(\neg A \cap \neg B)
    \tag{inclusion--exclusion~\cref{eq:ie-yrev}} \\
    &\ge (1-\alpha) + \Pr(B) - 1 + \Pr(\neg A \cap \neg B)
    \tag{\cref{eq:label-set-constraint-alpha-appendix}} \\
    &\ge (1-\alpha) + (1-\beta)\delta_{\mathrm{de}} - 1 + \Pr(\neg A \cap \neg B)
    \tag{\cref{prop:nesy-forward-coverage}} \\
    &=  (1 - \beta)\delta_{\mathrm{de}} - \alpha 
      + \Pr(\vy^* \notin \Yset \land \vy^* \notin \Yfwset).
    \tag{subst. $\Pr(\neg A \cap \neg B)$}
\end{align}
Under \emph{deductive soundness}~(\cref{def:soundness}), we have $\delta_{\mathrm{de}} = 1$. Hence, we get
\[
    \Pr(\vy^* \in \Yrev) \ge 1 - \alpha - \beta
    + \Pr(\vy^* \notin \Yset \land \vy^* \notin \Yfwset).
\]
This concludes the proof.
\end{proof}

\begin{remark}
The term $\Pr(\vy^* \notin \Yset \land \vy^* \notin \Yfwset)$ is the probability that the ground-truth label $\vy^*$ is captured by \emph{neither} the label conformal set ($\Yset$) \emph{nor} the deduced label set ($\Yfwset$). 
\end{remark}

\subsection{Fixed point proof of \method~(\cref{prop:joint-fixed-point})} 

To prove that \method does not require further iterations, let $R$ be the operator that \emph{revises} the label and concept sets. Formally:
\[
    R: (\Gamma, \Upsilon) \mapsto (R(\Gamma), R(\Upsilon)) \subseteq \calC \times \calY ,
    \qquad
    R(\Gamma, \Upsilon) := \bigl(\Gamma \cap g^{-\dagger}(\Upsilon), \Upsilon \cap g^{\dagger}(\Gamma)\bigr).
    \label{eq:def-revision-map}
\]
By construction, \method applies $R$ \textbf{only once} to the pair $(\Cset, \Yset)$, producing 
\[
    (\Crev, \Yrev) = R(\Cset, \Yset)
\]

Before introducing the main proposition, we show two useful lemmas.

\begin{lemma}[Set behavior of $g^{-\dagger}$ on intersection]
\label{lemma:g-inverse-distributes}
For any $\calA, \calB \subseteq \calY$,
\[
    g^{-\dagger}(\calA \cap \calB) \;\subseteq\; g^{-\dagger}(\calA) \cap g^{-\dagger}(\calB).
\]
The reverse inclusion does \emph{not} hold in general for $g^\dagger$.
\end{lemma}

\begin{proof}
Let $\vc \in g^{-\dagger}(\calA \cap \calB)$, there exists $\vy \in g^\dagger(\vc) \cap (\calA \cap \calB)$. Then $\vy \in g^\dagger(\vc) \cap \calA$, so $\vc \in g^{-\dagger}(\calA)$, and likewise $\vc \in g^{-\dagger}(\calB)$. Hence $\vc \in g^{-\dagger}(\calA) \cap g^{-\dagger}(\calB)$.

The reverse inclusion fails when ties are present: if $g^\dagger(\vc) = \{y_1, y_2\}$ with $y_1 \in \calA \setminus \calB$ and $y_2 \in \calB \setminus \calA$, then $\vc \in g^{-\dagger}(\calA) \cap g^{-\dagger}(\calB)$ but $\vc \notin g^{-\dagger}(\calA \cap \calB)$.
\end{proof}

\begin{lemma}%
\label{lemma:round-trip}
For any $\Gamma' \subseteq \calC$,
\[
    \Gamma' \subseteq g^{-\dagger}\bigl(g^{\dagger}(\Gamma')\bigr).
\]
\label{eq:g-transitivity}
\end{lemma}
\begin{proof}
Let $\vc \in \Gamma'$. By definition of the image, $g^{\dagger}(\vc) \in g^{\dagger}(\Gamma')$. By \cref{eq:inverse-g-appendix},
\[
    g^{-\dagger}\bigl(g^{\dagger}(\Gamma')\bigr)
    = \bigcup_{\vy \in g^{\dagger}(\Gamma')} \{\vc' \in \calC : \vy \in \argmax g(\vc')\},
\]
so for any $\vc' \in \calC$,
\[
    g^{\dagger}(\vc') \in g^{\dagger}(\Gamma')
    \implies
    \vc' \in g^{-\dagger}\bigl(g^{\dagger}(\Gamma')\bigr).
\]
In particular, $\Gamma' \subseteq \vg^{-\dagger}  (\vg^\dagger (\Gamma'))$. 
\qedhere
\end{proof}

Now we can proceed with the main proposition.

\begin{proposition}[\method reaches a fixed point]
\label{prop:joint-fixed-point}
For a NeSy-CBM $(f,g)$, let $R$ be defined as \cref{eq:def-revision-map}. Then, for any $\Gamma \subseteq \calC$ and $\Upsilon \subseteq \calY$,
\[
    R\bigl(R(\Gamma, \Upsilon)\bigr) = R(\Gamma, \Upsilon) \, .
\]
In particular, $(\Crev, \Yrev)$ is a \textbf{\emph{fixed point}} of $R$, and no iterative refinement is needed.
\end{proposition}

\begin{proof}
Let $(\Gamma_1, \Upsilon_1) = R(\Gamma, \Upsilon)$. By definition of $R$, these are
\[
    \Gamma_1 = \Gamma \cap g^{-\dagger}(\Upsilon),
    \qquad
    \Upsilon_1 = \Upsilon \cap g^\dagger (\Gamma) \, .
\]
Now, denote with $(\Gamma_2, \Upsilon_2)$ the sets obtained by applying $R$ twice on $\Gamma, \Upsilon$, namely
\[
(\Gamma_2, \Upsilon_2) = R(\Gamma_1, \Upsilon_1) = R(R(\Gamma, \Upsilon)) \, ,
\]
which are given by
\[
    \Gamma_2 = \Gamma_1 \cap g^{-\dagger}(\Upsilon_1),
    \qquad
    \Upsilon_2 = \Upsilon_1 \cap g^\dagger(\Gamma_1).
\]

Now, to proceed with the proof, we use the standard set identity: for any two sets $S, T$,
\begin{equation}
    S \subseteq T \iff S \cap T = S.
    \label{eq:absorption}
\end{equation}
This means that, following~\cref{eq:absorption}:
\begin{align}
 &(I) &\Gamma_2 = \Gamma_1 \iff \Gamma_1 \subseteq g^{-\dagger}(\Upsilon_1) = g^{-\dagger}(\Upsilon \cap g^{\dagger}(\Gamma)) \label{eq:fp-incl-I} \\ &(II) &\Upsilon_2 = \Upsilon_1 \iff \Upsilon_1 \subseteq g^\dagger(\Gamma_1) = g^{\dagger}(\Gamma \cap g^{-\dagger}(\Upsilon)). \label{eq:fp-incl-II}
\end{align}

To prove the fixed point, we prove these two inclusions in turn.

\textbf{Proof of (I): $\Gamma_1 \subseteq g^{-\dagger}\bigl(\Upsilon \cap g^{\dagger}(\Gamma)\bigr)$}. 

Let $\vc \in \Gamma_1$, this means that, since $\Gamma_1 = \Gamma \cap g^{-\dagger}(\Upsilon)$, $\vc \in \Gamma$ and $\vc \in g^{-\dagger}(\Upsilon)$. We prove $\vc \in g^{-\dagger}(\Upsilon_1)$, \ie that $g^\dagger(\vc) \cap \Upsilon_1 \neq \emptyset$.

By definition of $\Gamma_1 = \Gamma \cap g^{-\dagger}(\Upsilon)$, the concept $\vc$ satisfies
\begin{enumerate}
    \item[(a)] $\vc \in \Gamma$, hence $g^\dagger(\vc) \subseteq g^\dagger(\Gamma)$;
    \item[(b)] $\vc \in g^{-\dagger}(\Upsilon)$, hence $g^\dagger(\vc) \cap \Upsilon \neq \emptyset$.
\end{enumerate}

Now, let $\vy$ be any $\vy \in (g^\dagger(\vc) \cap \Upsilon) \neq \emptyset$ from (b). By (a), $\vy \in g^\dagger(\Gamma)$ as well. Therefore 
\begin{align}
    &\vy \in g^\dagger(\vc) \cap \Upsilon \cap g^\dagger(\Gamma) \\
    &\vy \in g^\dagger(\vc) \cap \Upsilon_1 \tag{from $\Upsilon_1 = \Upsilon \cap g^\dagger(\Gamma)$}
\end{align}

since $g^{\dagger}(\vc) \cap \Upsilon_1 \neq \emptyset$, this concludes the proof of the first inclusion (I)~\cref{eq:fp-incl-I}.

\textbf{Proof of (II): $\Upsilon_1 \subseteq g^{\dagger}(\Gamma_1) \subseteq g^{\dagger}\bigl(\Gamma \cap g^{-\dagger}(\Upsilon)\bigr)$}. 
This proof does not follow the steps of  (I): 
unlike $g^{-\dagger}$, the image $g^{\dagger}$ does not distribute over intersection in general, so we cannot mirror the argument above. We proceed elementwise.
 
Let $\vy \in \Upsilon_1 = \Upsilon \cap g^{\dagger}(\Gamma)$. Then:
\begin{itemize}
    \item $\vy \in \Upsilon$, and
    \item $\vy \in g^{\dagger}(\Gamma)$, so by definition of the image there exists some $\vc \in \Gamma$ such that $\vy \in g^\dagger (\vc)$. %
\end{itemize}
Let us then fix $\vc$ such that $\vy \in g^\dagger (\vc)$.%
Now, we   verify that $\vc \in \Gamma_1 = \Gamma \cap g^{-\dagger}(\Upsilon)$:
\begin{itemize}
    \item $\vc \in \Gamma$ holds by the choice of $\vc$;
    \item $g^{\dagger}(\vc) = \vy \in \Upsilon$, so by the definition of \cref{eq:inverse-g-appendix}:
    \[
        g^{-\dagger}(\Upsilon) = \bigcup_{\vy \in \calC} \{
        \vc \in \Gamma: \vy \in \argmax g(\vc)
        \}
    \]
    We can conclude that $\vc \in g^{-\dagger}(\Upsilon)$.
\end{itemize}
Hence $\vc \in \Gamma_1$ and $g^{\dagger}(\vc) = \vy$, which means $\vy \in g^{\dagger}(\Gamma_1)$, proving the second inclusion (II) \cref{eq:fp-incl-II}.
 
\textbf{Conclusion}. By (I), $\Gamma_2 = \Gamma_1 \cap g^{-\dagger}(\Upsilon_1) = \Gamma_1$. By (II), $\Upsilon_2 = \Upsilon_1 \cap g^{\dagger}(\Gamma_1) = \Upsilon_1$. Therefore $R\bigl(R(\Gamma, \Upsilon)\bigr) = (\Gamma_1, \Upsilon_1) = R(\Gamma, \Upsilon)$, i.e.\ $(\Gamma_1, \Upsilon_1)$ is a fixed point of $R$. 

This concludes the proof.
\end{proof}

\subsection{Proof of the Optimality of \method}

\begin{proposition}[Optimality of \method]
\label{prop:joint-optimality}

Let $\Cset \subseteq \calC$ and $\Yset \subseteq \calY$ be arbitrary prediction sets, and let $(\Crev, \Yrev) = R(\Cset, \Yset)$ be their joint revision. 

Then $(\Crev, \Yrev)$ is the largest pair $(\Gamma', \Upsilon') \subseteq \calC \times \calY$ satisfying:
\begin{enumerate}
    \item \emph{Containment:} $\Gamma' \subseteq \Cset$ and $\Upsilon' \subseteq \Yset$.
    \item \emph{Mutual coverage:}
    \begin{enumerate}
        \item every concept entails at least one retained label:
              $\forall \vc \in \Gamma',\; g^\dagger(\vc) \cap \Upsilon' \neq \emptyset$
              (equivalently, $\Gamma' \subseteq g^{-\dagger}(\Upsilon')$);
        \item every label is entailed by at least one retained concept:
              $\forall \vy \in \Upsilon',\; \exists \vc \in \Gamma' \text{ with } \vy \in g^\dagger(\vc)$
              (equivalently, $\Upsilon' \subseteq g^\dagger(\Gamma')$).
    \end{enumerate}
\end{enumerate}
Consequently,
\[
    |\Crev| \le |\Cset|,
    \qquad\qquad
    |\Yrev| \le |\Yset|.
\]
Moreover, $(\Crev, \Yrev)$ is a minimal revision of $(\Cset, \Yset)$ that enforces logical consistency with $g$.
\end{proposition}
 
\begin{proof}
 
\textbf{Step 1: $(\Crev, \Yrev)$ itself satisfies (1) and (2)}.
 
By definition,
\[
    \Crev = \Cset \cap g^{-\dagger}(\Yset) \subseteq \Cset
\]
and 
\[
\Yrev = \Yset \cap g^{\dagger}(\Cset) \subseteq \Yset,
\]
therefore (1) holds.
 
\emph{Mutual coverage (2a): $\Crev \subseteq g^{-\dagger}(\Yrev)$.}

By definition, $(\Crev, \Yrev) = R(\Cset, \Yset)$, so $\Crev = \Cset \cap g^{-\dagger}(\Yset)$ plays the role of $\Gamma_1$ and $\Yrev = \Yset \cap g^{\dagger}(\Cset)$ plays the role of $\Upsilon_1$ in the proof of~\cref{prop:joint-fixed-point}. Inclusion~(I) of that proof established $\Gamma_1 \subseteq g^{-\dagger}(\Upsilon_1)$; specialising to the present sets gives $\Crev \subseteq g^{-\dagger}(\Yrev)$.

\emph{Mutual coverage (2b): $\Yrev \subseteq g^\dagger(\Crev)$.} 

Recall that $\Yrev = \Yset \cap g^{\dagger}(\Cset)$ and $\Crev = \Cset \cap g^{-\dagger}(\Yset)$. 

Let $\vy \in \Yrev$. 
Since $\vy \in g^\dagger(\Cset)$, by definition of the image there is some $\vc \in \Cset$ with $\vy \in g^\dagger(\vc)$. This $\vc$ lies in $\Crev$ by definition of $\Crev$. Using $\vy \in \Yrev \subseteq \Yset$ from (1),
\[
    \vy \in g^\dagger(\vc) \cap \Yset \;\neq\; \emptyset
    \quad\Longrightarrow\quad
    \vc \in g^{-\dagger}(\Yset),
\]
hence $\vc \in \Cset \cap g^{-\dagger}(\Yset) = \Crev$.

\medskip
\textbf{Step 2: Any pair satisfying (1) and (2) is contained in $(\Crev, \Yrev)$}.
 
Let $(\Gamma', \Upsilon')$ satisfy (1) and (2). We must show that $\Gamma' \subseteq \Crev$ and $\Upsilon' \subseteq \Yrev$.

\emph{Show $\Gamma' \subseteq \Crev$.} Let $\vc \in \Gamma'$. We prove $\vc \in \Crev$ by 
checking the two conditions defining $\Crev = \Cset \cap g^{-\dagger}(\Yset)$.

\begin{itemize}
    \item[(i)] $\vc \in \Cset$: this follows directly from $\vc \in \Gamma' \subseteq \Cset$ by~(1).
    \item[(ii)] $\vc \in g^{-\dagger}(\Yset)$: by definition, this requires
    $g^{\dagger}(\vc) \cap \Yset \neq \emptyset$. From~(2a), 
    $g^{\dagger}(\vc) \cap \Upsilon' \neq \emptyset$, and from~(1), $\Upsilon' \subseteq \Yset$.
    Hence
    \[
        \emptyset \;\neq\; g^{\dagger}(\vc) \cap \Upsilon'
              \;\subseteq\; g^{\dagger}(\vc) \cap \Yset.
    \]
\end{itemize}
Combining (i) and (ii) gives $\vc \in \Cset \cap g^{-\dagger}(\Yset) = \Crev$.

\medskip
\emph{Show $\Upsilon' \subseteq \Yrev$.} Let $\vy \in \Upsilon'$. We prove $\vy \in \Yrev$ by
checking the two conditions defining $\Yrev = \Yset \cap g^{\dagger}(\Cset)$.

\begin{itemize}
    \item[(i)] $\vy \in \Yset$: this follows directly from $\vy \in \Upsilon' \subseteq \Yset$ by~(1).
    \item[(ii)] $\vy \in g^{\dagger}(\Cset)$: by definition of the image, this requires some
    $\vc \in \Cset$ with $\vy \in g^{\dagger}(\vc)$. From~(2b), there exists $\vc \in \Gamma'$
    with $\vy \in g^{\dagger}(\vc)$, and from~(1), $\vc \in \Gamma' \subseteq \Cset$.
\end{itemize}
Combining (i) and (ii) gives $\vy \in \Yset \cap g^{\dagger}(\Cset) = \Yrev$.

This concludes the proof.
\end{proof}

\subsection{Coverage proof of Product Set leveraging e-values}

\begin{proposition}[Concept coverage via aggregated e-values]
\label{prop:e-world-coverage}

For $0 < \beta < 1$ and any $(\vx, \vc^*) \in \calD_{\vC}$,
let $\vc = (\vc_1, \dots, \vc_K) \in \calC = \calC_1 \times \cdots \times \calC_K$ be a concept vector, where each $\vc_j \in \{0, \dots, V-1\}$. For each concept $j$, let $s_j : \calX \times \calC_j \to \mathbb{R}_{\geq 0}$ be a non-negative conformity score and let
$\calD_j = \{(\vx_i, \vc_{ij}^*)\}_{i=1}^n$ be a calibration set.
Define the \emph{soft-rank e-value}~\citep{wang2022false} for concept $j$ and candidate $v$ as
\[
    E_j(\vx, \vv)
    =
    \frac{(n+1)\, s_j(\vx, v)}
    {\displaystyle\sum_{i=1}^n s_j(\vx_i, c_{ij}^*) + s_j(\vx, \vv)}.
\]
the \emph{joint e-value} for a candidate tuple $\vv = (v_1, \dots, v_K)$ as
\[
    E(\vx, \vv) = \frac{1}{K} \sum_{j=1}^{K} E_j(\vx, v_j).
\]
and the \emph{concept prediction set} for threshold $\beta > 0$ as
\[
    \Cset = \left\{ 
        \vv \in \bigtimes_{j=1}^K \{0,\dots,V-1\} : 
            E(\vx, \vv) < \frac{1}{\beta} 
        \right\}.
\]
Then:
\[
    \Pr\left(\vc^* \in \Cset\right) \geq 1 - \beta.
\]
In particular, this guarantee holds without Bonferroni correction (\cref{cor:bonferroni-coverage}) and does not degrade with the number of concepts $K$.

\end{proposition}

\begin{proof}

\textbf{Step 1: each $E_j(\vx, \vc_j^*)$ is a valid e-variable}.
By exchangeability of the calibration data and the test point, each soft-rank e-value satisfies $\mathbb{E}[E_j(\vx, \vc_j^*)] \leq 1$~\citet{vovk2025conformal, gauthier2025values}.

\textbf{Step 2: the average $E(\vx, \vc) = \frac{1}{K} \sum_{j=1}^{K} E_j(\vx, \vc_j)$ is a valid e-variable}.
By linearity of expectation:
\[
    \mathbb{E}\left[E(\vx, \vc^*)\right]
    =
    \mathbb{E}\!\left[\frac{1}{K}\sum_{j=1}^K E_j(\vx, \vc_j^*)\right]
    =
    \frac{1}{K}\sum_{j=1}^K \mathbb{E}\!\left[E_j(\vx, \vc_j^*)\right]
    \leq
    \frac{1}{K} \cdot K
    =
    1,
\]
where the inequality uses Step~1 termwise. Hence $E(\vx, \vc^*)$ is a valid e-variable.

\textbf{Step 3: Coverage follows by Markov's inequality~\citep{ramdas2025hypothesis}}.
Since $E(\vx, \vc) \geq 0$ and $\mathbb{E}[E(\vx, \vc)] \leq 1$, Markov's inequality gives:
\[
    \Pr\left(E(\vx, \vc) \geq \frac{1}{\beta}\right) \leq \beta,
\]
or equivalently,
\[
    \Pr\left(E(\vx, \vc) < \frac{1}{\beta}\right) \geq 1 - \beta.
\]
Since $\vc \in \Cset$ if and only if $E(\vx, \vc) < 1/\beta$, the coverage guarantee follows.
\end{proof}

\begin{remark}[Comparison with Bonferroni correction (\cref{cor:bonferroni-coverage})]
\cref{propr:marginal-coverage-cprod} applies Bonferroni's correction, allocating miscoverage $\beta / K$ to each individual concept and requiring $K$ separate thresholds.
The e-value approach in \cref{prop:e-world-coverage} avoids this entirely: independence is exploited at the level of expectations via the product factorization in Step~2, so a single threshold $1/\beta$ suffices and the guarantee does not degrade with $K$.
\end{remark}

\begin{remark}[Average vs.\ product aggregation]
\label{rem:mean-vs-product}
An alternative aggregation is the \emph{product} $E(\vx, \vc) = \prod_{j=1}^{K} E_j(\vx, \vc_j)$. 
To make it a valid e-variable, it requires:
\[
    \mathbb{E}\!\left[\prod_{j=1}^K E_j(\vx, \vc_j^*)\right]
    =
    \prod_{j=1}^K \mathbb{E}\!\left[E_j(\vx, \vc_j^*)\right],
\]
which is \emph{not} implied by the conditional concept independence $\vC_i \indep \vC_j \mid \vX$ standard in NeSy-CBMs~\citep{vankrieken2025neurosymbolic, van2024independence}, because the per-concept e-values share the same test input $\vx$ and the same calibration set $\calD_{\mathrm{cal}}$. The product is a valid e-variable if one additionally assumes \emph{conditional} independence of the per-concept e-values given $(\vx, \calD_{\mathrm{cal}})$, which might not hold in pratice. Despite this, \EMethod also behaves well when using the product, as shown in our experiments for \CHX (see~\cref{sec:product-vs-average}).
\end{remark}

\subsection{Coverage proof of \method leveraging e-values}
\label{sec:proof-eval}

\begin{proposition}[Joint coverage of \method using e-values]
\label{prop:evalues-joint-coverage}
For $0 < \tilde{\alpha}, \tilde{\beta}$ possibly data-dependent, let $\Crev$ and $\Yrev$ be the jointly revised sets of \method (\cref{eq:y-c-rev}). Suppose the concept set $\Cset$ and the label set $\Yset$ are constructed using e-values so that
\[
    \Pr(\vc^* \notin \Cset) \le \bbE[\tilde{\beta}],
    \qquad
    \Pr(\vy^* \notin \Yset) \le \bbE[\tilde{\alpha}].
\]
If the NeSy-CBM satisfies both deductive and abductive soundness
(\cref{def:soundness}), then
\begin{align}
    &\Pr(\vc^* \in \Crev)
    \ge 1 - \bbE[\tilde{\alpha}] - \bbE[\tilde{\beta}]
    + \Pr(\vc^* \notin \Cset \land \vc^* \notin \Cab),\\
    &\Pr(\vy^* \in \Yrev)
    \ge 1 - \bbE[\tilde{\alpha}] - \bbE[\tilde{\beta}]
    + \Pr(\vy^* \notin \Yset \land \vy^* \notin \Yfwset).
\end{align}
\end{proposition}
 
\begin{proof}
We prove the bound for $\Crev$; the argument for $\Yrev$ is identical by symmetry.
 
Let $A := (\vc^* \in \Cset)$ and $B := (\vc^* \in \Cab)$, so $\Pr(\vc^* \in \Crev) = \Pr(A \cap B)$.
By inclusion--exclusion,
\begin{equation}
    \Pr(A \cap B) = \Pr(A) + \Pr(B) - 1 + \Pr(\neg A \cap \neg B).
    \label{eq:ie-evalues}
\end{equation}
Under abductive soundness, \cref{prop:abductive-soundness-implies-delta-ab} gives $\delta_{\mathrm{ab}} = 1$, so $\vc^* \in \Cab = g^{-\dagger}(\Yset)$ whenever $\vy^* \in \Yset$. Hence $\Pr(B) \ge 1 - \bbE[\tilde{\alpha}]$.

Using also $\Pr(A) \ge 1 - \bbE[\tilde{\beta}]$ and substituting
into~\cref{eq:ie-evalues},
\begin{align}
    \Pr(\vc^* \in \Crev)
    &= \Pr(A) + \Pr(B) - 1 + \Pr(\neg A \cap \neg B)
    \tag{\cref{eq:ie-evalues}} \\
    &\ge (1 - \bbE[\tilde{\beta}]) + (1 - \bbE[\tilde{\alpha}]) - 1
      + \Pr(\neg A \cap \neg B)
    \tag{e-value guarantees} \\
    &= 1 - \bbE[\tilde{\alpha}] - \bbE[\tilde{\beta}]
      + \Pr(\vc^* \notin \Cset \land \vc^* \notin \Cab).
    \tag{simplify, $\delta_{\textrm{de}}$}
\end{align}

This concludes the proof.
\end{proof}
 
\begin{remark}[Relation to fixed miscoverage levels]
When $\tilde{\alpha}$ and $\tilde{\beta}$ are fixed constants independent of the data, $\bbE[\tilde{\alpha}] = \alpha$ and $\bbE[\tilde{\beta}] = \beta$, and the bound recovers the guarantee of \cref{prop:joint-coverage} with the additional term $\Pr(\vc^* \notin \Cset \land \vc^* \notin \Cab)$~\citep{ramdas2025hypothesis}. If instead, \emph{abductive} and \emph{deductive} soundness do not hold, the bound can be recovered the same way as~\cref{prop:joint-coverage}.
\end{remark}

\section{Implementation Details}
\label{sec:implementation-details}

Here, we provide additional details about all metrics, datasets, and models useful for reproducibility.

\subsection{Implementation}

All experiments were conducted using Python 3.9 and PyTorch 2.5.1, and were run on a single NVIDIA A100 GPU. The implementations of \DPL and \LTN for \MNISTAdd, \MNISTHALF, and \MNISTEO were taken from~\citep{bortolotti2024benchmark}, while they were developed from scratch for all other datasets.

The \MNISTAdd variants were constructed starting from the original \MNIST images~\citep{lecun1998mnist}. The \CHX dataset was taken verbatim from~\citep{pugnana2025ask}. For \RIVAL, we used the images contained in the \texttt{ordinary} folder, preserving the original train/test split.

For \CEBAB, we used \texttt{train\_observational} as the training set, \texttt{dev} as the validation set, and \texttt{test} as the test set. For \CIFAR, we retained the original train/test split.

\subsection{Metrics details}
\label{sec:metrics-details}

For all datasets, we evaluate predictions at both the label and concept levels using three metrics, \emph{coverage}, \emph{size}, and \emph{consistency}, defined below.

\textbf{Coverage}. Coverage is defined as:
\[
    \textsc{Coverage}
    =
    \frac{1}{|\calD_{\textrm{test}}|}
    \sum_{\vx \in \calD_{\textrm{test}}}
    \Ind{\vc^* \in \Cset},
\]
where $\calD_{\textrm{test}}$ denotes the test set and $|\calD_{\textrm{test}}|$ its cardinality. This metric quantifies the proportion of test samples for which the ground-truth concept $\vc^*$ is included in the predicted concept set $\Cset$. An analogous definition applies at the label level by replacing $\vc^*$ and $\Cset$ with $\vy^*$ and $\Yset$, respectively.

\textbf{Size}. We measure the average size of the prediction sets as
\[
    \textsc{Size}
    =
    \frac{1}{|\calD_{\textrm{test}}|}
    \sum_{\vx \in \calD_{\textrm{test}}}
    |\Cset|,
\]
which corresponds to the expected cardinality of the predicted concept sets. The same definition applies to label sets $\Yset$.

\textbf{Consistency}.
Concept-level consistency is defined as
\[
    \textsc{Consistency}
    =
    \frac{1}{|\calD_{\textrm{test}}|}
    \sum_{\vx \in \calD_{\textrm{test}}}
    \sum_{\vc \in \Cset}
    \frac{\Ind{g^{\dagger}(\vc) \cap \Yset \neq \emptyset}}{\max(1, |\Cset|)},
\]
this metric measures the extent to which predicted concepts agree with predicted labels. Specifically, it captures the average fraction of concepts in $\Cset$ whose induced label $g^{\dagger}(\vc)$ is contained in $\Yset$. A fully consistent model ensures that every predicted concept is supported by at least one compatible label. An analogous definition applies at the label level.

\subsection{E-values details}
\label{sec:evalues}

A powerful and recent tool in the conformal prediction literature is conformal e-prediction, which uses \emph{e-values}~\citep{shafer2019game, vovk2021values, grunwald2020safe, ramdas2025hypothesis} to build prediction sets~\citep{vovk2025conformal, gauthier2025values, gauthier2025Conformal,
gauthier2026Conformal}.
An \emph{e-variable} $E$ is a nonnegative random variable satisfying $\bbE[E] \leq 1$~\citep{ramdas2025hypothesis}.
Thresholding $E$ at level $1/\alpha$ yields a valid prediction set with marginal coverage at least $1-\alpha$: by Markov's inequality, $\Pr(E \geq 1/\alpha) \leq \alpha$, or equivalently, $\Pr(E < 1/\alpha) \geq 1-\alpha$.

A standard e-value for conformal prediction is the \emph{soft-rank e-variable}~\citep{wang2022false, koning2023post,
balinsky2024enhancing}. 
Given a non-negative nonconformity score $s : \calX \times \calY \to \mathbb{R}_{\geq 0}$ and a calibration set $\calD_{\mathrm{cal}} = \{(\vx_i, \vy_i^*)\}_{i=1}^n$, it is defined as:
\begin{equation}
    \label{eq:soft-rank-evariable}
    E(\vx, \vy)
    =
    \frac{(n+1)\, s(\vx, \vy)}
    {\displaystyle\sum_{i=1}^{n} s(\vx_i, \vy_i^*) + s(\vx, \vy)}.
\end{equation}
This is a valid e-variable whenever the calibration scores $\{s(\vx_i,\vy_i^*)\}_{i=1}^n$ and the test score $s(\vx,\vy)$ are exchangeable~\citep{vovk2025conformal}.

\begin{proposition}[\citet{gauthier2025values}]
    \label{prop:e-value-guarantee}
    Consider a calibration set $\calD_{\mathrm{cal}} = \{(\vx_i, \vy_i^*)\}_{i=1}^n$ and a test point $(\vx, \vy^*)$ such that all $n+1$ points are exchangeable.
    Let $\tilde{\alpha}$ be any (possibly data-dependent) miscoverage level. Then:
    \begin{equation}
        \label{eq:e-guarantee}
        \bbE \left[
            \frac{\Pr(\vy^* \notin \Upsilon_{\tilde\alpha}(\vx) \mid \tilde{\alpha})}
                 {\tilde{\alpha}}
        \right] \leq 1,
    \end{equation}
    where
    \[
        \Upsilon_{\tilde\alpha}(\vx)
        =
        \left\{
            y \in \calY :
            \frac{(n+1)\,s(\vx, y)}
                 {\displaystyle\sum_{i=1}^n s(\vx_i, \vy_i^*) + s(\vx, y)}
            < \frac{1}{\tilde{\alpha}}
        \right\}.
        \label{eq:e-set}
    \]
\end{proposition}

\begin{remark}
    When $\tilde{\alpha}$ is a fixed constant independent of the data, \cref{eq:e-guarantee} reduces to the standard conformal guarantee
    $\Pr(\vy^* \in \Upsilon_\alpha(\vx)) \geq 1 - \alpha$~\citep{gauthier2025values}
\end{remark}

\textbf{E-value prediction sets}.
\EMethod (\cref{sec:method}) instantiates \cref{eq:soft-rank-evariable} with the negative log-probability nonconformity score $s(\vx, y) = -\log p_\theta(y \mid \vx)$%
, which is non-negative by design.

\textbf{Label sets}.  At calibration time, \EMethod computes the label score sum $S_Y = \sum_{i=1}^n s(\vx_i, \vy_i^*)$ and stores the calibration count $n$. At test time, given a miscoverage level $\tilde{\alpha}$, the label prediction set, following~\cref{eq:e-set}, is
\[
    \Yset
    =
    \left\{
        y \in \calY :
        \frac{(n+1)\,s(\vx, y)}{S_Y + s(\vx, y)} < \frac{1}{\tilde{\alpha}}
    \right\}.
\]

\textbf{Concept sets}. \EMethod aggregates per-concept e-values via averaging:
\[
    E(\vx, \vc) = \frac{1}{K}\sum_{j=1}^{K} E_j(\vx, v_j),
\]
which is a valid e-variable by linearity of expectation~(\cref{prop:e-world-coverage}). The concept prediction set
\[
    \Cset = \left\{ \vc \in \calC : E(\vx, \vc) < \tfrac{1}{\tilde{\beta}} \right\}
\]
satisfies $\Pr(\vc^* \in \Cset) \geq 1 - \bbE[\tilde{\beta}]$ (\cref{prop:e-world-coverage}), \emph{without} Bonferroni correction (\cref{cor:bonferroni-coverage}) and \emph{without} degradation in $K$. The product alternative $\prod_j E_j$ is also valid but requires conditional independence of the $E_j$ given $(\vx, \mathcal{D}_{\mathrm{cal}})$, which does not hold in practice (see \cref{rem:mean-vs-product}). Nevertheless, \EMethod remains robust to this violation (see \cref{sec:product-vs-average}).

\textbf{Size-budget calibration: computing $\tilde{\alpha}$ and $\tilde{\beta}$}.

Rather than fixing $\alpha$ and $\beta$ a priori, \EMethod selects the \emph{smallest miscoverage} levels that keep the average prediction-set sizes within user-specified budgets $C_Y$ (for labels) and $C_C$ (for concepts).

Formally, following~\citet{gauthier2025Conformal}, given a test set $\calD$ and calibration statistics $(S_Y, \{S_j\}_{j=1}^K)$:
\begin{align}
    \label{eq:alpha-tilde}
    \tilde{\alpha}
    &=
    \inf\left\{
        \alpha \in (0,1) :
        \frac{1}{|\calD|}
        \sum_{\vx \in \calD}
        |\Upsilon_\alpha(\vx)|
        \leq C_Y
    \right\},
    \\[4pt]
    \label{eq:beta-tilde}
    \tilde{\beta}
    &=
    \inf\left\{
        \beta \in (0,1) :
        \frac{1}{|\calD|}
        \sum_{\vx \in \calD}
        |\Gamma_\beta(\vx)|
        \leq C_C
    \right\}.
\end{align}
Both $\tilde{\alpha}$ and $\tilde{\beta}$ are data-dependent: they depend on the calibration score sums $S_Y$ and $\{S_j\}$. According to~\cref{prop:evalues-joint-coverage}, under deductive and abductive soundness (\cref{def:soundness}), the jointly revised sets $\Crev$ and $\Yrev$ produced by \EMethod satisfy:
\[
    \Pr(\vc^* \in \Crev)
    \ge
    1 - \bbE[\tilde{\alpha}] - \bbE[\tilde{\beta}]
    + \Pr(\vc^* \notin \Cset \land \vc^* \notin \Cab).
\]

\textbf{Bootstrap estimation of $\bbE[\tilde{\alpha}]$ and $\bbE[\tilde{\beta}]$}.
We estimate these quantities over $T = 100$ %
bootstrap iterations. In each iteration $t$, a calibration sample $\calD_t$ of size $|\calD_{\mathrm{cal}}|$ is drawn \emph{with replacement} from $\calD_{\mathrm{cal}}$, score sums are recomputed on $\calD_t$, and prediction sets are built on the fixed test set for all $(\alpha,\beta)$ in $(\{0.10, 0.15, 0.20, 0.25, 0.30\}, \{0.10, 0.35, 0.40, 0.45, 0.50, 0.55, 0.60\})$
The minimal miscoverage levels $\tilde{\alpha}_t$ and $\tilde{\beta}_t$ are then selected as the smallest joint values achieving average set sizes at most $C_Y = 2$
and $C_C = 5$, respectively (\cref{eq:beta-tilde}). The empirical coverage lower bounds are reported as
\[
    1 - \frac{1}{T}\sum_{t=1}^T \tilde{\alpha}_t
    \qquad\text{and}\qquad
    1 - \frac{1}{T}\sum_{t=1}^T \tilde{\beta}_t.
\]

\subsection{Neuro-Symbolic Models}
\label{sec:nesy-details}

In this subsection, we report relevant details about the implementation and use of the neuro-symbolic models used in the experiments, \ie \DPL and \LTN.

\subsubsection{Inference}

We begin by describing how inference is carried out at test time in \DPL and \LTN, that is, how the concept distribution $p_\theta(\vC \mid \vx)$ produced by the neural backbone is combined with prior knowledge $\BK$ to yield a label prediction $\hat{\vy}$.

\textbf{DeepProbLog}. We use the \DPL implementation of \citet{bortolotti2024benchmark}. At inference time, given an input $\vx$, the model computes the joint distribution over concepts and labels under prior knowledge $\BK$
\[
    p_\theta(\vY, \vC \mid \vx, \BK) = p_\theta(\vC \mid \vX) \Ind{(\vC, \vY) \models \BK}.
\]
Hard predictions are obtained by solving the MAP problem, which is NP-hard~\citep{shimony1994finding}, that is:
\[
    (\hat{\vc}, \hat{\vy}) = \argmax_{\vc \in \calC, \vy \in \calY} p_{\theta}(\vC \mid \vx) \Ind{(\vC, \vY) \models \BK}
\]
where
\[
    p_\theta(\vC \mid \vx) = \prod_j^K p_\theta(\vC_j \mid \vx)
\]
Namely, $p_\theta(\vC \mid \vx)$ factorizes over all the $K$ concepts due to conditional independence~\citep{manhaeve2018deepproblog}.

\textbf{Logic Tensor Networks}. To implement \LTN, we used the \texttt{ltn-torch} library~\citep{carraro2024ltntorch}, following the approach of~\citet{bortolotti2024benchmark}. Given a formula $\BK$ and concept distribution $p_\theta(\vC \mid \vx)$, each candidate label $\vy$ is scored according to the degree to which the pair $(\vc, \vy)$ satisfies $\BK$ under the chosen fuzzy relaxation.

We experiment with three t-(co)norms: 
\begin{itemize}
    \item G\"odel semantics
    \[
        \calT_G(a,b) = \min(a, b), \quad
        \calS_G(a,b) = \max(a, b), \quad
        \calN_G(a) = \Ind{a = 0}
    \]
    \item Product semantics
    \[
        \calT_P(a,b) = a \cdot b, \quad
        \calS_P(a,b) = a + b - a \cdot b, \quad
        \calN_P(a) = 1 - a
    \]
    \item \L{}ukasiewicz semantics
    \[
        \calT_L(a,b) = \max(a + b - 1, 0), \quad
        \calS_L(a,b) = \min(a + b, 1), \quad
        \calN_L(a) = 1 - a
    \]
\end{itemize}

Where $\calT$ denotes a \textit{t-norm}, $\calS$ denotes a \textit{t-conorm}, $\calN$ denotes the \textit{negation operation}, and $a$ and $b$ are fuzzy truth values in $[0, 1]$.

The inference layer selects the label:
\[
    \hat{\vy} =
    \argmax_{\vy \in \calY}
    \calT_{\BK}\left(p_\theta(\vC \mid \vx), \Ind{\vY=\vy}\right)
\]
where $\mathcal{T}_{\mathcal{K}} \in [0,1]$ denotes the fuzzy satisfaction degree. 

In practice, inference is performed by first computing:
\[
    \hat{\vc} = \argmax_{\vc \in \calC}  p_\theta(\vC \mid \vx)
\]
and then predicting the final label leveraging prior knowledge $\BK$:
\[
    \hat{\vy} = \BK(\hat{\vc})
\]

\subsubsection{Grounding}
\label{sec:grounding}

For each dataset, we detail how the prior knowledge $\BK$ is grounded in \DPL and \LTN. In \DPL, $\BK$ is compiled into a \emph{concept-to-label} matrix $W \in [0,1]^{|\calC| \times |\calY|}$, where each entry $W_{\vc, \vy}$ represents the probability of $(\vc, \vy) \models \BK$. The label distribution is then computed by $p_\theta(\vY \mid \vx) = W^\top p_\theta(\vC \mid \vx)$. In \LTN, $\BK$ is expressed as a first-order logic formula whose predicates are grounded on the neural outputs $p_\theta(\vc \mid \vx) \in [0,1]$, while logical connectives are interpreted using a chosen t-norm (and corresponding co-tnorm), as detailed in \cref{sec:architecture-details}.

\textbf{\MNISTAdd (and \MNISTEO, \MNISTHALF)}.
Each input consists of two images $(\vx_1, \vx_2)$ with associated digit variables $\vc_1, \vc_2 \in \{0,\dots,9\}$ and output $\vy \in \{0,\dots,18\}$.
\underline{\DPL}: The circuit enumerates all $10^2$ digit combinations and sets
$W_{(\vc_1, \vc_2),\, \vy} = \Ind{\vc_1 + \vc_2 = \vy}$, so that each concept vector $\vc$ maps deterministically to exactly one sum.
\underline{\LTN}: The knowledge base is
\[
    \BK_{\MNISTAdd} :=
    \forall \vx_1, \vx_2, \vy.
    \exists d_1, d_2.
    \bigl( \mathrm{digit}(\vx_1, d_1) \land \mathrm{digit}(\vx_2, d_2) \bigr)
    \text{ s.t. } d_1 + d_2 = \vy,
\]
where $\mathrm{digit}(\vx_i, d_i) := p_\theta(\vc_i = d \mid \vx_i)$ is the digit probability.

\textbf{\MNISTSumXor}.
The setup is as in \MNISTAdd, but with $\vY \in \{0,1\}$ encoding the parity of the sum.
\underline{\DPL}: $W_{(\vc_1,\vc_2), \vy} = \Ind{(\vc_1 + \vc_2) \bmod 2 = \vy}$.
\underline{\LTN}: Same FOL template as \MNISTAdd, with arithmetic condition $d_1 + d_2 = \vy \pmod 2$.

\textbf{\CHX}.
The concepts are the four binary symptoms $\vc = (\vc_1, \vc_2, \vc_3, \vc_4) \in {0,1}^4$ (\texttt{Fracture}, \texttt{Pneumothorax}, \texttt{Airspace opacity}, \texttt{Nodule/mass}), and the label $\vy$ denotes the severity level.
\underline{\DPL}: The model enumerates all $2^4$ possible concept combinations and defines
$W_{\vc,\vy} = \Ind{\sum_{i=1}^{4} \vc_i = \vy}$, corresponding to the five severity levels ${\texttt{healthy}, \texttt{green}, \texttt{yellow}, \texttt{orange}, \texttt{red}}$.
\underline{\LTN}: We instead formulate the knowledge base in first-order logic. Let $\mathrm{present}(\vx, i) := p_\theta(\vc_i = 1 \mid \vx)$. The knowledge base is then defined as follows:
\[
    \BK^{\mathrm{mc}}_{\CHX} :=
    \forall \vx, \vy.\ \exists s_1, s_2, s_3, s_4 \in \{0,1\}.
    \bigwedge_{i=1}^{4} \bigl(s_i \iff \mathrm{present}(\vx, i)\bigr)
    \text{ s.t. } \textstyle\sum_{i=1}^{4} s_i = \vy.
\]

\textbf{\CEBAB}. Recall that $\vc \in {\texttt{Positive}, \texttt{Negative}, \texttt{Neutral}, \texttt{unknown}}^5$ and the label $\vy \in {\texttt{Positive}, \texttt{Negative}, \texttt{Neutral}, \texttt{unknown}, \texttt{conflict}}$ is defined as the majority vote with the tie-breaking rule described in \cref{sec:dataset-details}.
\underline{\DPL}: The circuit enumerates all $4^5$ possible concept combinations. For each combination, let $N_v := \sum_{i=1}^{5} \Ind{\vc_i = v}$ denote the count of each value $v \in {\texttt{Positive}, \texttt{Negative}, \texttt{Neutral}, \texttt{unknown}}$. Then
\[
    W_{\vc,\, y} =
    \begin{cases}
        \Ind{N_{\texttt{Pos}} = N_{\texttt{Neg}} \,\land\, N_{\texttt{Pos}} \ge \max(N_{\texttt{Neu}}, N_{\texttt{Unk}})} & \vy = \texttt{Conflict}, \\[2pt]
        \Ind{N_v \ge \max_{v' \ne v} N_{v'}} \cdot \big(1 - W_{\vc, \texttt{Conflict}}\big) & \vy \in \{\texttt{Pos}, \texttt{Neg}, \texttt{Neu}, \texttt{Unk}\},
    \end{cases}
\]
where the tie-breaking priority $\texttt{Pos} \succ \texttt{Neg} \succ \texttt{Neu} \succ \texttt{Unk}$ is enforced by evaluating the cases in order and assigning exactly one label per concept combination.
\underline{\LTN}: The discrete $\argmax$ rule is not directly differentiable. We therefore adopt a relaxed formulation based on \emph{soft counts} $\tilde{N}v(\vx) := \sum{i=1}^{5} p_\theta(\vC_i = v \mid \vx)$, together with smooth comparison operators:
\[
    \mathrm{ge}(a, b) := \sigma(a - b), \qquad
    \mathrm{eq}(a, b) := \exp\bigl(-|a - b|\bigr),
\]
where $\sigma$ denotes the sigmoid function, i.e., $\sigma(x) = \frac{1}{1 + \exp(-x)}$. Strict inequality is expressed compositionally as $\mathrm{gt}(a, b) := \mathrm{ge}(a, b) \land \neg, \mathrm{eq}(a, b)$. The knowledge base is defined as the conjunction, over all five labels, of biconditionals between each label indicator and the corresponding count-based predicate:
\begin{align*}
    \BK_{\CEBAB} := \forall \vx, \vy.
    &\bigl(\Ind{\vy = \texttt{Pos}} \iff \mathrm{gt}(\tilde{N}_{\texttt{Pos}}, \tilde{N}_{\texttt{Neg}}) \land \mathrm{ge}(\tilde{N}_{\texttt{Pos}}, \tilde{N}_{\texttt{Neu}}) \land \mathrm{ge}(\tilde{N}_{\texttt{Pos}}, \tilde{N}_{\texttt{Unk}})\bigr) \\
    \land &\bigl(\Ind{\vy = \texttt{Neg}} \iff \mathrm{gt}(\tilde{N}_{\texttt{Neg}}, \tilde{N}_{\texttt{Pos}}) \land \mathrm{ge}(\tilde{N}_{\texttt{Neg}}, \tilde{N}_{\texttt{Neu}}) \land \mathrm{ge}(\tilde{N}_{\texttt{Neg}}, \tilde{N}_{\texttt{Unk}})\bigr) \\
    \land &\bigl(\Ind{\vy = \texttt{Neu}} \iff \mathrm{gt}(\tilde{N}_{\texttt{Neu}}, \tilde{N}_{\texttt{Pos}}) \land \mathrm{gt}(\tilde{N}_{\texttt{Neu}}, \tilde{N}_{\texttt{Neg}}) \land \mathrm{ge}(\tilde{N}_{\texttt{Neu}}, \tilde{N}_{\texttt{Unk}})\bigr) \\
    \land &\bigl(\Ind{\vy = \texttt{Unk}} \iff \mathrm{gt}(\tilde{N}_{\texttt{Unk}}, \tilde{N}_{\texttt{Pos}}) \land \mathrm{gt}(\tilde{N}_{\texttt{Unk}}, \tilde{N}_{\texttt{Neg}}) \land \mathrm{gt}(\tilde{N}_{\texttt{Unk}}, \tilde{N}_{\texttt{Neu}})\bigr) \\
    \land &\bigl(\Ind{\vy = \texttt{Conflict}} \iff \mathrm{eq}(\tilde{N}_{\texttt{Pos}}, \tilde{N}_{\texttt{Neg}}) \land \mathrm{ge}(\tilde{N}_{\texttt{Pos}}, \tilde{N}_{\texttt{Neu}}) \land \mathrm{ge}(\tilde{N}_{\texttt{Pos}}, \tilde{N}_{\texttt{Unk}})\bigr).
\end{align*}
The tie-breaking priority $\texttt{Pos} \succ \texttt{Neg} \succ \texttt{Neu} \succ \texttt{Unk}$ is encoded through the asymmetric use of $\mathrm{ge}$ and $\mathrm{gt}$: higher-priority categories are allowed to win ties via $\ge$, whereas lower-priority categories must satisfy the strict inequality $>$.

\textbf{\CIFAR and \RIVAL}.
The concepts are the $7$ binary visual attributes $\vc = (\vc_{\texttt{whl}}, \vc_{\texttt{met}}, \vc_{\texttt{wng}}, \vc_{\texttt{ani}}, \vc_{\texttt{hai}}, \vc_{\texttt{hrn}}, \vc_{\texttt{snt}}) \in {0,1}^7$, and the labels are the $10$ \CIFAR classes.
\underline{\DPL}: The circuit enumerates all $2^7$ concept combinations. For each combination, the rule base in \cref{sec:dataset-details} is evaluated; when the active concepts uniquely entail a single class, we set $W_{\vc, \vy} = 1$ for that class. When two classes share the same concepts (\ie \texttt{automobile}/\texttt{truck} and \texttt{dog}/\texttt{horse}), the mass is split equally, yielding $W_{\vc, \vy} = 1/2$ for each. Concept combinations that do not uniquely identify any class are handled by distributing probability uniformly over the consistent classes (\eg $W_{\vc,\texttt{plane}} = W_{\vc,\texttt{car}} = W_{\vc,\texttt{truck}} = W_{\vc,\texttt{ship}} = 1/4$ when only \texttt{met} is active). This explicitly encodes the \emph{incompleteness} of $\BK$ in the truth table.
\underline{\LTN}: For \LTN, let $\mathrm{present}(\vx, a) = p_\theta(\vc_a = 1 \mid \vx)$ for $a \in {\texttt{whl}, \texttt{met}, \texttt{wng}, \texttt{ani}, \texttt{hai}, \texttt{hrn}, \texttt{snt}}$. The knowledge base is expressed as a conjunction of biconditionals, one per class:
\begin{align*}
    \BK_{\CIFAR} := \forall \vx, \vy.
    &\bigl(\Ind{\vy = \texttt{plane}} \iff \mathrm{met} \land \neg\,\mathrm{ani} \land \mathrm{wng}\bigr) \\
    \land &\bigl(\Ind{\vy = \texttt{car}} \vee \Ind{\vy = \texttt{truck}} \iff \mathrm{whl} \land \neg\,\mathrm{wng} \land \mathrm{met} \land \neg\,\mathrm{ani}\bigr) \\
    \land &\bigl(\Ind{\vy = \texttt{bird}} \iff \mathrm{ani} \land \neg\,\mathrm{met} \land \mathrm{wng}\bigr) \\
    \land &\bigl(\Ind{\vy = \texttt{frog}} \iff \mathrm{ani} \land \neg\,\mathrm{met} \land \neg\,\mathrm{hai}\bigr) \\
    \land &\bigl(\Ind{\vy = \texttt{deer}} \iff \mathrm{ani} \land \neg\,\mathrm{met} \land \neg\,\mathrm{wng} \land \mathrm{hai} \land \mathrm{hrn}\bigr) \\
    \land &\bigl(\Ind{\vy = \texttt{cat}} \iff \mathrm{ani} \land \neg\,\mathrm{met} \land \neg\,\mathrm{wng} \land \mathrm{hai} \land \neg\,\mathrm{hrn} \land \neg\,\mathrm{snt} \bigr) \\
    \land &\bigl(\Ind{\vy = \texttt{dog}} \vee \Ind{\vy = \texttt{horse}} \iff \mathrm{ani} \land \neg\,\mathrm{met} \land \neg\,\mathrm{wng} \land \mathrm{hai} \land \neg\,\mathrm{hrn} \land \mathrm{snt}\bigr) \\
    \land &\bigl(\Ind{\vy = \texttt{ship}} \iff \mathrm{met} \land \neg\,\mathrm{ani} \land \neg\,\mathrm{whl} \land \neg\,\mathrm{wng}\bigr),
\end{align*}
\RIVAL uses the same grounding since the two datasets share the same labels and concepts.

\subsection{Dataset details}
\label{sec:dataset-details}

Here we report additional information about the evaluated datasets.

\subsubsection{\MNISTAdd}

\MNISTAdd, originally introduced in~\citet{manhaeve2018deepproblog}, is the canonical benchmark in Neuro-Symbolic AI. Each input consists of two \MNIST~\citep{lecun1998mnist} digit images, and the target label is their arithmetic sum. Formally, each sample is a pair of images $(\vx_1, \vx_2)$ representing \MNIST digits $(\vc_1, \vc_2)$, each taking values in $[0, 9]$, and the task is to predict:
\[
    \vy = \vc_1 + \vc_2,
\]
In contrast to \MNISTEO and \MNISTHALF, \MNISTAdd contains all possible digit combinations ($54,000$ train examples, $6,000$ validation examples and $10,000$ test examples) and does not induce structural ambiguity in the digit-to-value mapping. The underlying combinatorial structure uniquely identifies each digit, and no deterministic reasoning shortcuts arise. Consequently, \MNISTAdd serves as a shortcut-free reference benchmark for reasoning~\citep{bortolotti2025shortcuts}.

\textbf{Reasoning Shortcuts}. \MNISTAdd is provably reasoning shortcuts free. For further information refer to~\citet{bortolotti2024benchmark}.

\subsubsection{\CHX}

\CHX~\citep{cohen2022torch} is medical dataset characterized by four binary symptoms:
\texttt{Fracture}, \texttt{Pneumothorax}, \texttt{Airspace opacity}, and \texttt{Nodule/mass}. Each symptom is represented as a binary variable indicating its presence or absence. For preprocessing, we follow~\citet{pugnana2025ask}. However, in its original form, \CHX is not directly suitable for learning and reasoning tasks. For this reason, we construct our setting as follows: let the concept vector be defined as:
\[
    \mathbf{c} = (\vc_1, \vc_2, \vc_3, \vc_4) \in \{0,1\}^4,
\]
where:
\[
\vc =
\begin{pmatrix}
    \texttt{Fracture} \\
    \texttt{Pneumothorax} \\
    \texttt{Airspace opacity} \\
    \texttt{Nodule/mass}
\end{pmatrix}
\]
We designed our NeSy task, such that the final label is determined by the total number of present symptoms. Formally, let
\[
    \vs = \sum_{i=1}^{4} \vc_i.
\]
The diagnostic label $\vy$ is defined as:
\[
    \vy =
    \begin{cases}
        \texttt{healthy} & \text{if } \vs = 0, \\
        \texttt{green\_code}   & \text{if } \vs = 1, \\
        \texttt{yellow\_code}  & \text{if } \vs = 2, \\
        \texttt{orange\_code}     & \text{if } \vs = 3, \\
        \texttt{red\_code}     & \text{if } \vs = 4.
    \end{cases}
\]
Thus, the task consists of predicting a severity level derived deterministically from the number of active symptoms. \CHX consists of $3,150$ train samples, $350$ validation samples, and $876$ test samples.

\textbf{Reasoning Shortcuts}. \CHX contains reasoning shortcuts. For example, \texttt{yellow\_code} ($\vy = 2$) can be obtained not only by predicting $\vc = (1, 0, 1, 0)$, but also via other permutations such as $\vc' = (0, 1, 1, 0)$. For the exact (or approximate) number of reasoning shortcuts, we refer the reader to \texttt{rs-count}~\citep{bortolotti2024benchmark} specifying such settings.

\subsubsection{\CEBAB}

\CEBAB~\citep{abraham2022cebab} is a dataset designed to evaluate concept-based and causal reasoning in natural language models. It consists of restaurant reviews with structured, aspect-level annotations and counterfactual edits.
Each instance consists of:
\begin{itemize}
    \item A review text called \texttt{description}, \ie the input,
    \item A set of high-level semantic concepts $\vc \in [0, 4]$,
    \item An overall rating \texttt{review\_majority} $\in [0, 5]$ .
\end{itemize}
The core concepts correspond to distinct aspects of the restaurant experience:
\[
\vc =
\begin{pmatrix}
    \texttt{food\_aspect\_majority} \\
    \texttt{service\_aspect\_majority} \\
    \texttt{noise\_aspect\_majority} \\
    \texttt{ambiance\_aspect\_majority}
\end{pmatrix}
\]
Each aspect captures the sentiment expressed toward that dimension (\ie \texttt{Positive}, \texttt{Negative}, \texttt{Neutral}, \texttt{unknown}). From this concepts, the goal is to infer the causally influenced label that is \texttt{review\_majority}.
The dataset is distributed in JSON format. Below is an example instance:
\begin{verbatim}
{
  "id": "000000_000000",
  "original_id": "000000",
  "edit_id": "000000",
  "is_original": true,
  "edit_goal": "Negative",
  "edit_type": "noise",
  "edit_worker": "w82",
  "description": "Overbooked and did not honor reservation time, 
  put on wait list with walk-ins",
  "review_majority": "1",
  "review_label_distribution": {
    "1": 4,
    "2": 1
  },
  "review_workers": {
    "w244": "1",
    "w120": "2",
    "w197": "1",
    "w7": "1",
    "w132": "1"
  },
  "food_aspect_majority": "",
  "ambiance_aspect_majority": "",
  "service_aspect_majority": "Negative",
  "noise_aspect_majority": "unknown",
  "service_aspect_label_distribution": {
    "Negative": 5
  },
  "noise_aspect_label_distribution": {
    "unknown": 4,
    "Negative": 1
  },
  "opentable_metadata": {
    "restaurant_id": 6513,
    "restaurant_name": "Molino's Ristorante",
    "cuisine": "italian",
    "price_tier": "low",
    "dining_style": "Casual Elegant",
    "dress_code": "Smart Casual",
    "parking": "Private Lot",
    "region": "south",
    "rating_ambiance": 1,
    "rating_food": 3,
    "rating_noise": 2,
    "rating_service": 2,
    "rating_overall": 2
  }
}
\end{verbatim}

\textbf{Preprocessing}. As originally defined, \CEBAB does not provide a deterministic rule mapping aspect-level concepts to the overall review label, making it unsuitable for a NeSy setting in which symbolic inference must be explicit.
To address this, we introduce \texttt{review\_majority} as an additional concept, so that the concept vector $\vc$ becomes
\[ 
    \vc = 
        \begin{pmatrix} \texttt{food\_aspect\_majority} \\ \texttt{service\_aspect\_majority} \\ \texttt{noise\_aspect\_majority} \\ \texttt{ambiance\_aspect\_majority} \\ \texttt{review\_majority}
    \end{pmatrix} \in \{0, 1, 2, 3\}^5
\]
with each component $\vc_i \in$ \{\texttt{Positive}, \texttt{Neutral}, \texttt{Negative}, \texttt{unknown}\}.

To make it compliant with the other concepts, the original \texttt{review\_majority} field is quantized as:
\[
\texttt{review\_majority} =
\begin{cases}
    \texttt{unknown} & \text{if } \in \{\texttt{no majority}, \texttt{unknown}\}, \\
    \texttt{Negative} & \text{if } \le 2, \\
    \texttt{Neutral} & \text{if } = 3, \\
    \texttt{Positive} & \text{if } \ge 4,
\end{cases}
\]
We then construct our reasoning task as selecting the category with the maximum count, representing an aggregate of the restaurant's quality:
\[
    \vy = \argmax_{\vc_i}{\sum_{\vc_j \in C} \Ind{\vc_i = \vc_j}}, \quad \vy \in \{\texttt{Positive}, \texttt{Neutral}, \texttt{Negative}, \texttt{unknown}, \texttt{conflict}\},
\]
with the following tie-breaking rule: if both \texttt{Positive} and \texttt{Negative} attain the maximum count, then $y = \texttt{conflict}$; otherwise, ties follow the priority order \texttt{Positive} $\succ$ \texttt{Negative} $\succ$ \texttt{Neutral} $\succ$ \texttt{unknown}.
The dataset provides three different training sets: \texttt{observational}, \texttt{inclusive}, and \texttt{exclusive}. Since we do not analyze counterfactuals or interpretability, we use the \texttt{observational} training split. The dataset contains $1,755$ training samples, $1,673$ validation samples, and $1,689$ test samples. 

\textbf{Reasoning Shortcuts}. \CEBAB, as designed, is not free of reasoning shortcuts. For instance, for a given label such as \texttt{Negative}, multiple concept combinations are possible, such as $\vc = (\texttt{Negative}, \texttt{Negative}, \texttt{Negative}, \texttt{Negative})$ or $ (\texttt{Positive}, \texttt{Negative}, \texttt{Neutral}, \texttt{Negative})$. 
If interested, exact (or approximate) number of reasoning shortcuts can be found by using \texttt{rs-count}~\citep{bortolotti2024benchmark} instatiating the formalized problem.

\subsubsection{\CIFAR}

The \CIFAR dataset \citep{krizhevsky2009cifar} is a standard benchmark for image classification. It contains $60,000$ color images of size $32 \times 32$, divided into $10$ classes, with $50,000$ training images (which we further split into $35,000$ for training and $15,000$ for validation) and $10,000$ test images. The original class labels are:
\[
\texttt{\{airplane, automobile, bird, cat, deer, dog, frog, horse, ship, truck\}}
\]
\textbf{Preprocessing}. As with \CEBAB, the raw \CIFAR dataset is not directly suitable for learning and reasoning tasks. For this reason, we introduce a set of intermediate concepts that are \textit{intended} to be present in all images and useful for distinguishing the $10$ classes.
However, to avoid trivial label leakage and to ensure that some concepts remain insufficient, we \textit{deliberately} design the concept-to-label mapping to be incomplete. In particular, we construct the labeling such that, given only the provided concepts, it is not possible to distinguish between \texttt{automobile} and \texttt{truck}, nor between \texttt{dog} and \texttt{horse}.
In practice, we build a binary concept vector $\mathbf{c} \in \{0,1\}^{7}$ for each \CIFAR image, encoding high-level semantic attributes such as \texttt{ani} (animal), \texttt{wng} (wing), \texttt{met} (metallic), \texttt{whl} (wheels), \texttt{hai} (hairy), \texttt{hrn} (horn) and \texttt{snt} (snout).
These concepts are deterministically derived from the original class label. 
Then, given the concepts, we build the following rules for the \CIFAR labels:
\[
\begin{aligned}
    \texttt{airplane} &\iff \texttt{met} \land \neg \texttt{ani} \land \texttt{wng}, \\
    \texttt{automobile, truck} &\iff \texttt{whl} \land \texttt{met} \land \neg \texttt{wng} \land \neg \texttt{ani}, \\
    \texttt{bird} &\iff \texttt{ani} \land \neg \texttt{met} \land \texttt{wng}, \\
    \texttt{frog} &\iff \texttt{ani} \land \neg \texttt{met} \land \neg \texttt{hai}, \\
    \texttt{deer} &\iff \texttt{ani} \land \neg \texttt{met} \land \neg \texttt{wng} \land \texttt{hai} \land \texttt{hrn}, \\
    \texttt{dog, horse} &\iff \texttt{ani} \land \neg \texttt{met} \land \neg \texttt{wng} \land \texttt{hai} \land \neg \texttt{hrn} \land \texttt{snt}, \\
    \texttt{ship} &\iff \texttt{met} \land \neg \texttt{ani} \land \neg \texttt{whl} \land \neg \texttt{wng}.
\end{aligned}
\]

\textbf{Reasoning Shortcuts}. \CIFAR, as constructed, is not free of reasoning shortcuts. 
For instance, \texttt{airplane} is defined as $\texttt{met} \land \neg \texttt{ani} \land \texttt{wng}$ but nothing is said about other concepts. For example, concept assignments such as $\texttt{met} = \top, \texttt{ani} = \bot, \texttt{wng} = \top, \texttt{hai} = \bot$ and $\texttt{met} = \top, \texttt{ani} = \bot, \texttt{wng} = \top, \texttt{hai} = \top$ both entail the \texttt{airplane} class. For the exact (or approximate) number of reasoning shortcuts, you can instantiate this problem in \texttt{rs-count}~\citep{bortolotti2024benchmark}.

\subsubsection{\RIVAL}

\RIVAL~\citep{moayeri2022rival} is a real-world image dataset built from \texttt{ImageNet}~\citep{deng2009imagenet} that contains natural images spanning the same $10$ semantic categories as \CIFAR: \texttt{airplane}, \texttt{automobile}, \texttt{bird}, \texttt{cat}, \texttt{deer}, \texttt{dog}, \texttt{frog}, \texttt{horse}, \texttt{ship}, and \texttt{truck}. We kept the original train/test split as in~\citep{moayeri2022rival}, but creating a validation set from the train data counting $14,768$ train samples, $6,330$ validation samples and $5,286$ test samples.

\textbf{Preprocessing}. Since \RIVAL shares the same labels as \CIFAR, we adopt the exact same preprocessing pipeline.

\textbf{Reasoning Shortcuts}. Since \RIVAL shares the same concepts, knowledge and concept combinations of \CIFAR (despite having different distributions). Refer to \CIFAR for the reasoning shortcuts discussion.

\subsubsection{\MNISTEO}

Following \citep{bortolotti2024benchmark}, we consider the \MNISTEO dataset, originally introduced in \citep{marconato2023neuro}. This dataset is a structured variant of \MNISTAdd in which only a restricted subset of digit combinations is observed during training. Specifically, the dataset contains exclusively pairs composed either of even digits or of odd digits, defined by the following constraints:

\[
    \begin{array}{ccc}
        \begin{cases}
            \MZero + \MSix &= 6\\ 
            \MTwo + \MEight &= 10\\ 
            \MFour + \MSix &= 10\\ 
            \MFour + \MEight &= 12
        \end{cases}
        &
        \land
        &
        \begin{cases}
            \MOne + \MFive &= 6\\ 
            \MThree + \MSeven &= 10\\ 
            \MOne + \MNine &= 10\\ 
            \MThree + \MNine &= 12
        \end{cases}
    \end{array}
\]

The dataset contains $6,720$ fully annotated training samples, $1,920$ validation samples, and $960$ in-distribution test samples.

\textbf{Reasoning Shortcuts}. \MNISTEO is one of the first datasets introduced to study reasoning shortcuts; for a detailed analysis, see~\citet{marconato2023not}.

\subsubsection{\MNISTHALF}

\MNISTHALF, originally introduced in~\citep{marconato2024bears}, is a biased variant of \MNISTAdd in which only digits from $0$ to $4$ are considered. Moreover, only a subset of digit combinations is observed during training, given by:

\[ \label{eq:mnist-half-combs}
    \begin{cases}
        \MZero + \MZero &= 0\\ 
        \MZero + \MOne &= 1\\ 
        \MTwo + \MThree &= 5\\ 
        \MTwo + \MFour &= 6
    \end{cases}
\]

This restricted supervision induces ambiguity in the digit-to-value mapping and therefore introduces reasoning shortcuts. In contrast to \MNISTEO, digits $0$ and $1$ remain uniquely identifiable from the observed equations, while digits $2$, $3$, and $4$ admit multiple consistent assignments.

The dataset comprises $2,940$ fully annotated training samples, $840$ validation samples, and $420$ in-distribution test samples.

\textbf{Reasoning Shortcuts}. Similarly to \MNISTEO, \MNISTHALF, was introduced to study reasoning shortcuts; for a detailed analysis, see~\citet{marconato2024bears}.

\subsubsection{\MNISTSumXor}

\MNISTSumXor, originally introduced in \citep{bortolotti2025shortcuts}, is a variant of \MNISTAdd that has a different logical rule. As in \MNISTAdd, the input consists of two \MNIST digit images corresponding to digits $(\vx_1, \vx_2)$. However, instead of predicting their arithmetic sum, the target label is defined as the \emph{parity} of the sum. For this reason, it has the same amount of samples as \MNISTAdd per split.

Formally, the task is:
\[
    \vy = \left( \vx_1 + \vx_2 \right) \bmod 2,
\]
where $\vy \in \{0,1\}$, with $0$ denoting an even sum and $1$ an odd sum.
For example:
\begin{align}
    &(\MTwo + \MFour) \mod 2 = 0 \\
    &(\MOne + \MTwo) \mod 2 = 1
\end{align}

\textbf{Reasoning Shortcuts}. Similarly to \MNISTEO and \MNISTHALF, \MNISTSumXor was introduced to study shortcuts; for a detailed analysis, see~\citet{bortolotti2025shortcuts}.

\section{Architecture Details \& Model Selection}
\label{sec:model-selection-super}

\subsection{Model Selection}
\label{sec:model-selection}

All experiments use either the \texttt{Adam}~\citep{kingma2014adam} or \texttt{SGD}~\citep{robbins1951stochastic} optimizer, selected as part of the hyperparameter search. Hyperparameters were selected through Bayesian optimization via \texttt{Optuna}~\cite{akiba2019optuna}, using a \texttt{BoTorch} sampler~\cite{akiba2019optuna} with a \texttt{TPE} sampler~\cite{watanabe2023tree} for the initial startup trials, maximizing the macro F1 score of the final label on a validation set. Each study ran for \texttt{<M>} (at least $20$) trials, with $10$ warm-up trials before the Bayesian model is fitted.
All experiments were run for \texttt{<N>} epochs employing early stopping, saving the checkpoint which achieves the highest F1 score on the validation set.
The learning rate $\gamma$ was selected from $\{10^{-5}, 10^{-4}, 10^{-3}, 10^{-2}, 10^{-1}\}$, and the momentum $\iota$ from $\{10^{-5}, 10^{-4}, 10^{-3}, 10^{-2}, 0.1, 0.5, 0.9, 0.99\}$. The batch size $\nu$ was chosen from $\{32, 64, 128, 256\}$.

For \LTN, a consistent logic family was selected jointly for all operators, choosing among G{\"o}del, Product, and \L{}ukasiewicz. Under Product logic, the implication operator defaults to Goguen implication for consistency. The quantifier parameter $p$ was searched over integer values in the range $[1, 9]$. 
For \LTN, since the training objective can be hard to optimize since the $p$ parameter governs the sharpness of the existential and universal quantifiers. We therefore linearly anneal $p$ from $1$ to the target value $p^*$ across all training epochs:
\begin{equation}
  p_t = \mathrm{round}\!\left(1 + (p^* - 1)\,\frac{t}{T - 1}\right), \quad t = 0, \ldots, T-1,
\end{equation}
where $T$ is the total number of epochs. This keeps the logic soft early in training and progressively makes it crisper as concept representations stabilize.
Similarly, when concept supervision is available, we suppress concept supervision for the first 25\% of training and then linearly ramp the concept-loss weight from $0$ to $1$.

All models were trained for $200$ epochs with early stopping, except for \CEBAB, where models were trained for $500$ epochs. Early stopping selects the checkpoint achieving the highest F1 score on the validation set with respect to the final label. All experiments have been tested with the same $10$ seeds for consistency, specifically $\{1011, 1213, 1415, 1617, 1819, 2021, 2223, 2425, 2627, 2829\}$.

Below, you find all the hyperparameters which performed the best on our datasets.

\textbf{Hyperparameters for \MNISTAdd}.
\begin{itemize}
    \item \DPL, $\gamma = 10^{-1}$, $\nu = 32$, $\iota = 0.1$, optimizer = \texttt{SGD};
    \item \LTN, $\gamma = 10^{-1}$, $\nu = 64$, $\iota = 10^{-5}$, optimizer = \texttt{SGD}, $p = 8$, {\sc AND/OR/IMPLICATION} set to Product.
\end{itemize}

\textbf{Hyperparameters for \CHX}.
\begin{itemize}
    \item \DPL (with and without concept supervision), $\gamma = 10^{-4}$, $\nu = 128$, $\iota = 10^{-2}$, optimizer = \texttt{Adam};
    \item \LTN (with and without concept supervision), $\gamma = 10^{-4}$, $\nu = 256$, $\iota = 0.9$, optimizer = \texttt{Adam}, $p = 4$, {\sc AND/OR/IMPLICATION} set to Product.
\end{itemize}

\textbf{Hyperparameters for \RIVAL}:
\begin{itemize}
    \item \DPL (with concept supervision), $\gamma = 10^{-3}$, $\nu = 128$, $\iota = 0.99$, optimizer = \texttt{SGD};
    \item \DPL (without concept supervision), $\gamma = 10^{-4}$, $\nu = 128$, $\iota = 0.99$, optimizer = \texttt{SGD};
    \item \LTN (with and without concept supervision), $\gamma = 10^{-4}$, $\nu = 256$, $\iota = 0.9$, optimizer = \texttt{Adam}, $p = 4$, {\sc AND/OR/IMPLICATION} set to Product.
\end{itemize}

\textbf{Hyperparameters for \CIFAR}:
\begin{itemize}
    \item \DPL (with concept supervision), $\gamma = 10^{-3}$, $\nu = 128$, $\iota = 0.99$, optimizer = \texttt{SGD};
    \item \DPL (without concept supervision), $\gamma = 10^{-4}$, $\nu = 128$, $\iota = 10^{-2}$, optimizer = \texttt{Adam};
    \item \LTN (with and without concept supervision), $\gamma = 10^{-4}$, $\nu = 256$, $\iota = 0.9$, optimizer = \texttt{Adam}, $p = 4$, {\sc AND/OR/IMPLICATION} set to Product.
\end{itemize}

\textbf{Hyperparameters for \CEBAB}:
\begin{itemize}
    \item \DPL (with concept supervision), $\gamma = 10^{-2}$, $\nu = 64$, $\iota = 1e-05$, optimizer = \texttt{Adam};
    \item \DPL (without concept supervision), $\gamma = 10^{-4}$, $\nu = 128$, $\iota = 10^{-4}$, optimizer = \texttt{Adam};
    \item \LTN (with and without concept supervision), $\gamma = 10^{-4}$, $\nu = 256$, $\iota = 0.9$, optimizer = \texttt{Adam}, $p = 4$, {\sc AND/OR/IMPLICATION} set to Product.
\end{itemize}

\textbf{Hyperparameters for other \MNIST-based datasets}:
\begin{itemize}
    \item \DPL (\MNISTHALF, \MNISTSumXor, \MNISTEO), $\gamma = 10^{-1}$, $\nu = 32$, $\iota = 10^{-1}$, optimizer = \texttt{SGD};
    \item \LTN (\MNISTHALF, \MNISTSumXor), $\gamma = 10^{-4}$, $\nu = 256$, $\iota = 0.9$, optimizer = \texttt{Adam}, $p = 4$, {\sc AND/OR/IMPLICATION} set to Product.
    \item \LTN (\MNISTEO), $\gamma = 10^{-2}$, $\nu = 64$, $\iota = 0.99$, optimizer = \texttt{Adam}, $p = 3$, {\sc AND/OR/IMPLICATION} set to G\"odel.
\end{itemize}

\subsection{Model Architectures}
\label{sec:architecture-details}

We describe below the three neural architectures used as concept extractors $f$ across our experiments.

\textbf{LeNet5}.
LeNet5~\citep{lecun2002gradient} is used as backbone for \MNIST~\citep{lecun1998mnist} images (\ie \MNISTAdd and its variants).
The architecture takes as input a single-channel $28 \times 28$ image and produces a $\mathtt{num\_concepts}$-dimensional output.
\cref{tab:lenet-arch} details the layer configuration.

\begin{table}[h]
    \centering
    \caption{LeNet5 architecture used for \MNISTAdd and its variants.}
    \label{tab:lenet-arch}
    \scalebox{0.9}{
    \begin{tabular}{llll}
        \toprule
        \textsc{Input} & \textsc{Layer Type} & \textsc{Parameter} & \textsc{Activation} \\
        \midrule
        $(1, 28, 28)$   & Conv2d     & depth$=6$, kernel$=5$, stride$=1$, padding$=2$  & Tanh \\
        $(6, 28, 28)$   & AvgPool2d  & kernel$=2$, stride$=2$                          & \\
        $(6, 14, 14)$   & Conv2d     & depth$=16$, kernel$=5$, stride$=1$              & Tanh \\
        $(16, 10, 10)$  & AvgPool2d  & kernel$=2$, stride$=2$                          & \\
        $(16, 5, 5)$    & Flatten    & dim$=400$                                        & \\
        $(400)$         & Linear     & $400 \to 120$                                   & Tanh \\
        $(120)$         & Linear     & $120 \to 84$                                    & Tanh \\
        $(84)$          & Linear     & $84 \to \mathtt{num\_concepts}$                  & \\
        \bottomrule
    \end{tabular}
    }
\end{table}

\textbf{ResNet18}.
ResNet18~\citep{he2016deep} is used as backbone for higher-resolution image datasets (\ie \CHX, \CIFAR, and \RIVAL).
We use weights pretrained on ImageNet~\citep{deng2009imagenet} for both \CIFAR and \CHX, but not for \RIVAL, as its images are sourced from the same \texttt{ImageNet}~\citep{deng2009imagenet} dataset on which ResNet-18 was pretrained.
For all datasets, we resize the images to $224 \times 224$ to make them compatible with the ResNet18 architecture.
The final fully connected layer is replaced with a linear layer mapping to $\mathtt{num\_concepts}$ outputs.
\cref{tab:resnet-arch} details the layer configuration.

\begin{table}[h]
    \centering
    \caption{ResNet18 architecture used for \CHX, \CIFAR, and \RIVAL.}
    \label{tab:resnet-arch}
    \scalebox{0.8}{
    \begin{tabular}{llll}
        \toprule
        \textsc{Input} & \textsc{Layer Type} & \textsc{Parameter} & \textsc{Activation} \\
        \midrule
        $(3, 224, 224)$     & Conv2d              & depth$=64$, kernel$=7$, stride$=2$, padding$=3$                    & \\
        $(64, 112, 112)$    & BatchNorm2d         & dim$=64$                                                           & ReLU \\
        $(64, 112, 112)$    & MaxPool2d           & kernel$=3$, stride$=2$, padding$=1$                                & \\
        $(64, 56, 56)$      & 2$\times$BasicBlock & depth$=64$, kernel$=(1,1)$, stride$=(1,1)$, padding$=(1,1)$        & \\
        $(128, 28, 28)$     & 2$\times$BasicBlock & depth$=128$, kernel$=(1,1)$, stride$=(2,2)$, padding$=(1,1)$       & \\
        $(256, 14, 14)$     & 2$\times$BasicBlock & depth$=256$, kernel$=(1,1)$, stride$=(2,2)$, padding$=(1,1)$       & \\
        $(512, 7, 7)$       & 2$\times$BasicBlock & depth$=512$, kernel$=(1,1)$, stride$=(2,2)$, padding$=(1,1)$       & \\
        $(512, 7, 7)$       & AvgPool2d           & dim$=(1,1)$                                                        & \\
        $(512, 1, 1)$       & Flatten             & dim$=512$                                                          & \\
        $(512)$             & Linear              & $512 \to \mathtt{num\_concepts}$                                    & \\
        \bottomrule
    \end{tabular}
    }
\end{table}

\textbf{BERT}.
\texttt{bert-base-uncased}~\citep{devlin2019bert} is used as backbone for the text dataset \CEBAB.
We adopt the standard pretrained \texttt{bert-base-uncased} encoder, freezing its weights, and append a trainable MLP classification head producing $\mathtt{num\_concepts} \times \mathtt{num\_concept\_dim}$ logits, which are reshaped into per-concept predictions of dimension $\mathtt{num\_concept\_dim}$.
\cref{tab:bert-arch} details the layer configuration.
\cref{tab:bert-arch} details the layer configuration.

\begin{table}[h]
    \centering
    \caption{BERT architecture used for \CEBAB.}
    \label{tab:bert-arch}
    \scalebox{0.7}{
\begin{tabular}{llll}
        \toprule
        \textsc{Input} & \textsc{Layer Type} & \textsc{Parameter} & \textsc{Activation} \\
        \midrule
        Token sequence               & \texttt{bert-base-uncased} (frozen) & hidden dim$=768$, layers$=12$, heads$=12$         & \\
        \texttt{[CLS]} token $(768)$ & Linear                              & $768 \to 256$                                     & ReLU \\
        $(256)$                      & Dropout                             & $p=0.1$                                           & \\
        $(256)$                      & Linear                              & $256 \to 128$                                     & ReLU \\
        $(128)$                      & Dropout                             & $p=0.1$                                           & \\
        $(128)$                      & Linear                              & $128 \to \mathtt{num\_concepts} \times \mathtt{num\_concept\_dim}$ & \\
        $(\mathtt{num\_concepts} \times \mathtt{num\_concept\_dim})$ & Reshape & $\to (\mathtt{num\_concepts},\, \mathtt{num\_concept\_dim})$ & \\
        \bottomrule
    \end{tabular}
    }
\end{table}

\section{Additional Experiments}
\label{sec:additional-experiments}

In this section, we report experiments on additional datasets and under different settings (with or without concept supervision), as well as configurations involving e-values. All tables, by default, use concept supervision unless otherwise specified. We also provide two ablation studies on the proposed method for encoding conformal predictions in NeSy-CBMs, along with a runtime analysis.

\begin{table*}[!h]
    \centering
    \scriptsize
\begin{tabular}{lcccccc}
\toprule
& \multicolumn{3}{c}{{\sc Concepts}} & \multicolumn{3}{c}{{\sc Labels}} \\
\cmidrule(lr){2-4} \cmidrule(lr){5-7} {\sc Method} & {\sc Const.} & {\sc Size} & {\sc Cov.} & {\sc Const.} & {\sc Size} & {\sc Cov.} \\
\midrule
\DPL & $1.00 \pm 0.00$ & $1.00 \pm 0.00$ & $0.98 \pm 0.00$ & $1.00 \pm 0.00$ & $1.00 \pm 0.00$ & $0.98 \pm 0.00$ \\
\rowcolor[HTML]{EFEFEF}
\DPL + \TaskOnly & $0.90 \pm 0.01$ & $1.00 \pm 0.00$ & $0.98 \pm 0.00$ & $1.00 \pm 0.00$ & $0.90 \pm 0.01$ & $0.90 \pm 0.01$ \\
\DPL + \TaskPlusAbduction & $1.00 \pm 0.00$ & $6.01 \pm 0.05$ & $0.90 \pm 0.01$ & $1.00 \pm 0.00$ & $0.90 \pm 0.01$ & $0.90 \pm 0.01$ \\
\rowcolor[HTML]{EFEFEF}
\DPL + \ConceptsOnly & $1.00 \pm 0.00$ & $1.00 \pm 0.00$ & $0.90 \pm 0.01$ & $0.91 \pm 0.01$ & $1.00 \pm 0.00$ & $0.98 \pm 0.00$ \\
\DPL + \ConceptsPlusDeduction & $1.00 \pm 0.00$ & $1.00 \pm 0.00$ & $0.90 \pm 0.01$ & $1.00 \pm 0.00$ & $0.91 \pm 0.01$ & $0.90 \pm 0.01$ \\
\rowcolor[HTML]{EFEFEF}
\DPL + \bastani & $1.00 \pm 0.00$ & $1.00 \pm 0.00$ & $0.90 \pm 0.01$ & $1.00 \pm 0.00$ & $0.90 \pm 0.01$ & $0.90 \pm 0.01$ \\
\DPL + \method & $1.00 \pm 0.00$ & $0.99 \pm 0.00$ & $0.90 \pm 0.01$ & $1.00 \pm 0.00$ & $0.90 \pm 0.01$ & $0.90 \pm 0.01$ \\
\bottomrule
\end{tabular}

    \caption{Results for \DPL (without concept supervision) on \MNISTAdd.}
    \label{tab:dpl-mnist-add-no-cs}
\end{table*}

\begin{table*}[!h]
    \centering
    \scriptsize
\begin{tabular}{lcccccc}
\toprule
& \multicolumn{3}{c}{{\sc Concepts}} & \multicolumn{3}{c}{{\sc Labels}} \\
\cmidrule(lr){2-4} \cmidrule(lr){5-7} {\sc Method} & {\sc Const.} & {\sc Size} & {\sc Cov.} & {\sc Const.} & {\sc Size} & {\sc Cov.} \\
\midrule
\LTN & $1.00 \pm 0.00$ & $1.00 \pm 0.00$ & $0.97 \pm 0.00$ & $1.00 \pm 0.00$ & $1.00 \pm 0.00$ & $0.97 \pm 0.00$ \\
\rowcolor[HTML]{EFEFEF}
\LTN + \TaskOnly & $1.00 \pm 0.00$ & $1.00 \pm 0.00$ & $0.97 \pm 0.00$ & $1.00 \pm 0.00$ & $1.00 \pm 0.00$ & $0.97 \pm 0.00$ \\
\LTN + \TaskPlusAbduction & $1.00 \pm 0.00$ & $6.66 \pm 0.02$ & $0.97 \pm 0.00$ & $1.00 \pm 0.00$ & $1.00 \pm 0.00$ & $0.97 \pm 0.00$ \\
\rowcolor[HTML]{EFEFEF}
\LTN + \ConceptsOnly & $1.00 \pm 0.00$ & $1.00 \pm 0.00$ & $0.90 \pm 0.01$ & $0.90 \pm 0.01$ & $1.00 \pm 0.00$ & $0.97 \pm 0.00$ \\
\LTN + \ConceptsPlusDeduction & $1.00 \pm 0.00$ & $1.00 \pm 0.00$ & $0.90 \pm 0.01$ & $1.00 \pm 0.00$ & $0.90 \pm 0.01$ & $0.90 \pm 0.01$ \\
\rowcolor[HTML]{EFEFEF}
\LTN + \bastani & $1.00 \pm 0.00$ & $1.00 \pm 0.00$ & $0.90 \pm 0.01$ & $1.00 \pm 0.00$ & $0.90 \pm 0.01$ & $0.90 \pm 0.01$ \\
\LTN + \method & $1.00 \pm 0.00$ & $0.90 \pm 0.01$ & $0.90 \pm 0.01$ & $1.00 \pm 0.00$ & $0.90 \pm 0.01$ & $0.90 \pm 0.01$ \\
\bottomrule
\end{tabular}

    \caption{Results for \LTN (without concept supervision) on \MNISTAdd.}
    \label{tab:ltn-mnist-add-no-cs}
\end{table*}

\begin{table*}[!h]
    \centering
    \scriptsize
\begin{tabular}{lcccccc}
\toprule
& \multicolumn{3}{c}{{\sc Concepts}} & \multicolumn{3}{c}{{\sc Labels}} \\
\cmidrule(lr){2-4} \cmidrule(lr){5-7} {\sc Method} & {\sc Const.} & {\sc Size} & {\sc Cov.} & {\sc Const.} & {\sc Size} & {\sc Cov.} \\
\midrule
\DPL & $1.00 \pm 0.00$ & $1.00 \pm 0.00$ & $0.00 \pm 0.00$ & $1.00 \pm 0.00$ & $1.00 \pm 0.00$ & $0.61 \pm 0.01$ \\
\rowcolor[HTML]{EFEFEF}
\DPL + \TaskOnly & $1.00 \pm 0.00$ & $1.00 \pm 0.00$ & $0.00 \pm 0.00$ & $0.33 \pm 0.14$ & $3.59 \pm 1.49$ & $0.94 \pm 0.06$ \\
\DPL + \TaskPlusAbduction & $1.00 \pm 0.00$ & $742.55 \pm 296.73$ & $0.94 \pm 0.06$ & $1.00 \pm 0.00$ & $3.59 \pm 1.49$ & $0.94 \pm 0.06$ \\
\rowcolor[HTML]{EFEFEF}
\DPL + \ConceptsOnly & $0.23 \pm 0.01$ & $907.88 \pm 76.19$ & $0.97 \pm 0.02$ & $1.00 \pm 0.00$ & $1.00 \pm 0.00$ & $0.61 \pm 0.01$ \\
\DPL + \ConceptsPlusDeduction & $1.00 \pm 0.00$ & $907.88 \pm 76.19$ & $0.97 \pm 0.02$ & $1.00 \pm 0.00$ & $5.00 \pm 0.00$ & $1.00 \pm 0.00$ \\
\rowcolor[HTML]{EFEFEF}
\DPL + \bastani & $0.74 \pm 0.27$ & $907.88 \pm 76.19$ & $0.97 \pm 0.02$ & $1.00 \pm 0.00$ & $3.59 \pm 1.49$ & $0.94 \pm 0.06$ \\
\DPL + \method & $1.00 \pm 0.00$ & $670.15 \pm 252.46$ & $0.92 \pm 0.06$ & $1.00 \pm 0.00$ & $3.59 \pm 1.49$ & $0.94 \pm 0.06$ \\
\bottomrule
\end{tabular}

    \caption{Results for \DPL (without concept supervision) on \CEBAB.}
    \label{tab:dpl-cebab-no-cs}
\end{table*}

\begin{table*}[!h]
    \centering
    \scriptsize
\begin{tabular}{lcccccc}
\toprule
& \multicolumn{3}{c}{{\sc Concepts}} & \multicolumn{3}{c}{{\sc Labels}} \\
\cmidrule(lr){2-4} \cmidrule(lr){5-7} {\sc Method} & {\sc Const.} & {\sc Size} & {\sc Cov.} & {\sc Const.} & {\sc Size} & {\sc Cov.} \\
\midrule
\DPL & $1.00 \pm 0.00$ & $1.00 \pm 0.00$ & $0.18 \pm 0.02$ & $1.00 \pm 0.00$ & $1.00 \pm 0.00$ & $0.58 \pm 0.03$ \\
\rowcolor[HTML]{EFEFEF}
\DPL + \TaskOnly & $1.00 \pm 0.00$ & $1.00 \pm 0.00$ & $0.18 \pm 0.02$ & $0.42 \pm 0.03$ & $2.41 \pm 0.18$ & $0.90 \pm 0.02$ \\
\DPL + \TaskPlusAbduction & $1.00 \pm 0.00$ & $504.38 \pm 34.71$ & $0.90 \pm 0.02$ & $1.00 \pm 0.00$ & $2.41 \pm 0.18$ & $0.90 \pm 0.02$ \\
\rowcolor[HTML]{EFEFEF}
\DPL + \ConceptsOnly & $0.39 \pm 0.05$ & $75.06 \pm 38.35$ & $0.90 \pm 0.04$ & $0.97 \pm 0.01$ & $1.00 \pm 0.00$ & $0.58 \pm 0.03$ \\
\DPL + \ConceptsPlusDeduction & $1.00 \pm 0.00$ & $75.06 \pm 38.35$ & $0.90 \pm 0.04$ & $1.00 \pm 0.00$ & $3.71 \pm 0.65$ & $0.96 \pm 0.02$ \\
\rowcolor[HTML]{EFEFEF}
\DPL + \bastani & $0.78 \pm 0.06$ & $75.06 \pm 38.35$ & $0.90 \pm 0.04$ & $1.00 \pm 0.00$ & $2.31 \pm 0.16$ & $0.87 \pm 0.01$ \\
\DPL + \method & $1.00 \pm 0.00$ & $53.49 \pm 25.93$ & $0.83 \pm 0.03$ & $1.00 \pm 0.00$ & $2.31 \pm 0.16$ & $0.87 \pm 0.01$ \\
\bottomrule
\end{tabular}

    \caption{Results for \DPL (with concept supervision) on \CEBAB.}
    \label{tab:dpl-cebab}
\end{table*}

\begin{table*}[!h]
    \centering
    \scriptsize
\begin{tabular}{lcccccc}
\toprule
& \multicolumn{3}{c}{{\sc Concepts}} & \multicolumn{3}{c}{{\sc Labels}} \\
\cmidrule(lr){2-4} \cmidrule(lr){5-7} {\sc Method} & {\sc Const.} & {\sc Size} & {\sc Cov.} & {\sc Const.} & {\sc Size} & {\sc Cov.} \\
\midrule
\LTN & $1.00 \pm 0.00$ & $1.00 \pm 0.00$ & $0.00 \pm 0.00$ & $1.00 \pm 0.00$ & $1.00 \pm 0.00$ & $0.62 \pm 0.01$ \\
\rowcolor[HTML]{EFEFEF}
\LTN + \TaskOnly & $1.00 \pm 0.00$ & $1.00 \pm 0.00$ & $0.00 \pm 0.00$ & $0.20 \pm 0.00$ & $5.00 \pm 0.00$ & $1.00 \pm 0.00$ \\
\LTN + \TaskPlusAbduction & $1.00 \pm 0.00$ & $1024.00 \pm 0.00$ & $1.00 \pm 0.00$ & $1.00 \pm 0.00$ & $5.00 \pm 0.00$ & $1.00 \pm 0.00$ \\
\rowcolor[HTML]{EFEFEF}
\LTN + \ConceptsOnly & $0.30 \pm 0.01$ & $349.70 \pm 34.12$ & $0.91 \pm 0.02$ & $1.00 \pm 0.00$ & $1.00 \pm 0.00$ & $0.62 \pm 0.01$ \\
\LTN + \ConceptsPlusDeduction & $1.00 \pm 0.00$ & $349.70 \pm 34.12$ & $0.91 \pm 0.02$ & $1.00 \pm 0.00$ & $4.31 \pm 0.18$ & $0.99 \pm 0.01$ \\
\rowcolor[HTML]{EFEFEF}
\LTN + \bastani & $1.00 \pm 0.00$ & $349.70 \pm 34.12$ & $0.91 \pm 0.02$ & $1.00 \pm 0.00$ & $4.31 \pm 0.18$ & $0.99 \pm 0.01$ \\
\LTN + \method & $1.00 \pm 0.00$ & $349.70 \pm 34.12$ & $0.91 \pm 0.02$ & $1.00 \pm 0.00$ & $4.31 \pm 0.18$ & $0.99 \pm 0.01$ \\
\bottomrule
\end{tabular}

    \caption{Results for \LTN (without concept supervision) on \CEBAB.}
    \label{tab:ltn-cebab-no-cs}
\end{table*}

\begin{table*}[!h]
    \centering
    \scriptsize
\begin{tabular}{lcccccc}
\toprule
& \multicolumn{3}{c}{{\sc Concepts}} & \multicolumn{3}{c}{{\sc Labels}} \\
\cmidrule(lr){2-4} \cmidrule(lr){5-7} {\sc Method} & {\sc Const.} & {\sc Size} & {\sc Cov.} & {\sc Const.} & {\sc Size} & {\sc Cov.} \\
\midrule
\LTN & $1.00 \pm 0.00$ & $1.00 \pm 0.00$ & $0.24 \pm 0.01$ & $1.00 \pm 0.00$ & $1.00 \pm 0.00$ & $0.63 \pm 0.01$ \\
\rowcolor[HTML]{EFEFEF}
\LTN + \TaskOnly & $1.00 \pm 0.00$ & $1.00 \pm 0.00$ & $0.24 \pm 0.01$ & $0.20 \pm 0.00$ & $5.00 \pm 0.00$ & $1.00 \pm 0.00$ \\
\LTN + \TaskPlusAbduction & $1.00 \pm 0.00$ & $1024.00 \pm 0.00$ & $1.00 \pm 0.00$ & $1.00 \pm 0.00$ & $5.00 \pm 0.00$ & $1.00 \pm 0.00$ \\
\rowcolor[HTML]{EFEFEF}
\LTN + \ConceptsOnly & $0.41 \pm 0.02$ & $36.19 \pm 4.07$ & $0.90 \pm 0.02$ & $0.98 \pm 0.01$ & $1.00 \pm 0.00$ & $0.63 \pm 0.01$ \\
\LTN + \ConceptsPlusDeduction & $1.00 \pm 0.00$ & $36.19 \pm 4.07$ & $0.90 \pm 0.02$ & $1.00 \pm 0.00$ & $3.13 \pm 0.13$ & $0.96 \pm 0.01$ \\
\rowcolor[HTML]{EFEFEF}
\LTN + \bastani & $1.00 \pm 0.00$ & $36.19 \pm 4.07$ & $0.90 \pm 0.02$ & $1.00 \pm 0.00$ & $3.13 \pm 0.13$ & $0.96 \pm 0.01$ \\
\LTN + \method & $1.00 \pm 0.00$ & $36.05 \pm 4.08$ & $0.90 \pm 0.02$ & $1.00 \pm 0.00$ & $3.13 \pm 0.13$ & $0.96 \pm 0.01$ \\
\bottomrule
\end{tabular}

    \caption{Results for \LTN (with concept supervision) on \CEBAB.}
    \label{tab:ltn-cebab}
\end{table*}

\begin{table*}[!h]
    \centering
    \scriptsize
\begin{tabular}{lcccccc}
\toprule
& \multicolumn{3}{c}{{\sc Concepts}} & \multicolumn{3}{c}{{\sc Labels}} \\
\cmidrule(lr){2-4} \cmidrule(lr){5-7} {\sc Method} & {\sc Const.} & {\sc Size} & {\sc Cov.} & {\sc Const.} & {\sc Size} & {\sc Cov.} \\
\midrule
\DPL & $1.00 \pm 0.00$ & $1.00 \pm 0.00$ & $0.45 \pm 0.09$ & $1.00 \pm 0.00$ & $1.00 \pm 0.00$ & $0.73 \pm 0.01$ \\
\rowcolor[HTML]{EFEFEF}
\DPL + \TaskOnly & $1.00 \pm 0.00$ & $1.00 \pm 0.00$ & $0.45 \pm 0.09$ & $0.61 \pm 0.08$ & $1.67 \pm 0.21$ & $0.90 \pm 0.03$ \\
\DPL + \TaskPlusAbduction & $1.00 \pm 0.00$ & $5.05 \pm 0.92$ & $0.90 \pm 0.03$ & $1.00 \pm 0.00$ & $1.67 \pm 0.21$ & $0.90 \pm 0.03$ \\
\rowcolor[HTML]{EFEFEF}
\DPL + \ConceptsOnly & $0.28 \pm 0.06$ & $6.85 \pm 2.53$ & $0.91 \pm 0.02$ & $1.00 \pm 0.01$ & $1.00 \pm 0.00$ & $0.73 \pm 0.01$ \\
\DPL + \ConceptsPlusDeduction & $1.00 \pm 0.00$ & $6.85 \pm 2.53$ & $0.91 \pm 0.02$ & $1.00 \pm 0.00$ & $3.50 \pm 0.49$ & $0.99 \pm 0.01$ \\
\rowcolor[HTML]{EFEFEF}
\DPL + \bastani & $0.52 \pm 0.12$ & $6.85 \pm 2.53$ & $0.91 \pm 0.02$ & $1.00 \pm 0.00$ & $1.66 \pm 0.22$ & $0.90 \pm 0.04$ \\
\DPL + \method & $1.00 \pm 0.00$ & $3.36 \pm 0.80$ & $0.84 \pm 0.04$ & $1.00 \pm 0.00$ & $1.66 \pm 0.22$ & $0.90 \pm 0.04$ \\
\bottomrule
\end{tabular}

    \caption{Results for \DPL (without concept supervision) on \CHX.}
    \label{tab:dpl-chx-no-cs}
\end{table*}

\begin{table*}[!h]
    \centering
    \scriptsize
\begin{tabular}{lcccccc}
\toprule
& \multicolumn{3}{c}{{\sc Concepts}} & \multicolumn{3}{c}{{\sc Labels}} \\
\cmidrule(lr){2-4} \cmidrule(lr){5-7} {\sc Method} & {\sc Const.} & {\sc Size} & {\sc Cov.} & {\sc Const.} & {\sc Size} & {\sc Cov.} \\
\midrule
\DPL & $1.00 \pm 0.00$ & $1.00 \pm 0.00$ & $0.47 \pm 0.11$ & $1.00 \pm 0.00$ & $1.00 \pm 0.00$ & $0.72 \pm 0.01$ \\
\rowcolor[HTML]{EFEFEF}
\DPL + \TaskOnly & $1.00 \pm 0.00$ & $1.00 \pm 0.00$ & $0.47 \pm 0.11$ & $0.62 \pm 0.07$ & $1.63 \pm 0.19$ & $0.89 \pm 0.03$ \\
\DPL + \TaskPlusAbduction & $1.00 \pm 0.00$ & $4.96 \pm 0.86$ & $0.89 \pm 0.03$ & $1.00 \pm 0.00$ & $1.63 \pm 0.19$ & $0.89 \pm 0.03$ \\
\rowcolor[HTML]{EFEFEF}
\DPL + \ConceptsOnly & $0.31 \pm 0.04$ & $5.30 \pm 1.01$ & $0.91 \pm 0.03$ & $0.99 \pm 0.01$ & $1.00 \pm 0.00$ & $0.72 \pm 0.01$ \\
\DPL + \ConceptsPlusDeduction & $1.00 \pm 0.00$ & $5.30 \pm 1.01$ & $0.91 \pm 0.03$ & $1.00 \pm 0.00$ & $3.19 \pm 0.30$ & $0.98 \pm 0.01$ \\
\rowcolor[HTML]{EFEFEF}
\DPL + \bastani & $0.55 \pm 0.05$ & $5.30 \pm 1.01$ & $0.91 \pm 0.03$ & $1.00 \pm 0.00$ & $1.60 \pm 0.18$ & $0.88 \pm 0.03$ \\
\DPL + \method & $1.00 \pm 0.00$ & $2.90 \pm 0.56$ & $0.83 \pm 0.04$ & $1.00 \pm 0.00$ & $1.60 \pm 0.18$ & $0.88 \pm 0.03$ \\
\bottomrule
\end{tabular}

    \caption{Results for \DPL (with concept supervision) on \CHX.}
    \label{tab:dpl-chx}
\end{table*}

\begin{table*}[!h]
    \centering
    \scriptsize
\begin{tabular}{lcccccc}
\toprule
& \multicolumn{3}{c}{{\sc Concepts}} & \multicolumn{3}{c}{{\sc Labels}} \\
\cmidrule(lr){2-4} \cmidrule(lr){5-7} {\sc Method} & {\sc Const.} & {\sc Size} & {\sc Cov.} & {\sc Const.} & {\sc Size} & {\sc Cov.} \\
\midrule
\LTN & $1.00 \pm 0.00$ & $1.00 \pm 0.00$ & $0.45 \pm 0.10$ & $1.00 \pm 0.00$ & $1.00 \pm 0.00$ & $0.73 \pm 0.02$ \\
\rowcolor[HTML]{EFEFEF}
\LTN + \TaskOnly & $1.00 \pm 0.00$ & $1.00 \pm 0.00$ & $0.45 \pm 0.10$ & $0.20 \pm 0.00$ & $5.00 \pm 0.00$ & $1.00 \pm 0.00$ \\
\LTN + \TaskPlusAbduction & $1.00 \pm 0.00$ & $16.00 \pm 0.00$ & $1.00 \pm 0.00$ & $1.00 \pm 0.00$ & $5.00 \pm 0.00$ & $1.00 \pm 0.00$ \\
\rowcolor[HTML]{EFEFEF}
\LTN + \ConceptsOnly & $0.24 \pm 0.05$ & $7.58 \pm 2.49$ & $0.91 \pm 0.02$ & $0.99 \pm 0.01$ & $1.00 \pm 0.00$ & $0.73 \pm 0.02$ \\
\LTN + \ConceptsPlusDeduction & $1.00 \pm 0.00$ & $7.58 \pm 2.49$ & $0.91 \pm 0.02$ & $1.00 \pm 0.00$ & $3.68 \pm 0.44$ & $0.98 \pm 0.01$ \\
\rowcolor[HTML]{EFEFEF}
\LTN + \bastani & $1.00 \pm 0.00$ & $7.58 \pm 2.49$ & $0.91 \pm 0.02$ & $1.00 \pm 0.00$ & $3.68 \pm 0.44$ & $0.98 \pm 0.01$ \\
\LTN + \method & $1.00 \pm 0.00$ & $7.54 \pm 2.51$ & $0.91 \pm 0.02$ & $1.00 \pm 0.00$ & $3.68 \pm 0.44$ & $0.98 \pm 0.01$ \\
\bottomrule
\end{tabular}

    \caption{Results for \LTN (without concept supervision) on \CHX.}
    \label{tab:ltn-chx-no-cs}
\end{table*}

\begin{table*}[!h]
    \centering
    \scriptsize
\begin{tabular}{lcccccc}
\toprule
& \multicolumn{3}{c}{{\sc Concepts}} & \multicolumn{3}{c}{{\sc Labels}} \\
\cmidrule(lr){2-4} \cmidrule(lr){5-7} {\sc Method} & {\sc Const.} & {\sc Size} & {\sc Cov.} & {\sc Const.} & {\sc Size} & {\sc Cov.} \\
\midrule
\LTN & $1.00 \pm 0.00$ & $1.00 \pm 0.00$ & $0.46 \pm 0.09$ & $1.00 \pm 0.00$ & $1.00 \pm 0.00$ & $0.74 \pm 0.02$ \\
\rowcolor[HTML]{EFEFEF}
\LTN + \TaskOnly & $1.00 \pm 0.00$ & $1.00 \pm 0.00$ & $0.46 \pm 0.09$ & $0.20 \pm 0.00$ & $5.00 \pm 0.00$ & $1.00 \pm 0.00$ \\
\LTN + \TaskPlusAbduction & $1.00 \pm 0.00$ & $16.00 \pm 0.00$ & $1.00 \pm 0.00$ & $1.00 \pm 0.00$ & $5.00 \pm 0.00$ & $1.00 \pm 0.00$ \\
\rowcolor[HTML]{EFEFEF}
\LTN + \ConceptsOnly & $0.26 \pm 0.04$ & $6.58 \pm 1.64$ & $0.92 \pm 0.01$ & $1.00 \pm 0.00$ & $1.00 \pm 0.00$ & $0.74 \pm 0.02$ \\
\LTN + \ConceptsPlusDeduction & $1.00 \pm 0.00$ & $6.58 \pm 1.64$ & $0.92 \pm 0.01$ & $1.00 \pm 0.00$ & $3.53 \pm 0.35$ & $0.99 \pm 0.01$ \\
\rowcolor[HTML]{EFEFEF}
\LTN + \bastani & $1.00 \pm 0.00$ & $6.58 \pm 1.64$ & $0.92 \pm 0.01$ & $1.00 \pm 0.00$ & $3.53 \pm 0.35$ & $0.99 \pm 0.01$ \\
\LTN + \method & $1.00 \pm 0.00$ & $6.56 \pm 1.66$ & $0.92 \pm 0.01$ & $1.00 \pm 0.00$ & $3.53 \pm 0.35$ & $0.99 \pm 0.01$ \\
\bottomrule
\end{tabular}

    \caption{Results for \LTN (with concept supervision) on \CHX.}
    \label{tab:ltn-chx}
\end{table*}

\begin{table*}[!h]
    \centering
    \scriptsize
\begin{tabular}{lcccccc}
\toprule
& \multicolumn{3}{c}{{\sc Concepts}} & \multicolumn{3}{c}{{\sc Labels}} \\
\cmidrule(lr){2-4} \cmidrule(lr){5-7} {\sc Method} & {\sc Const.} & {\sc Size} & {\sc Cov.} & {\sc Const.} & {\sc Size} & {\sc Cov.} \\
\midrule
\DPL & $0.97 \pm 0.05$ & $1.00 \pm 0.00$ & $0.34 \pm 0.02$ & $0.97 \pm 0.05$ & $1.00 \pm 0.00$ & $0.66 \pm 0.00$ \\
\rowcolor[HTML]{EFEFEF}
\DPL + \TaskOnly & $0.99 \pm 0.01$ & $1.00 \pm 0.00$ & $0.34 \pm 0.02$ & $0.84 \pm 0.05$ & $1.83 \pm 0.08$ & $0.90 \pm 0.01$ \\
\DPL + \TaskPlusAbduction & $1.00 \pm 0.00$ & $171.66 \pm 7.97$ & $0.90 \pm 0.01$ & $1.00 \pm 0.00$ & $1.83 \pm 0.08$ & $0.90 \pm 0.01$ \\
\rowcolor[HTML]{EFEFEF}
\DPL + \ConceptsOnly & $0.77 \pm 0.05$ & $57.60 \pm 24.96$ & $0.97 \pm 0.01$ & $1.00 \pm 0.00$ & $1.00 \pm 0.00$ & $0.66 \pm 0.00$ \\
\DPL + \ConceptsPlusDeduction & $1.00 \pm 0.00$ & $57.60 \pm 24.96$ & $0.97 \pm 0.01$ & $1.00 \pm 0.00$ & $5.54 \pm 0.98$ & $0.78 \pm 0.01$ \\
\rowcolor[HTML]{EFEFEF}
\DPL + \bastani & $0.79 \pm 0.05$ & $57.60 \pm 24.96$ & $0.97 \pm 0.01$ & $1.00 \pm 0.00$ & $1.83 \pm 0.08$ & $0.90 \pm 0.01$ \\
\DPL + \method & $1.00 \pm 0.00$ & $45.43 \pm 19.68$ & $0.90 \pm 0.01$ & $1.00 \pm 0.00$ & $1.83 \pm 0.08$ & $0.90 \pm 0.01$ \\
\bottomrule
\end{tabular}

    \caption{Results for \DPL (without concept supervision) on \CIFAR.}
    \label{tab:dpl-cifar-no-cs}
\end{table*}

\begin{table*}[!h]
    \centering
    \scriptsize
\begin{tabular}{lcccccc}
\toprule
& \multicolumn{3}{c}{{\sc Concepts}} & \multicolumn{3}{c}{{\sc Labels}} \\
\cmidrule(lr){2-4} \cmidrule(lr){5-7} {\sc Method} & {\sc Const.} & {\sc Size} & {\sc Cov.} & {\sc Const.} & {\sc Size} & {\sc Cov.} \\
\midrule
\DPL & $1.00 \pm 0.00$ & $1.00 \pm 0.00$ & $0.34 \pm 0.01$ & $1.00 \pm 0.00$ & $1.00 \pm 0.00$ & $0.66 \pm 0.00$ \\
\rowcolor[HTML]{EFEFEF}
\DPL + \TaskOnly & $1.00 \pm 0.00$ & $1.00 \pm 0.00$ & $0.34 \pm 0.01$ & $0.88 \pm 0.02$ & $1.76 \pm 0.06$ & $0.90 \pm 0.01$ \\
\DPL + \TaskPlusAbduction & $1.00 \pm 0.00$ & $164.21 \pm 5.73$ & $0.90 \pm 0.01$ & $1.00 \pm 0.00$ & $1.76 \pm 0.06$ & $0.90 \pm 0.01$ \\
\rowcolor[HTML]{EFEFEF}
\DPL + \ConceptsOnly & $0.88 \pm 0.02$ & $20.75 \pm 2.33$ & $0.93 \pm 0.01$ & $1.00 \pm 0.00$ & $1.00 \pm 0.00$ & $0.66 \pm 0.00$ \\
\DPL + \ConceptsPlusDeduction & $1.00 \pm 0.00$ & $20.75 \pm 2.33$ & $0.93 \pm 0.01$ & $1.00 \pm 0.00$ & $2.89 \pm 0.25$ & $0.75 \pm 0.00$ \\
\rowcolor[HTML]{EFEFEF}
\DPL + \bastani & $0.91 \pm 0.02$ & $20.75 \pm 2.33$ & $0.93 \pm 0.01$ & $1.00 \pm 0.00$ & $1.76 \pm 0.06$ & $0.90 \pm 0.01$ \\
\DPL + \method & $1.00 \pm 0.00$ & $18.87 \pm 2.12$ & $0.88 \pm 0.01$ & $1.00 \pm 0.00$ & $1.76 \pm 0.06$ & $0.90 \pm 0.01$ \\
\bottomrule
\end{tabular}

    \caption{Results for \DPL (with concept supervision) on \CIFAR.}
    \label{tab:dpl-cifar}
\end{table*}

\begin{table*}[!h]
    \centering
    \scriptsize
\begin{tabular}{lcccccc}
\toprule
& \multicolumn{3}{c}{{\sc Concepts}} & \multicolumn{3}{c}{{\sc Labels}} \\
\cmidrule(lr){2-4} \cmidrule(lr){5-7} {\sc Method} & {\sc Const.} & {\sc Size} & {\sc Cov.} & {\sc Const.} & {\sc Size} & {\sc Cov.} \\
\midrule
\LTN & $0.98 \pm 0.00$ & $1.00 \pm 0.00$ & $0.13 \pm 0.04$ & $0.98 \pm 0.00$ & $1.00 \pm 0.00$ & $0.50 \pm 0.03$ \\
\rowcolor[HTML]{EFEFEF}
\LTN + \TaskOnly & $1.00 \pm 0.00$ & $1.00 \pm 0.00$ & $0.13 \pm 0.04$ & $0.73 \pm 0.05$ & $3.09 \pm 0.16$ & $0.90 \pm 0.01$ \\
\LTN + \TaskPlusAbduction & $1.00 \pm 0.00$ & $288.11 \pm 15.71$ & $0.90 \pm 0.00$ & $1.00 \pm 0.00$ & $3.09 \pm 0.16$ & $0.90 \pm 0.01$ \\
\rowcolor[HTML]{EFEFEF}
\LTN + \ConceptsOnly & $0.84 \pm 0.04$ & $28.83 \pm 3.66$ & $0.93 \pm 0.01$ & $1.00 \pm 0.00$ & $1.00 \pm 0.00$ & $0.50 \pm 0.03$ \\
\LTN + \ConceptsPlusDeduction & $1.00 \pm 0.00$ & $28.83 \pm 3.66$ & $0.93 \pm 0.01$ & $1.00 \pm 0.00$ & $3.15 \pm 0.11$ & $0.75 \pm 0.01$ \\
\rowcolor[HTML]{EFEFEF}
\LTN + \bastani & $0.94 \pm 0.03$ & $28.83 \pm 3.66$ & $0.93 \pm 0.01$ & $1.00 \pm 0.00$ & $3.08 \pm 0.17$ & $0.90 \pm 0.01$ \\
\LTN + \method & $1.00 \pm 0.00$ & $27.09 \pm 2.77$ & $0.88 \pm 0.01$ & $1.00 \pm 0.00$ & $3.08 \pm 0.17$ & $0.90 \pm 0.01$ \\
\bottomrule
\end{tabular}

    \caption{Results for \LTN (without concept supervision) on \CIFAR.}
    \label{tab:ltn-cifar-no-cs}
\end{table*}

\begin{table*}[!h]
    \centering
    \scriptsize
\begin{tabular}{lcccccc}
\toprule
& \multicolumn{3}{c}{{\sc Concepts}} & \multicolumn{3}{c}{{\sc Labels}} \\
\cmidrule(lr){2-4} \cmidrule(lr){5-7} {\sc Method} & {\sc Const.} & {\sc Size} & {\sc Cov.} & {\sc Const.} & {\sc Size} & {\sc Cov.} \\
\midrule
\LTN & $0.98 \pm 0.00$ & $1.00 \pm 0.00$ & $0.31 \pm 0.01$ & $0.98 \pm 0.00$ & $1.00 \pm 0.00$ & $0.58 \pm 0.01$ \\
\rowcolor[HTML]{EFEFEF}
\LTN + \TaskOnly & $1.00 \pm 0.00$ & $1.00 \pm 0.00$ & $0.31 \pm 0.01$ & $0.77 \pm 0.03$ & $2.06 \pm 0.14$ & $0.89 \pm 0.01$ \\
\LTN + \TaskPlusAbduction & $1.00 \pm 0.00$ & $196.29 \pm 13.29$ & $0.90 \pm 0.01$ & $1.00 \pm 0.00$ & $2.06 \pm 0.14$ & $0.89 \pm 0.01$ \\
\rowcolor[HTML]{EFEFEF}
\LTN + \ConceptsOnly & $0.85 \pm 0.03$ & $19.65 \pm 2.28$ & $0.93 \pm 0.01$ & $1.00 \pm 0.00$ & $1.00 \pm 0.00$ & $0.58 \pm 0.01$ \\
\LTN + \ConceptsPlusDeduction & $1.00 \pm 0.00$ & $19.65 \pm 2.28$ & $0.93 \pm 0.01$ & $1.00 \pm 0.00$ & $3.12 \pm 0.14$ & $0.75 \pm 0.01$ \\
\rowcolor[HTML]{EFEFEF}
\LTN + \bastani & $0.93 \pm 0.02$ & $19.65 \pm 2.28$ & $0.93 \pm 0.01$ & $1.00 \pm 0.00$ & $2.06 \pm 0.14$ & $0.89 \pm 0.01$ \\
\LTN + \method & $1.00 \pm 0.00$ & $18.22 \pm 1.79$ & $0.88 \pm 0.01$ & $1.00 \pm 0.00$ & $2.06 \pm 0.14$ & $0.89 \pm 0.01$ \\
\bottomrule
\end{tabular}

    \caption{Results for \LTN (with concept supervision) on \CIFAR.}
    \label{tab:ltn-cifar}
\end{table*}

\begin{table*}[!h]
    \centering
    \scriptsize
\begin{tabular}{lcccccc}
\toprule
& \multicolumn{3}{c}{{\sc Concepts}} & \multicolumn{3}{c}{{\sc Labels}} \\
\cmidrule(lr){2-4} \cmidrule(lr){5-7} {\sc Method} & {\sc Const.} & {\sc Size} & {\sc Cov.} & {\sc Const.} & {\sc Size} & {\sc Cov.} \\
\midrule
\DPL & $1.00 \pm 0.00$ & $1.00 \pm 0.00$ & $0.41 \pm 0.03$ & $1.00 \pm 0.00$ & $1.00 \pm 0.00$ & $0.78 \pm 0.00$ \\
\rowcolor[HTML]{EFEFEF}
\DPL + \TaskOnly & $0.92 \pm 0.01$ & $1.00 \pm 0.00$ & $0.41 \pm 0.03$ & $1.00 \pm 0.00$ & $1.24 \pm 0.01$ & $0.91 \pm 0.01$ \\
\DPL + \TaskPlusAbduction & $1.00 \pm 0.00$ & $116.77 \pm 0.97$ & $0.91 \pm 0.01$ & $1.00 \pm 0.00$ & $1.24 \pm 0.01$ & $0.91 \pm 0.01$ \\
\rowcolor[HTML]{EFEFEF}
\DPL + \ConceptsOnly & $0.97 \pm 0.03$ & $8.11 \pm 2.22$ & $0.93 \pm 0.01$ & $0.98 \pm 0.01$ & $1.00 \pm 0.00$ & $0.78 \pm 0.00$ \\
\DPL + \ConceptsPlusDeduction & $1.00 \pm 0.00$ & $8.11 \pm 2.22$ & $0.93 \pm 0.01$ & $1.00 \pm 0.00$ & $2.06 \pm 0.26$ & $0.76 \pm 0.00$ \\
\rowcolor[HTML]{EFEFEF}
\DPL + \bastani & $0.92 \pm 0.03$ & $8.11 \pm 2.22$ & $0.93 \pm 0.01$ & $1.00 \pm 0.00$ & $1.23 \pm 0.02$ & $0.90 \pm 0.01$ \\
\DPL + \method & $1.00 \pm 0.00$ & $7.31 \pm 1.77$ & $0.86 \pm 0.01$ & $1.00 \pm 0.00$ & $1.23 \pm 0.02$ & $0.90 \pm 0.01$ \\
\bottomrule
\end{tabular}

    \caption{Results for \DPL (without concept supervision) on \RIVAL.}
    \label{tab:dpl-rival-no-cs}
\end{table*}

\begin{table*}[!h]
    \centering
    \scriptsize
\begin{tabular}{lcccccc}
\toprule
& \multicolumn{3}{c}{{\sc Concepts}} & \multicolumn{3}{c}{{\sc Labels}} \\
\cmidrule(lr){2-4} \cmidrule(lr){5-7} {\sc Method} & {\sc Const.} & {\sc Size} & {\sc Cov.} & {\sc Const.} & {\sc Size} & {\sc Cov.} \\
\midrule
\DPL & $1.00 \pm 0.00$ & $1.00 \pm 0.00$ & $0.43 \pm 0.02$ & $1.00 \pm 0.00$ & $1.00 \pm 0.00$ & $0.80 \pm 0.00$ \\
\rowcolor[HTML]{EFEFEF}
\DPL + \TaskOnly & $0.91 \pm 0.01$ & $1.00 \pm 0.00$ & $0.43 \pm 0.02$ & $1.00 \pm 0.00$ & $1.21 \pm 0.02$ & $0.90 \pm 0.01$ \\
\DPL + \TaskPlusAbduction & $1.00 \pm 0.00$ & $114.20 \pm 1.38$ & $0.90 \pm 0.01$ & $1.00 \pm 0.00$ & $1.21 \pm 0.02$ & $0.90 \pm 0.01$ \\
\rowcolor[HTML]{EFEFEF}
\DPL + \ConceptsOnly & $0.99 \pm 0.00$ & $6.33 \pm 0.65$ & $0.92 \pm 0.01$ & $0.97 \pm 0.01$ & $1.00 \pm 0.00$ & $0.80 \pm 0.00$ \\
\DPL + \ConceptsPlusDeduction & $1.00 \pm 0.00$ & $6.33 \pm 0.65$ & $0.92 \pm 0.01$ & $1.00 \pm 0.00$ & $1.71 \pm 0.10$ & $0.74 \pm 0.01$ \\
\rowcolor[HTML]{EFEFEF}
\DPL + \bastani & $0.94 \pm 0.01$ & $6.33 \pm 0.65$ & $0.92 \pm 0.01$ & $1.00 \pm 0.00$ & $1.20 \pm 0.02$ & $0.89 \pm 0.01$ \\
\DPL + \method & $1.00 \pm 0.00$ & $5.85 \pm 0.62$ & $0.84 \pm 0.01$ & $1.00 \pm 0.00$ & $1.20 \pm 0.02$ & $0.89 \pm 0.01$ \\
\bottomrule
\end{tabular}

    \caption{Results for \DPL (with concept supervision) on \RIVAL.}
    \label{tab:dpl-rival}
\end{table*}

\begin{table*}[!h]
    \centering
    \scriptsize
\begin{tabular}{lcccccc}
\toprule
& \multicolumn{3}{c}{{\sc Concepts}} & \multicolumn{3}{c}{{\sc Labels}} \\
\cmidrule(lr){2-4} \cmidrule(lr){5-7} {\sc Method} & {\sc Const.} & {\sc Size} & {\sc Cov.} & {\sc Const.} & {\sc Size} & {\sc Cov.} \\
\midrule
\LTN & $0.99 \pm 0.02$ & $1.00 \pm 0.00$ & $0.14 \pm 0.11$ & $0.99 \pm 0.02$ & $1.00 \pm 0.00$ & $0.56 \pm 0.12$ \\
\rowcolor[HTML]{EFEFEF}
\LTN + \TaskOnly & $1.00 \pm 0.00$ & $1.00 \pm 0.00$ & $0.14 \pm 0.11$ & $0.90 \pm 0.04$ & $2.15 \pm 0.38$ & $0.90 \pm 0.02$ \\
\LTN + \TaskPlusAbduction & $1.00 \pm 0.00$ & $201.12 \pm 33.62$ & $0.90 \pm 0.02$ & $1.00 \pm 0.00$ & $2.15 \pm 0.38$ & $0.90 \pm 0.02$ \\
\rowcolor[HTML]{EFEFEF}
\LTN + \ConceptsOnly & $0.93 \pm 0.03$ & $12.00 \pm 1.78$ & $0.92 \pm 0.01$ & $0.98 \pm 0.01$ & $1.00 \pm 0.00$ & $0.56 \pm 0.12$ \\
\LTN + \ConceptsPlusDeduction & $1.00 \pm 0.00$ & $12.00 \pm 1.78$ & $0.92 \pm 0.01$ & $1.00 \pm 0.00$ & $1.97 \pm 0.08$ & $0.74 \pm 0.01$ \\
\rowcolor[HTML]{EFEFEF}
\LTN + \bastani & $0.99 \pm 0.00$ & $12.00 \pm 1.78$ & $0.92 \pm 0.01$ & $1.00 \pm 0.00$ & $2.05 \pm 0.30$ & $0.89 \pm 0.02$ \\
\LTN + \method & $1.00 \pm 0.00$ & $11.65 \pm 1.74$ & $0.85 \pm 0.01$ & $1.00 \pm 0.00$ & $2.05 \pm 0.30$ & $0.89 \pm 0.02$ \\
\bottomrule
\end{tabular}

    \caption{Results for \LTN (without concept supervision) on \RIVAL.}
    \label{tab:ltn-rival-no-cs}
\end{table*}

\begin{table*}[!h]
    \centering
    \scriptsize
\begin{tabular}{lcccccc}
\toprule
& \multicolumn{3}{c}{{\sc Concepts}} & \multicolumn{3}{c}{{\sc Labels}} \\
\cmidrule(lr){2-4} \cmidrule(lr){5-7} {\sc Method} & {\sc Const.} & {\sc Size} & {\sc Cov.} & {\sc Const.} & {\sc Size} & {\sc Cov.} \\
\midrule
\LTN & $0.99 \pm 0.00$ & $1.00 \pm 0.00$ & $0.41 \pm 0.02$ & $0.99 \pm 0.00$ & $1.00 \pm 0.00$ & $0.72 \pm 0.01$ \\
\rowcolor[HTML]{EFEFEF}
\LTN + \TaskOnly & $0.97 \pm 0.01$ & $1.00 \pm 0.00$ & $0.41 \pm 0.02$ & $0.99 \pm 0.01$ & $1.35 \pm 0.03$ & $0.90 \pm 0.01$ \\
\LTN + \TaskPlusAbduction & $1.00 \pm 0.00$ & $128.86 \pm 2.72$ & $0.90 \pm 0.01$ & $1.00 \pm 0.00$ & $1.35 \pm 0.03$ & $0.90 \pm 0.01$ \\
\rowcolor[HTML]{EFEFEF}
\LTN + \ConceptsOnly & $0.97 \pm 0.01$ & $9.41 \pm 0.72$ & $0.92 \pm 0.01$ & $0.98 \pm 0.01$ & $1.00 \pm 0.00$ & $0.72 \pm 0.01$ \\
\LTN + \ConceptsPlusDeduction & $1.00 \pm 0.00$ & $9.41 \pm 0.72$ & $0.92 \pm 0.01$ & $1.00 \pm 0.00$ & $1.84 \pm 0.12$ & $0.75 \pm 0.01$ \\
\rowcolor[HTML]{EFEFEF}
\LTN + \bastani & $0.96 \pm 0.01$ & $9.41 \pm 0.72$ & $0.92 \pm 0.01$ & $1.00 \pm 0.00$ & $1.33 \pm 0.02$ & $0.89 \pm 0.01$ \\
\LTN + \method & $1.00 \pm 0.00$ & $8.85 \pm 0.68$ & $0.85 \pm 0.01$ & $1.00 \pm 0.00$ & $1.33 \pm 0.02$ & $0.89 \pm 0.01$ \\
\bottomrule
\end{tabular}

    \caption{Results for \LTN (with concept supervision) on \RIVAL.}
    \label{tab:ltn-rival}
\end{table*}

\begin{table*}[!h]
    \centering
    \scriptsize
\begin{tabular}{lcccccc}
\toprule
& \multicolumn{3}{c}{{\sc Concepts}} & \multicolumn{3}{c}{{\sc Labels}} \\
\cmidrule(lr){2-4} \cmidrule(lr){5-7} {\sc Method} & {\sc Const.} & {\sc Size} & {\sc Cov.} & {\sc Const.} & {\sc Size} & {\sc Cov.} \\
\midrule
\DPL & $1.00 \pm 0.00$ & $1.00 \pm 0.00$ & $0.13 \pm 0.08$ & $1.00 \pm 0.00$ & $1.00 \pm 0.00$ & $0.72 \pm 0.01$ \\
\rowcolor[HTML]{EFEFEF}
\DPL + \TaskOnly & $1.00 \pm 0.00$ & $1.00 \pm 0.00$ & $0.13 \pm 0.08$ & $0.65 \pm 0.07$ & $1.55 \pm 0.18$ & $0.90 \pm 0.02$ \\
\DPL + \TaskPlusAbduction & $1.00 \pm 0.00$ & $12.13 \pm 1.37$ & $0.90 \pm 0.02$ & $1.00 \pm 0.00$ & $1.55 \pm 0.18$ & $0.90 \pm 0.02$ \\
\rowcolor[HTML]{EFEFEF}
\DPL + \ConceptsOnly & $0.08 \pm 0.00$ & $100.00 \pm 0.00$ & $1.00 \pm 0.00$ & $1.00 \pm 0.00$ & $1.00 \pm 0.00$ & $0.72 \pm 0.01$ \\
\DPL + \ConceptsPlusDeduction & $1.00 \pm 0.00$ & $100.00 \pm 0.00$ & $1.00 \pm 0.00$ & $1.00 \pm 0.00$ & $19.00 \pm 0.00$ & $1.00 \pm 0.00$ \\
\rowcolor[HTML]{EFEFEF}
\DPL + \bastani & $0.12 \pm 0.01$ & $100.00 \pm 0.00$ & $1.00 \pm 0.00$ & $1.00 \pm 0.00$ & $1.55 \pm 0.18$ & $0.90 \pm 0.02$ \\
\DPL + \method & $1.00 \pm 0.00$ & $12.13 \pm 1.37$ & $0.90 \pm 0.02$ & $1.00 \pm 0.00$ & $1.55 \pm 0.18$ & $0.90 \pm 0.02$ \\
\bottomrule
\end{tabular}

    \caption{Results for \DPL (without concept supervision) on \MNISTEO.}
    \label{tab:dpl-mnisteo-no-cs}
\end{table*}

\begin{table*}[!h]
    \centering
    \scriptsize
\begin{tabular}{lcccccc}
\toprule
& \multicolumn{3}{c}{{\sc Concepts}} & \multicolumn{3}{c}{{\sc Labels}} \\
\cmidrule(lr){2-4} \cmidrule(lr){5-7} {\sc Method} & {\sc Const.} & {\sc Size} & {\sc Cov.} & {\sc Const.} & {\sc Size} & {\sc Cov.} \\
\midrule
\LTN & $1.00 \pm 0.00$ & $1.00 \pm 0.00$ & $0.00 \pm 0.00$ & $1.00 \pm 0.00$ & $1.00 \pm 0.00$ & $0.94 \pm 0.01$ \\
\rowcolor[HTML]{EFEFEF}
\LTN + \TaskOnly & $1.00 \pm 0.00$ & $1.00 \pm 0.00$ & $0.00 \pm 0.00$ & $1.00 \pm 0.00$ & $1.00 \pm 0.00$ & $0.94 \pm 0.01$ \\
\LTN + \TaskPlusAbduction & $1.00 \pm 0.00$ & $8.05 \pm 0.05$ & $0.94 \pm 0.01$ & $1.00 \pm 0.00$ & $1.00 \pm 0.00$ & $0.94 \pm 0.01$ \\
\rowcolor[HTML]{EFEFEF}
\LTN + \ConceptsOnly & $0.08 \pm 0.00$ & $100.00 \pm 0.00$ & $1.00 \pm 0.00$ & $1.00 \pm 0.00$ & $1.00 \pm 0.00$ & $0.94 \pm 0.01$ \\
\LTN + \ConceptsPlusDeduction & $1.00 \pm 0.00$ & $100.00 \pm 0.00$ & $1.00 \pm 0.00$ & $1.00 \pm 0.00$ & $19.00 \pm 0.00$ & $1.00 \pm 0.00$ \\
\rowcolor[HTML]{EFEFEF}
\LTN + \bastani & $0.08 \pm 0.00$ & $100.00 \pm 0.00$ & $1.00 \pm 0.00$ & $1.00 \pm 0.00$ & $1.00 \pm 0.00$ & $0.94 \pm 0.01$ \\
\LTN + \method & $1.00 \pm 0.00$ & $8.05 \pm 0.05$ & $0.94 \pm 0.01$ & $1.00 \pm 0.00$ & $1.00 \pm 0.00$ & $0.94 \pm 0.01$ \\
\bottomrule
\end{tabular}

    \caption{Results for \LTN (without concept supervision) on \MNISTEO.}
    \label{tab:ltn-mnisteo-no-cs}
\end{table*}

\begin{table*}[!h]
    \centering
    \scriptsize
\begin{tabular}{lcccccc}
\toprule
& \multicolumn{3}{c}{{\sc Concepts}} & \multicolumn{3}{c}{{\sc Labels}} \\
\cmidrule(lr){2-4} \cmidrule(lr){5-7} {\sc Method} & {\sc Const.} & {\sc Size} & {\sc Cov.} & {\sc Const.} & {\sc Size} & {\sc Cov.} \\
\midrule
\DPL & $1.00 \pm 0.00$ & $1.00 \pm 0.00$ & $0.50 \pm 0.03$ & $1.00 \pm 0.00$ & $1.00 \pm 0.00$ & $0.99 \pm 0.01$ \\
\rowcolor[HTML]{EFEFEF}
\DPL + \TaskOnly & $0.89 \pm 0.05$ & $1.00 \pm 0.00$ & $0.50 \pm 0.03$ & $1.00 \pm 0.00$ & $0.89 \pm 0.05$ & $0.89 \pm 0.05$ \\
\DPL + \TaskPlusAbduction & $1.00 \pm 0.00$ & $2.20 \pm 0.14$ & $0.89 \pm 0.05$ & $1.00 \pm 0.00$ & $0.89 \pm 0.05$ & $0.89 \pm 0.05$ \\
\rowcolor[HTML]{EFEFEF}
\DPL + \ConceptsOnly & $0.10 \pm 0.00$ & $25.00 \pm 0.00$ & $1.00 \pm 0.00$ & $1.00 \pm 0.00$ & $1.00 \pm 0.00$ & $0.99 \pm 0.01$ \\
\DPL + \ConceptsPlusDeduction & $1.00 \pm 0.00$ & $25.00 \pm 0.00$ & $1.00 \pm 0.00$ & $1.00 \pm 0.00$ & $9.00 \pm 0.00$ & $1.00 \pm 0.00$ \\
\rowcolor[HTML]{EFEFEF}
\DPL + \bastani & $0.09 \pm 0.01$ & $25.00 \pm 0.00$ & $1.00 \pm 0.00$ & $1.00 \pm 0.00$ & $0.89 \pm 0.05$ & $0.89 \pm 0.05$ \\
\DPL + \method & $1.00 \pm 0.00$ & $2.20 \pm 0.14$ & $0.89 \pm 0.05$ & $1.00 \pm 0.00$ & $0.89 \pm 0.05$ & $0.89 \pm 0.05$ \\
\bottomrule
\end{tabular}

    \caption{Results for \DPL (without concept supervision) on \MNISTHALF.}
    \label{tab:dpl-mnisthalf-no-cs}
\end{table*}

\begin{table*}[!h]
    \centering
    \scriptsize
\begin{tabular}{lcccccc}
\toprule
& \multicolumn{3}{c}{{\sc Concepts}} & \multicolumn{3}{c}{{\sc Labels}} \\
\cmidrule(lr){2-4} \cmidrule(lr){5-7} {\sc Method} & {\sc Const.} & {\sc Size} & {\sc Cov.} & {\sc Const.} & {\sc Size} & {\sc Cov.} \\
\midrule
\LTN & $1.00 \pm 0.00$ & $1.00 \pm 0.00$ & $0.50 \pm 0.03$ & $1.00 \pm 0.00$ & $1.00 \pm 0.00$ & $0.83 \pm 0.12$ \\
\rowcolor[HTML]{EFEFEF}
\LTN + \TaskOnly & $1.00 \pm 0.00$ & $1.00 \pm 0.00$ & $0.50 \pm 0.03$ & $0.47 \pm 0.46$ & $5.80 \pm 4.13$ & $0.99 \pm 0.02$ \\
\LTN + \TaskPlusAbduction & $1.00 \pm 0.00$ & $16.00 \pm 11.62$ & $0.99 \pm 0.02$ & $1.00 \pm 0.00$ & $5.80 \pm 4.13$ & $0.99 \pm 0.02$ \\
\rowcolor[HTML]{EFEFEF}
\LTN + \ConceptsOnly & $0.10 \pm 0.01$ & $19.48 \pm 1.70$ & $0.87 \pm 0.04$ & $1.00 \pm 0.00$ & $1.00 \pm 0.00$ & $0.83 \pm 0.12$ \\
\LTN + \ConceptsPlusDeduction & $1.00 \pm 0.00$ & $19.48 \pm 1.70$ & $0.87 \pm 0.04$ & $1.00 \pm 0.00$ & $8.62 \pm 0.24$ & $1.00 \pm 0.00$ \\
\rowcolor[HTML]{EFEFEF}
\LTN + \bastani & $0.64 \pm 0.46$ & $19.48 \pm 1.70$ & $0.87 \pm 0.04$ & $1.00 \pm 0.00$ & $5.60 \pm 3.97$ & $0.99 \pm 0.02$ \\
\LTN + \method & $1.00 \pm 0.00$ & $12.77 \pm 9.35$ & $0.86 \pm 0.04$ & $1.00 \pm 0.00$ & $5.60 \pm 3.97$ & $0.99 \pm 0.02$ \\
\bottomrule
\end{tabular}

    \caption{Results for \LTN (without concept supervision) on \MNISTHALF.}
    \label{tab:ltn-mnisthalf-no-cs}
\end{table*}

\begin{table*}[!h]
    \centering
    \scriptsize
\begin{tabular}{lcccccc}
\toprule
& \multicolumn{3}{c}{{\sc Concepts}} & \multicolumn{3}{c}{{\sc Labels}} \\
\cmidrule(lr){2-4} \cmidrule(lr){5-7} {\sc Method} & {\sc Const.} & {\sc Size} & {\sc Cov.} & {\sc Const.} & {\sc Size} & {\sc Cov.} \\
\midrule
\DPL & $1.00 \pm 0.00$ & $1.00 \pm 0.00$ & $0.03 \pm 0.02$ & $1.00 \pm 0.00$ & $1.00 \pm 0.00$ & $0.99 \pm 0.00$ \\
\rowcolor[HTML]{EFEFEF}
\DPL + \TaskOnly & $0.90 \pm 0.01$ & $1.00 \pm 0.00$ & $0.03 \pm 0.02$ & $1.00 \pm 0.00$ & $0.90 \pm 0.01$ & $0.90 \pm 0.01$ \\
\DPL + \TaskPlusAbduction & $1.00 \pm 0.00$ & $45.19 \pm 0.35$ & $0.90 \pm 0.01$ & $1.00 \pm 0.00$ & $0.90 \pm 0.01$ & $0.90 \pm 0.01$ \\
\rowcolor[HTML]{EFEFEF}
\DPL + \ConceptsOnly & $0.81 \pm 0.22$ & $46.81 \pm 36.72$ & $0.93 \pm 0.05$ & $1.00 \pm 0.00$ & $1.00 \pm 0.00$ & $0.99 \pm 0.00$ \\
\DPL + \ConceptsPlusDeduction & $1.00 \pm 0.00$ & $46.81 \pm 36.72$ & $0.93 \pm 0.05$ & $1.00 \pm 0.00$ & $1.34 \pm 0.46$ & $1.00 \pm 0.00$ \\
\rowcolor[HTML]{EFEFEF}
\DPL + \bastani & $0.73 \pm 0.19$ & $46.81 \pm 36.72$ & $0.93 \pm 0.05$ & $1.00 \pm 0.00$ & $0.90 \pm 0.01$ & $0.90 \pm 0.01$ \\
\DPL + \method & $1.00 \pm 0.00$ & $27.92 \pm 12.05$ & $0.84 \pm 0.05$ & $1.00 \pm 0.00$ & $0.90 \pm 0.01$ & $0.90 \pm 0.01$ \\
\bottomrule
\end{tabular}

    \caption{Results for \DPL (without concept supervision) on \MNISTSumXor.}
    \label{tab:dpl-mnistsump-no-cs}
\end{table*}

\begin{table*}[!h]
    \centering
    \scriptsize
\begin{tabular}{lcccccc}
\toprule
& \multicolumn{3}{c}{{\sc Concepts}} & \multicolumn{3}{c}{{\sc Labels}} \\
\cmidrule(lr){2-4} \cmidrule(lr){5-7} {\sc Method} & {\sc Const.} & {\sc Size} & {\sc Cov.} & {\sc Const.} & {\sc Size} & {\sc Cov.} \\
\midrule
\LTN & $1.00 \pm 0.00$ & $1.00 \pm 0.00$ & $0.03 \pm 0.02$ & $1.00 \pm 0.00$ & $1.00 \pm 0.00$ & $0.98 \pm 0.00$ \\
\rowcolor[HTML]{EFEFEF}
\LTN + \TaskOnly & $1.00 \pm 0.00$ & $1.00 \pm 0.00$ & $0.03 \pm 0.02$ & $1.00 \pm 0.00$ & $1.00 \pm 0.00$ & $0.98 \pm 0.00$ \\
\LTN + \TaskPlusAbduction & $1.00 \pm 0.00$ & $50.00 \pm 0.00$ & $0.98 \pm 0.00$ & $1.00 \pm 0.00$ & $1.00 \pm 0.00$ & $0.98 \pm 0.00$ \\
\rowcolor[HTML]{EFEFEF}
\LTN + \ConceptsOnly & $0.50 \pm 0.00$ & $100.00 \pm 0.00$ & $1.00 \pm 0.00$ & $1.00 \pm 0.00$ & $1.00 \pm 0.00$ & $0.98 \pm 0.00$ \\
\LTN + \ConceptsPlusDeduction & $1.00 \pm 0.00$ & $100.00 \pm 0.00$ & $1.00 \pm 0.00$ & $1.00 \pm 0.00$ & $2.00 \pm 0.00$ & $1.00 \pm 0.00$ \\
\rowcolor[HTML]{EFEFEF}
\LTN + \bastani & $0.50 \pm 0.00$ & $100.00 \pm 0.00$ & $1.00 \pm 0.00$ & $1.00 \pm 0.00$ & $1.00 \pm 0.00$ & $0.98 \pm 0.00$ \\
\LTN + \method & $1.00 \pm 0.00$ & $50.00 \pm 0.00$ & $0.98 \pm 0.00$ & $1.00 \pm 0.00$ & $1.00 \pm 0.00$ & $0.98 \pm 0.00$ \\
\bottomrule
\end{tabular}

    \caption{Results for \LTN (without concept supervision) on \MNISTSumXor.}
    \label{tab:ltn-mnistsump-no-cs}
\end{table*}

\begin{table*}[!h]
    \centering
    \centering
\scriptsize
\setlength{\tabcolsep}{4pt}
\scalebox{0.9}{
\begin{tabular}{lcccccccc}
\toprule
 & \multicolumn{4}{c}{\sc Concepts} & \multicolumn{4}{c}{\sc Labels} \\
\cmidrule(lr){2-5} \cmidrule(lr){6-9}
{\sc Method} & {\sc Const.} & {\sc Size} & {\sc $1 - \mathbb{E}[\tilde{\beta}]$} & {\sc Cov.} & {\sc Const.} & {\sc Size} & {\sc $1 - \mathbb{E}[\tilde{\alpha}]$} & {\sc Cov.} \\
\midrule
Concept-fixed & $5$ & $3.37 \pm 0.30$ & $0.90 \pm 0.00$ & $0.86 \pm 0.02$ & --- & $1.43 \pm 0.08$ & --- & --- \\
\rowcolor[HTML]{EFEFEF}
Label-fixed & --- & $2.42 \pm 0.45$ & --- & --- & $2$ & $1.72 \pm 0.12$ & $0.82 \pm 0.04$ & $0.90 \pm 0.02$ \\
Both-fixed & $5$ & $3.57 \pm 0.73$ & $0.77 \pm 0.13$ & $0.89 \pm 0.03$ & $2$ & $1.80 \pm 0.06$ & $0.82 \pm 0.04$ & $0.91 \pm 0.01$ \\
\bottomrule
\end{tabular}}

    \caption{E-prediction theoretical coverage target ({\sc $1 - \mathbb{E}[\alpha]$, $1 - \mathbb{E}[\beta]$}, respectively) results on \CHX, averaged over $10$ seeds ($\text{mean}\pm\text{std}$). {\sc Const.} is the imposed set-size constraint (--\,if unconstrained); {\sc Size} is the average prediction-set size; {\sc Cov.} is the empirical coverage.}
    \label{tab:econformal}
\end{table*}

\begin{table*}[!h]
    \centering
    \scriptsize
\scalebox{0.8}{
\begin{tabular}{p{1em}lcccccc}
\toprule
& & \multicolumn{3}{c}{{\sc Concepts}} & \multicolumn{3}{c}{{\sc Labels}} \\
\cmidrule(lr){3-5} \cmidrule(lr){6-8} {\sc} & {\sc Method} & {\sc Const.} & {\sc Size} & {\sc Cov.} & {\sc Const.} & {\sc Size} & {\sc Cov.} \\
\midrule
\multirow[c]{3}[0]{1em}{\rotatebox{90}{\texttt{MN-Add}}} & \rule{0pt}{2ex} \DPL$^{\square}$ & $\mathbf{1.00 \pm 0.00}$ & $1.00 \pm 0.00$ & $0.98 \pm 0.00$ & $\mathbf{1.00 \pm 0.00}$ & $1.00 \pm 0.00$ & $0.98 \pm 0.00$ \\
 & \rule{0pt}{2ex} \DPL + \bastani$^{\square}$ & $1.00 \pm 0.00$ & $1.00 \pm 0.00$ & $\mathbf{0.90 \pm 0.01}$ & $\mathbf{1.00 \pm 0.00}$ & $\mathbf{0.90 \pm 0.01}$ & $\mathbf{0.90 \pm 0.01}$ \\
 & \rule{0pt}{2ex} \DPL + \method$^{\square}$ & $\mathbf{1.00 \pm 0.00}$ & $\mathbf{0.99 \pm 0.00}$ & $\mathbf{0.90 \pm 0.01}$ & $\mathbf{1.00 \pm 0.00}$ & $\mathbf{0.90 \pm 0.01}$ & $\mathbf{0.90 \pm 0.01}$ \\
\cmidrule(lr){2-8}
\multirow[c]{3}[0]{1em}{\rotatebox{90}{\texttt{MN-H}}} & \rule{0pt}{2ex} \DPL$^{\square}$ & $\mathbf{1.00 \pm 0.00}$ & $1.00 \pm 0.00$ & $0.50 \pm 0.03$ & $\mathbf{1.00 \pm 0.00}$ & $1.00 \pm 0.00$ & $0.99 \pm 0.01$ \\
 & \rule{0pt}{2ex} \DPL + \bastani$^{\square}$ & $0.09 \pm 0.01$ & $25.00 \pm 0.00$ & $\mathbf{1.00 \pm 0.00}$ & $\mathbf{1.00 \pm 0.00}$ & $\mathbf{0.89 \pm 0.05}$ & $\mathbf{0.89 \pm 0.05}$ \\
 & \rule{0pt}{2ex} \DPL + \method$^{\square}$ & $\mathbf{1.00 \pm 0.00}$ & $\mathbf{2.20 \pm 0.14}$ & $0.89 \pm 0.05$ & $\mathbf{1.00 \pm 0.00}$ & $\mathbf{0.89 \pm 0.05}$ & $\mathbf{0.89 \pm 0.05}$ \\
\bottomrule
\end{tabular}
}

    \caption{Full results over the $8$ datasets for \DPL. $\square$ indicates no concept supervision.}
    \label{tab:global-dpl}
\end{table*}

\begin{table*}[!h]
    \centering
    \scriptsize
\scalebox{0.8}{
\begin{tabular}{p{1em}lcccccc}
\toprule
& & \multicolumn{3}{c}{{\sc Concepts}} & \multicolumn{3}{c}{{\sc Labels}} \\
\cmidrule(lr){3-5} \cmidrule(lr){6-8} {\sc} & {\sc Method} & {\sc Const.} & {\sc Size} & {\sc Cov.} & {\sc Const.} & {\sc Size} & {\sc Cov.} \\
\midrule
\multirow[c]{3}[0]{1em}{\rotatebox{90}{\texttt{MN-Add}}} & \rule{0pt}{2ex} \LTN$^{\square}$ & $\mathbf{1.00 \pm 0.00}$ & $1.00 \pm 0.00$ & $0.97 \pm 0.00$ & $\mathbf{1.00 \pm 0.00}$ & $1.00 \pm 0.00$ & $0.97 \pm 0.00$ \\
 & \rule{0pt}{2ex} \LTN + \bastani$^{\square}$ & $\mathbf{1.00 \pm 0.00}$ & $1.00 \pm 0.00$ & $\mathbf{0.90 \pm 0.01}$ & $\mathbf{1.00 \pm 0.00}$ & $\mathbf{0.90 \pm 0.01}$ & $\mathbf{0.90 \pm 0.01}$ \\
 & \rule{0pt}{2ex} \LTN + \method$^{\square}$ & $\mathbf{1.00 \pm 0.00}$ & $\mathbf{0.90 \pm 0.01}$ & $\mathbf{0.90 \pm 0.01}$ & $\mathbf{1.00 \pm 0.00}$ & $\mathbf{0.90 \pm 0.01}$ & $\mathbf{0.90 \pm 0.01}$ \\
\cmidrule(lr){2-8}
\multirow[c]{3}[0]{1em}{\rotatebox{90}{\texttt{MN-H}}} & \rule{0pt}{2ex} \LTN$^{\square}$ & $\mathbf{1.00 \pm 0.00}$ & $1.00 \pm 0.00$ & $0.50 \pm 0.03$ & $\mathbf{1.00 \pm 0.00}$ & $1.00 \pm 0.00$ & $0.83 \pm 0.12$ \\
 & \rule{0pt}{2ex} \LTN + \bastani$^{\square}$ & $0.64 \pm 0.46$ & $19.48 \pm 1.70$ & $\mathbf{0.87 \pm 0.04}$ & $\mathbf{1.00 \pm 0.00}$ & $\mathbf{5.60 \pm 3.97}$ & $\mathbf{0.99 \pm 0.02}$ \\
 & \rule{0pt}{2ex} \LTN + \method$^{\square}$ & $\mathbf{1.00 \pm 0.00}$ & $\mathbf{12.77 \pm 9.35}$ & $0.86 \pm 0.04$ & $\mathbf{1.00 \pm 0.00}$ & $\mathbf{5.60 \pm 3.97}$ & $\mathbf{0.99 \pm 0.02}$ \\
\bottomrule
\end{tabular}
}

    \caption{Full results over the $8$ datasets for \LTN. $\square$ indicates no concept supervision.}
    \label{tab:global-ltn}
\end{table*}

\begin{figure*}[!h]
    \centering
    \begin{subfigure}[b]{0.48\linewidth}
        \centering
        \includegraphics[width=\linewidth]{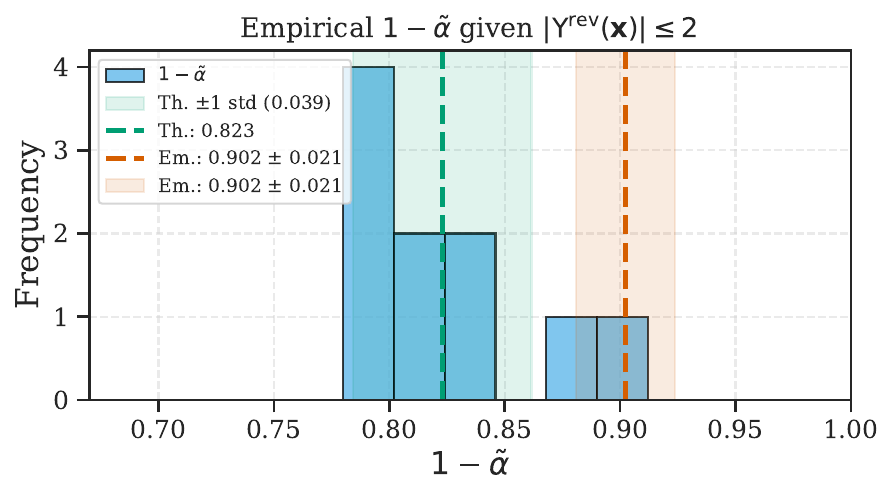}
        \caption{Label-fixed: $|\Upsilon^{\sf rev}(\mathbf{x})| \leq 2$}
        \label{fig:app-label-fixed}
    \end{subfigure}
    \hfill
    \begin{subfigure}[b]{0.48\linewidth}
        \centering
        \includegraphics[width=\linewidth]{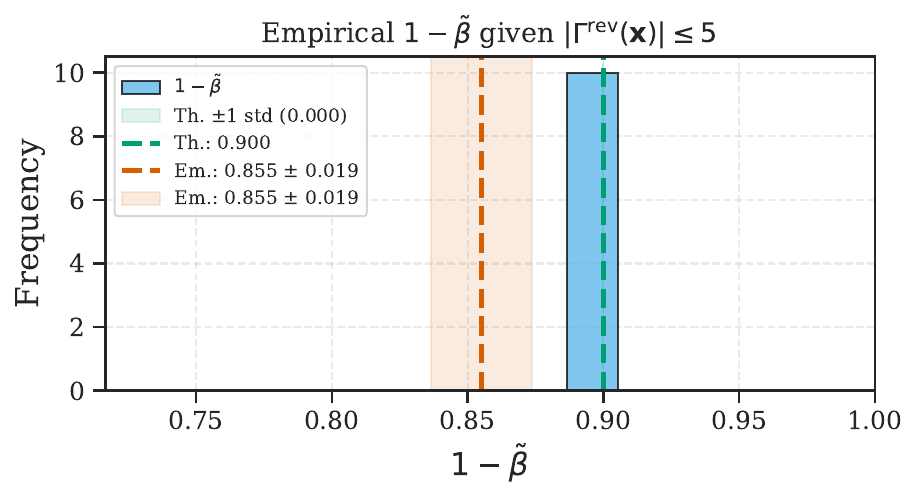}
        \caption{Concept-fixed: $|\Gamma^{\sf rev}(\mathbf{x})| \leq 5$}
        \label{fig:app-concept-fixed}
    \end{subfigure}
    \caption{E-conformal coverage under single-sided constraints on \CHX averaged over $10$ seeds. Empirical coverage closely tracks the
    theoretical guarantee in both regimes.}
    \label{fig:app-single-sided}
\end{figure*}

\begin{figure*}[!h]
    \label{fig:joint-failure}
    \centering
    \includegraphics[width=0.9\linewidth]{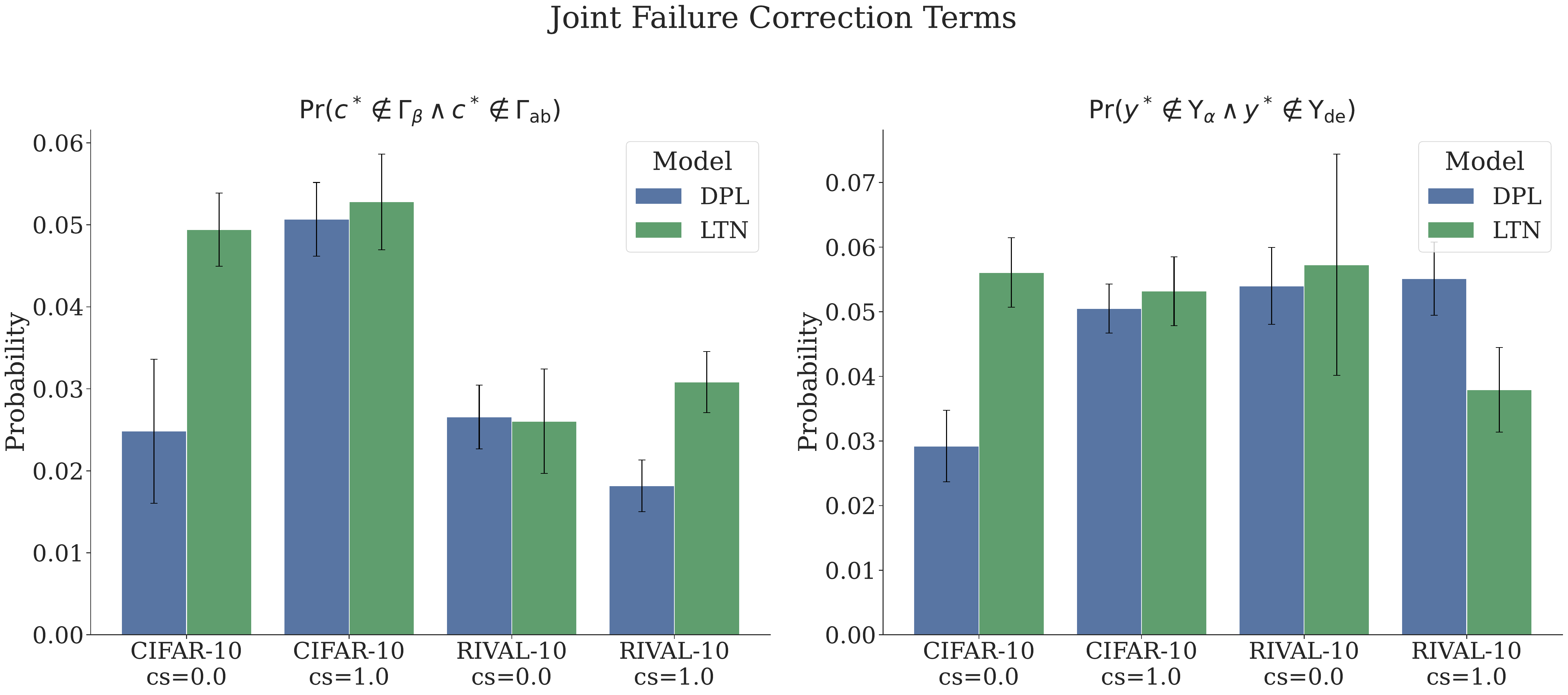}
    \caption{Joint failure probabilities on \CIFAR and \RIVAL. Even if low, this estimate contributes to tightening the lower bound in~\cref{prop:joint-coverage} and~\cref{prop:evalues-joint-coverage}.}
\end{figure*}

\begin{figure*}[!h]
    \centering
    \includegraphics[width=0.8\linewidth]{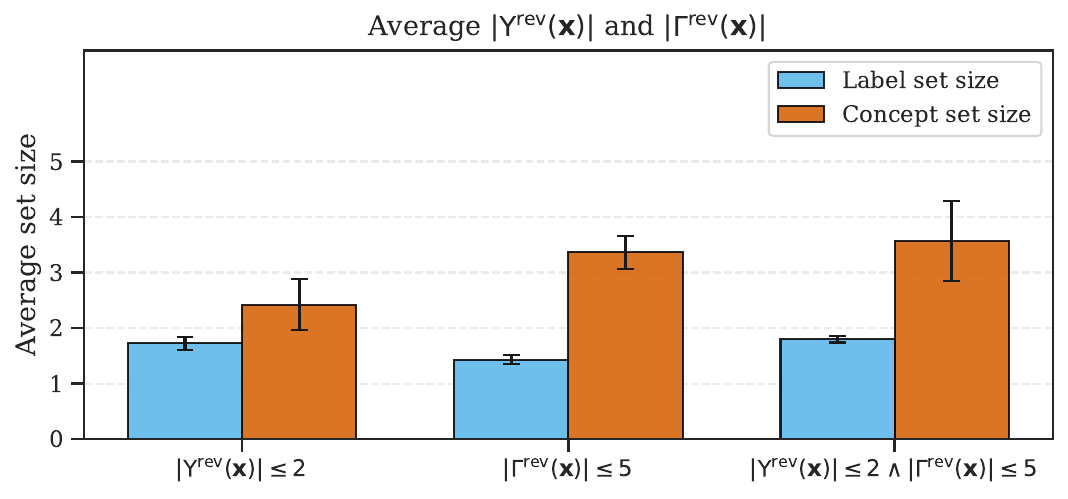}
    \caption{Average prediction-set sizes across constraint regimes; error bars show $\pm 1$ std across $10$ seeds.}
    \label{fig:app-set-sizes}
\end{figure*}

\begin{figure*}[!h]
    \centering
    \includegraphics[width=\linewidth]{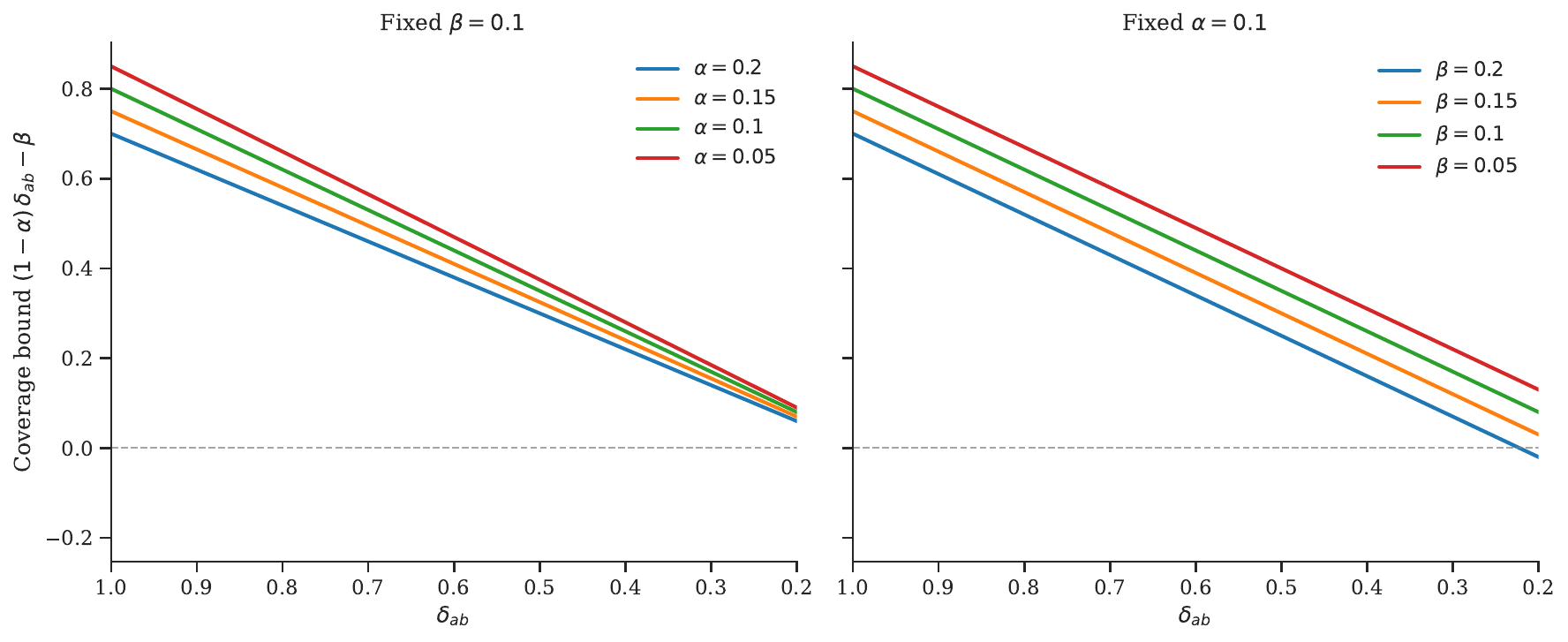}
    \caption{Synthetic experiment showing $(1-\alpha)\delta_{ab} - \beta$ for varying $\alpha$ (left) and $\beta$ (right) starting at nominal guarantee of $1 - \alpha = 1 - \beta = 0.9$. The bound decrease linearly depending on the quality of the inference layer $g$.}
    \label{fig:bound-analysis}
\end{figure*}

\clearpage

\subsection{Estimating $\delta_{\mathrm{ab}}$ and $\delta_{\mathrm{de}}$}
\label{sec:delta-estimation}

We estimate the soundness gaps empirically as
\begin{align}
\hat{\delta}_{\mathrm{de}}
  &= \frac{1}{|\mathcal{D}_{\mathrm{test}}|}
     \sum_{(\vx, \vc^*, \vy^*) \in \mathcal{D}_{\mathrm{test}}}
     \Ind{g^\dagger(\vc^*) = \vy^*}, \\
\hat{\delta}_{\mathrm{ab}}
  &= \frac{1}{|\mathcal{D}_{\mathrm{test}}|}
     \sum_{(\vx, \vc^*, \vy^*) \in \mathcal{D}_{\mathrm{test}}}
     \Ind{\vc^* \in g^{-\dagger}(\vy^*)}.
\end{align}
where $\mathcal{D}_{\mathrm{test}}$ is the test set. \RIVAL and \CIFAR are constructed such that $4$ of $10$ classes share concepts with another class, so $\delta_{\mathrm{de}} = 0.8$ is a structural property of the dataset by design.

\begin{table}[!h]
\scriptsize
\centering
\label{tab:delta-estimates}
\begin{tabular}{lcc}
\toprule
    \textsc{Dataset} & $\hat{\delta}_{\mathrm{ab}}$ & $\hat{\delta}_{\mathrm{de}}$ \\
\midrule
    \MNISTAdd  & $1.00 \pm 0.00$ & $1.00 \pm 0.00$ \\
    \MNISTHALF & $1.00 \pm 0.00$ & $1.00 \pm 0.00$ \\
    \MNISTSumXor & $1.00 \pm 0.00$ & $1.00 \pm 0.00$ \\
    \MNISTEO   & $1.00 \pm 0.00$ & $1.00 \pm 0.00$ \\
    \CHX     & $1.00 \pm 0.00$ & $1.00 \pm 0.00$ \\
    \CEBAB     & $1.00 \pm 0.00$ & $1.00 \pm 0.00$ \\
    \CIFAR   & $1.00 \pm 0.00$ & $0.79 \pm 0.01$ \\
    \RIVAL   & $1.00 \pm 0.00$ & $0.79 \pm 0.01$ \\
\bottomrule\\
\end{tabular}
\caption{Empirical estimates of $\delta_{\mathrm{ab}}$ and $\delta_{\mathrm{de}}$ on the test set, averaged over $10$ seeds (mean $\pm$ std).}

\end{table}

\subsection{Product vs. Average e-value Aggregation: when Assumptions do not hold}
\label{sec:product-vs-average}

\begin{figure}[h]
    \centering
    \begin{subfigure}[b]{0.48\linewidth}
        \centering
        \includegraphics[width=\linewidth]{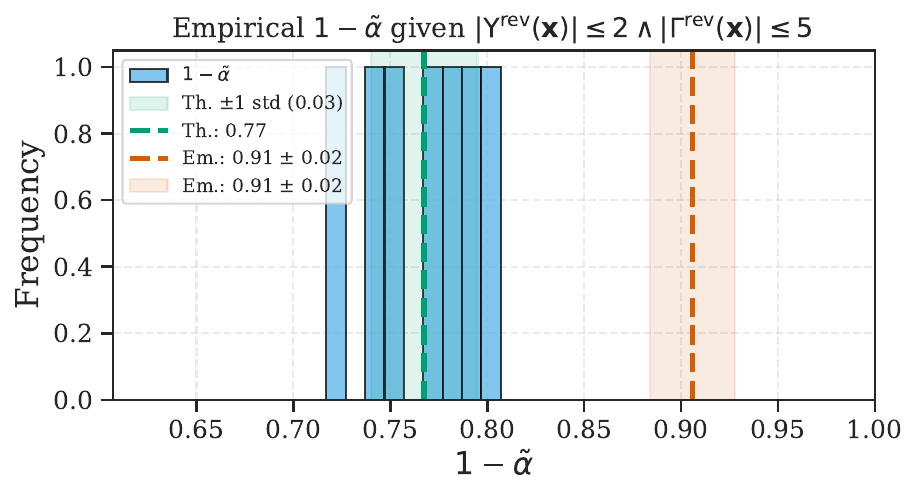}
        \caption{Label coverage}
    \end{subfigure}
    \hfill
    \begin{subfigure}[b]{0.48\linewidth}
        \centering
        \includegraphics[width=\linewidth]{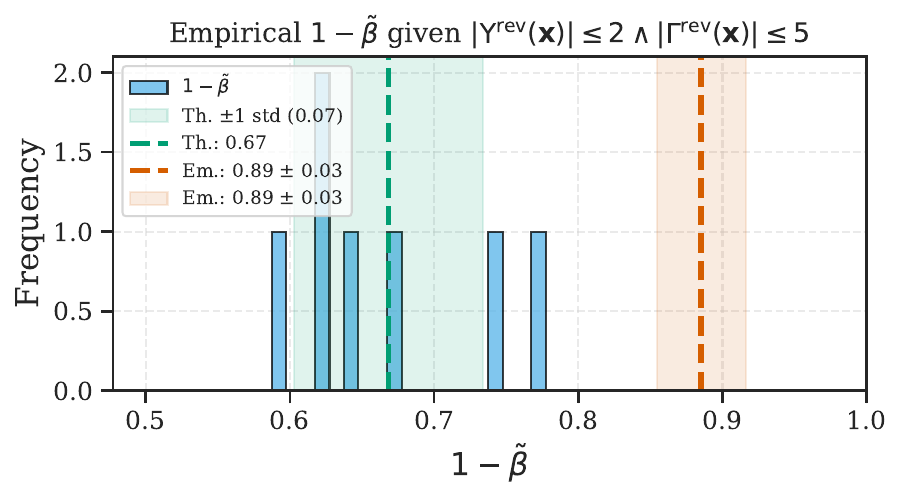}
        \caption{Concept coverage}
    \end{subfigure}
    \caption{\EMethod on \CHX with \DPL under fixed regime $|\Yrev| \leq 2 \land |\Crev| \leq 5$, bootstrapped over $100$ calibration resamples and averaged across $10$ seeds. Empirical coverage (\textcolor{orange}{orange}) exceeds the theoretical target (\textcolor{ForestGreen}{green}), while sizes remain \emph{strictly} below the imposed budgets. Results using \textbf{product} as e-value aggregator.}
    \label{fig:e-values-ablation}
\end{figure}

\begin{figure}[h!]
    \centering
    \includegraphics[width=0.8\linewidth]{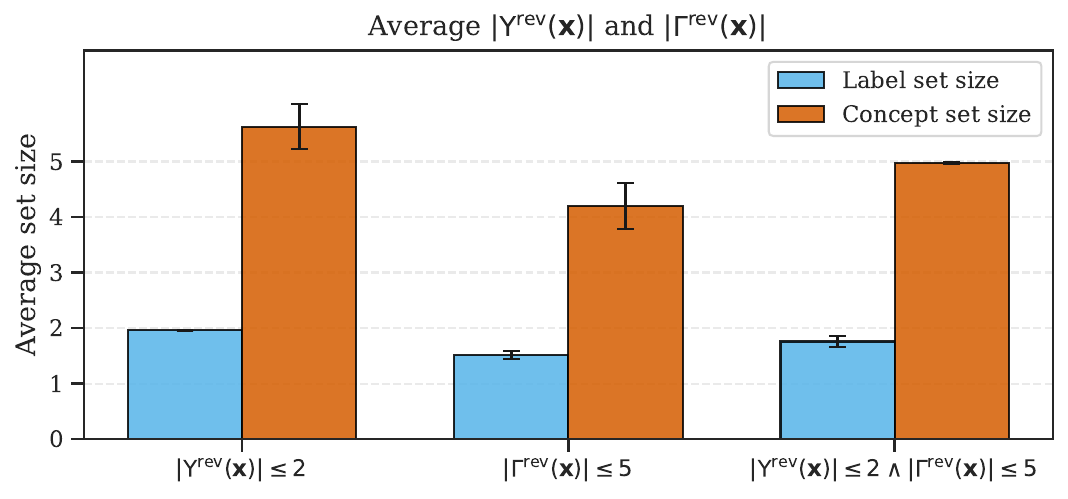}
    \caption{Average set-sizes using \EMethod and \textbf{product} as e-value aggregator.}
    \label{fig:set-sizes-ablation}
\end{figure}

As an additional ablation, we examine what happens when there is data leakage in the calibration process, as well as when we employ the \emph{product} as an e-value aggregator.
Specifically, we used the entire test set for evaluation. In this case, leakage occurs because the model is selected via early stopping on the validation set and is therefore influenced by it, which violates the \emph{exchangeability} assumption.

We evaluate \EMethod on \CHX, fixing the label and concept budgets to $2$ and $5$, respectively, and using \textbf{product} as the e-value aggregator. This approach is theoretically justified only under the assumption of \emph{conditional independence} of the per-concept e-values given $(\vx, \calD_{\textrm{cal}})$ -- an assumption that does not hold in practice.
For this analysis, we performed an extensive search over the space of $\alpha$ and $\beta$ -- the label and concept miscoverage rates, respectively -- and ran $100$ bootstrap iterations. Specifically, we analyze the following combinations: $(\{0.10, 0.11, \ldots, 0.30\}, \{0.10, 0.30, 0.31, \dots, 0.60\})$.
Surprisingly, despite these violations, \EMethod exhibit stable behavior on the test set. As shown in~\cref{fig:set-sizes-ablation}, the average set sizes are $1.74$ for labels and $4.97$ for concepts, both strictly below their respective budgets. The empirical coverages are $0.91$ and $0.89$, substantially exceeding the theoretical targets $1 - \mathbb{E}[\tilde{\alpha}] = 0.77$ and $1 - \mathbb{E}[\tilde{\beta}] = 0.67$ (see~\cref{fig:e-values-ablation}).

\subsection{The Full Spectrum of NeSy Conformal Prediction}
\label{sec:ablations}

In the main text we compared \method only against \bastani, the strongest competitor. Here we report the full ablation against the four natural conformalization strategies introduced in \cref{sec:other-ablations}: \TaskOnly, \TaskPlusAbduction, \ConceptsOnly, and \ConceptsPlusDeduction. The goal is threefold: (i) confirm that the two-sided joint revision of \method is genuinely necessary and that no simpler one-sided strategy achieves the same trade-off across \textbf{D1}--\textbf{D3}; (ii) quantify the failure modes of each baseline across the eight data sets; and (iii) ground the design choices of \method empirically. The full numerical results are \cref{tab:dpl-mnist-add-no-cs,tab:ltn-mnist-add-no-cs,tab:dpl-cebab-no-cs,tab:dpl-cebab,tab:ltn-cebab-no-cs,tab:ltn-cebab,tab:dpl-chx-no-cs,tab:dpl-chx,tab:ltn-chx-no-cs,tab:ltn-chx,tab:dpl-cifar-no-cs,tab:dpl-cifar,tab:ltn-cifar-no-cs,tab:ltn-cifar,tab:dpl-rival-no-cs,tab:dpl-rival,tab:ltn-rival-no-cs,tab:ltn-rival,tab:dpl-mnisteo-no-cs,tab:ltn-mnisteo-no-cs,tab:dpl-mnisthalf-no-cs,tab:ltn-mnisthalf-no-cs,tab:dpl-mnistsump-no-cs}; we summarise the qualitative results below. Throughout, the nominal coverage targets are $1 - \alpha = 1 - \beta = 0.9$.

\textbf{\TaskOnly offers conformalized labels but uncalibrated, inconsistent concepts}. \TaskOnly applies CP to the label distribution $p_\theta(\vY \mid \vx, \BK)$ directly. As predicted by \cref{eq:task-only-guarantee}, label coverage hits the nominal $1-\alpha=0.9$ target on every data set (\eg $0.90$ on \MNISTAdd, $0.89$ on \CHX with \DPL, $0.90$ on \RIVAL, \CEBAB, and \CIFAR). On the concept side, however, \TaskOnly returns the single $\argmax$ concept (size $1$), so concept coverage collapses to the base concept accuracy of the underlying NeSy-CBM: as low as $0.18$ on \CEBAB (\DPL, supervised), $0.00$ on \CEBAB unsupervised (\DPL and \LTN), $0.00$ and $0.13$ on \MNISTEO (\LTN and \DPL, respectively) and $0.03$ on \MNISTSumXor (\DPL and \LTN). Worse, label consistency drops well below $1$ whenever the conformal label set is non-trivial: \eg $0.20$ on \CHX (\LTN), $0.20$-$0.42$ on \CEBAB, and $0.62$ on \CHX (\DPL, supervised). \TaskOnly thus violates both \textbf{D1} (consistency) and \textbf{D2} (concept-side coverage).

\textbf{\ConceptsOnly achieves conformalized concepts but uncalibrated labels}. \ConceptsOnly applies CP at the concept level using the product construction of \cref{eq:product-concept-set} with Bonferroni's correction (\cref{cor:bonferroni-coverage}). Concept coverage matches the target on most data sets ($0.87$-$1.00$). However, \ConceptsOnly returns the single $\argmax$ label, so the label side inherits the base label accuracy. Label consistency drops: $0.08$-$0.40$ on \MNISTEO, \MNISTHALF, \CEBAB, and \CHX. Bonferroni's correction~(\cref{cor:bonferroni-coverage}) inflates the concept set as the number of concepts grows: $75.06$ on \CEBAB (\DPL, supervised), $57.60$ on \CIFAR (\DPL, unsupervised), $100.00$, the entire product space, on \MNISTSumXor (\LTN) and \MNISTEO (\LTN and \DPL). \ConceptsOnly thus partially satisfies \textbf{D2} on the concept side but fails both \textbf{D1} and \textbf{D3}.

\textbf{\TaskPlusAbduction achieves consistency and coverage but not conciseness}. \TaskPlusAbduction constructs a concept set via abduction, $\Cab = g^{-\dagger}(\Yset)$. This recovers both coverage and consistency at the concept level by construction: on every data set, \TaskPlusAbduction attains high concept coverage $\geq 0.89$ and $1.00$ consistency by construction. The caveat is concept-size. Because abduction indiscriminately includes every concept vector consistent with some label in $\Yset$, the resulting set is an order of magnitude larger than any competitor: $504.38$ and $742.55$ on \CEBAB for \DPL, supervised and unsupervised setting respectively, $1024$ for \LTN instead (the entire concept space), $171.66$ and $288.11$ on \CIFAR for \DPL and \LTN unsupervised, $2.20$-$50.00$ on the \MNIST variants.  \TaskPlusAbduction therefore satisfies \textbf{D1}--\textbf{D2} on the concept side but fails \textbf{D3} dramatically.

\textbf{\ConceptsPlusDeduction achieves consistency but inflated label sets}. \ConceptsPlusDeduction  constructs a label set by deducing $g^{\dagger}(\Cset)$ from the concept set. It achieves both label and concept coverage near the nominal level and $1.00$ consistency by construction (\cref{prop:nesy-forward-coverage}). The failure mode is on the label side: the deduced label set can be substantially larger than \TaskOnly. On \CEBAB (\DPL, supervised), as an example, \ConceptsPlusDeduction produces label sets of size $3.71$ vs.\ $2.41$ of \TaskOnly and on \MNISTHALF (\DPL) $9.00$ vs.\ $0.89$. And the same trend could be observed in every tested datasets. The issue is that every element of $\Cprod$ -- even low-probability ones -- contributes a label to $\Yfwset$. \ConceptsPlusDeduction is therefore a reasonable baseline when label size is not a binding constraint, but it inflates the label set failing \textbf{D3} on the labels.

\textbf{\bastani closes label-side consistency but leaves the concept set inconsistent}.  \bastani~\citep{ramalingam2024uncertainty} revises the label set as $\Yrev = \Yset \cap g^{\dagger}(\Cprod)$, but leaves $\Cprod$ untouched. It therefore produces \emph{identical} label sets to \method by construction, but underperforms on the concept side: concept consistency is typically $0.50$-$0.94$ across the eight data sets (vs.\ $1.00$ for \method), and the concept set is at least as large as \method's in every case where the two differ. Representative gaps include \CEBAB (\DPL, supervised), $0.78\to1.00$ consistency, $75.06\to53.49$ size, \CHX (DPL, supervised) $0.55 \to 1.00$, $5.30 \to 2.90$, \MNISTEO (\DPL) $0.12 \to 1.00$, $100.00 \to	12.13$, \MNISTEO (\LTN) $0.08\to1.00$, $100.00 \to 8.05$ and \MNISTHALF (\DPL) $0.09 \to 1.00$, $25.00 \to 2.20$. The takeaway is that revising only the label side leaves the concept side inconsistent with the deduced labels; the symmetric revision in \method is what closes this gap and delivers \textbf{D1} on \emph{both} sides while preserving \textbf{D2}.

\text{Summary.} Each one-sided baseline achieves one desideratum at the cost of another: \TaskOnly ensures label coverage but sacrifices concept coverage \emph{and} consistency; \ConceptsOnly ensures concept coverage but sacrifices consistency and concept conciseness; \TaskPlusAbduction ensures concept coverage and consistency but sacrifices concept conciseness; \ConceptsPlusDeduction ensures consistency but sacrifices label conciseness; \bastani ensures label consistency but leaves the concept set inflated and inconsistent. \method is the only method to revise both sides symmetrically in a single deduction--abduction step (\cref{prop:joint-fixed-point}) and to attain \textbf{D1}, \textbf{D2}, and \textbf{D3} simultaneously across all eight data sets and both NeSy backbones.

\subsection{Time Overhead Analysis}
\label{sec:runtime}

A key point of \method is that the joint revision operator reaches a fixed point in one step (\cref{prop:one-step}), making the additional cost a single computation of the abductive and deductive images, \ie $\Cab = g^{-\dagger}(\Yset)$ and $\Yfwset = g^{\dagger}(\Cset)$.

In practice, as shown by \cref{fig:time-aggregated,fig:time-per-dataset}, \method is comparable in runtime to all other ablations. While it is naturally slower than \TaskOnly and \bastani since it requires additional revision steps, \cref{prop:joint-fixed-point} ensures it remains tractable thanks to the \emph{\textbf{fixed-point}} theorem~(\cref{prop:joint-fixed-point}).

Looking at the aggregate results in \cref{fig:time-aggregated}, \method exhibits an overhead of approximately $160\%$ relative to \DPL and \LTN when considered individually in the unsupervised setting ($100\%$ horizontal line), and around $120\%$ in the supervised setting. This is comparable to \bastani ($140\%$ and $110\%$, respectively) and to \ConceptsPlusDeduction ($130-150\%$ unsupervised), and only moderately higher than \TaskOnly, which is the closest to the baseline.

Looking at individual datasets in \cref{fig:time-per-dataset}, the overhead of \method is dataset-dependent, as it depends on the number of concepts and labels produced by the conformal method, but remains consistently bounded, except for \CEBAB, where the number of concept combinations is $4\times$ higher than the baseline due to the large number of generated concepts and labels.
Notably, in \CHX -- our operational dataset -- and \RIVAL, the overhead is almost non-existent, demonstrating that in practical applications the overhead is negligible. In fact, \method is well-suited to settings where a human is in the loop and the number of concepts and produced labels is close to one, precisely the scenario in which \method excels.

Furthermore, when employing \method with e-values (\EMethod), as well as in other ablations, the computational overhead is approximately $200-250\%$ in the unsupervised setting and close to $150\%$ in the supervised setting. This is worse than \method even with the speed-up arising from avoiding the Bonferroni correction (\cref{prop:e-world-coverage}), as it eliminates the need to enumerate the full product set (\cref{eq:product-concept-set}).
\EMethod, used in a p-value-like fashion, produces consistently larger concept and label sets, except in \CHX and \RIVAL. Nevertheless, unlike all other approaches, \EMethod can be fixed in size and thus in computational overhead. In fact, with appropriate choices of $\hat{\beta}$ and $\hat{\alpha}$ -- namely those that keep the set size below a given threshold (as shown in~\cref{sec:experiments}) -- we can cap the computational cost, since we know in advance the size of the instances on which the inference layer $g$ will be applied.

\begin{figure*}[h!]
    \centering
    \includegraphics[width=\linewidth]{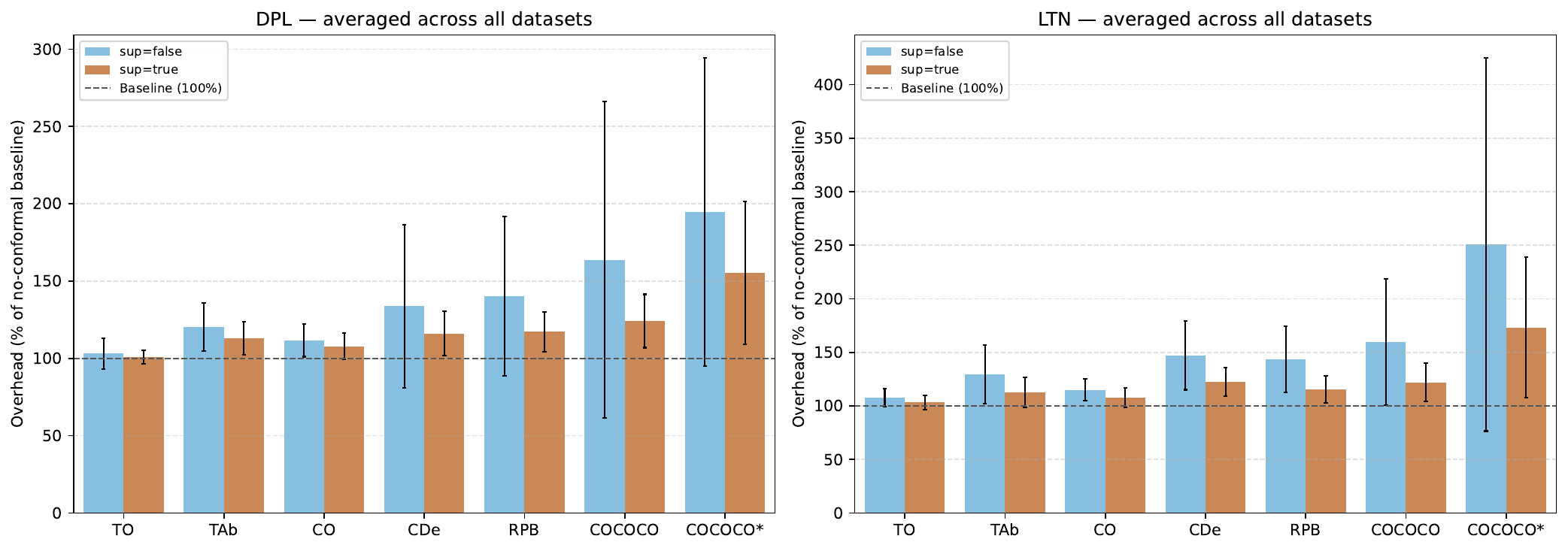}
    \label{fig:time-aggregated}
    \caption{Time overhead (\%) relative to \DPL (left) and \LTN (right) baseline (dashed line, $100\%$) for all methods (\TaskOnly, \TaskPlusAbduction, \ConceptsOnly, \ConceptsPlusDeduction, \bastani, \method and $\method^*$) averaged across all $10$ seeds and all datasets. Blue shows the case where concept supervision was absent, and orange shows the case where it was present. Error bars indicate standard deviation.}
\end{figure*}

\begin{figure*}[h!]
    \centering
    \begin{subfigure}{\linewidth}
        \centering
        \includegraphics[width=0.85\linewidth]{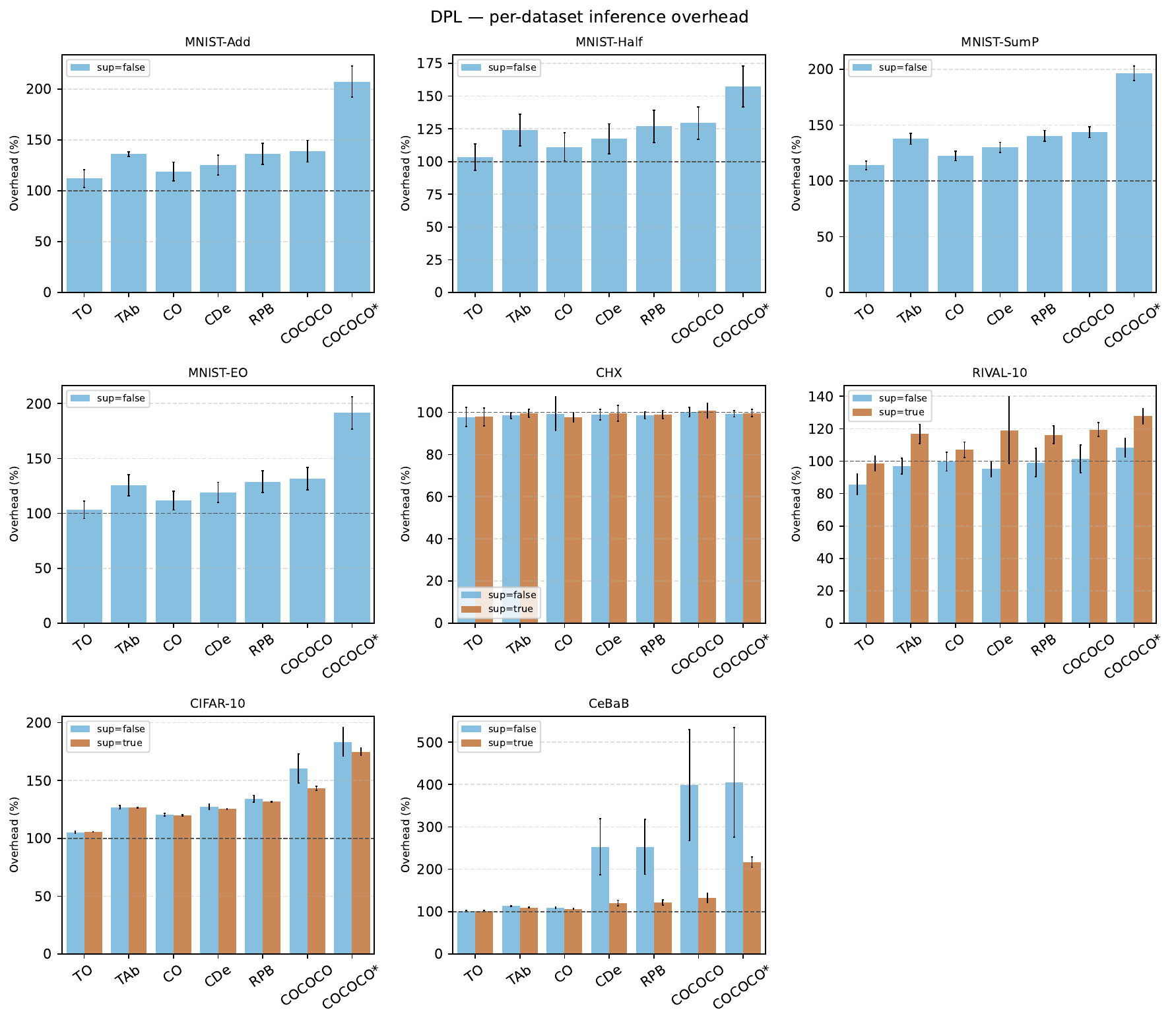}
    \end{subfigure}
    \vspace{0.5em}
    \begin{subfigure}{\linewidth}
        \centering
        \includegraphics[width=0.85\linewidth]{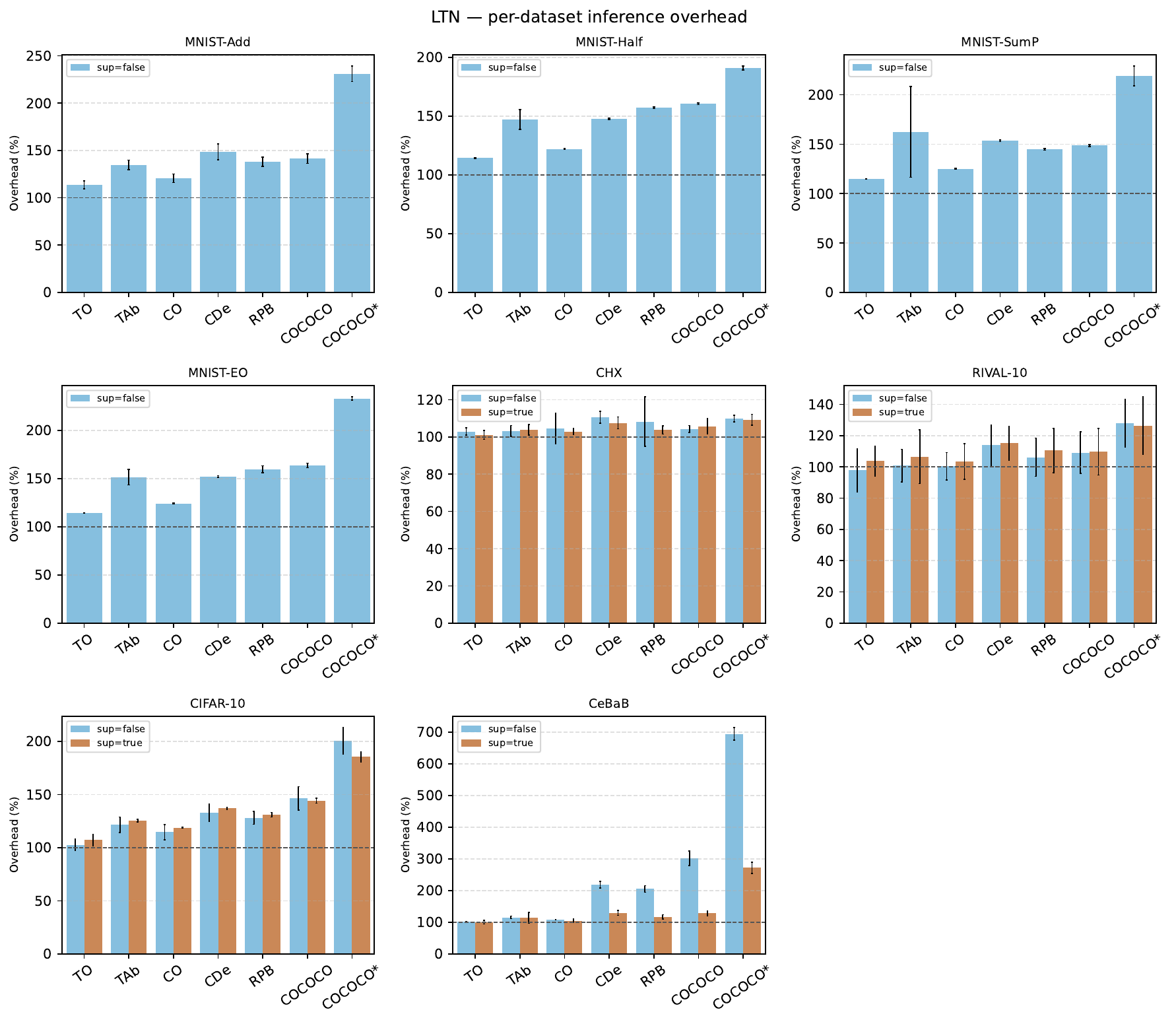}
    \end{subfigure}
    
    \caption{Time overhead (\%) relative to the \DPL and \LTN baseline (dashed line, $100\%$) for all methods (\TaskOnly, \TaskPlusAbduction, \ConceptsOnly, \ConceptsPlusDeduction, \bastani, \method, and $\method^*$), across all $8$ datasets, averaged over $10$ seeds. The top panel reports results with the \DPL baseline, and the bottom panel with the \LTN baseline. Blue indicates absence of concept supervision, while orange indicates its presence. Error bars denote standard deviation.}
    \label{fig:time-per-dataset}
\end{figure*}

\end{document}